\documentclass[11pt, letterpaper]{article}

\usepackage{geometry}
\usepackage{authblk}

\usepackage{tikz}
\usepackage{amsthm}
\usepackage[T1]{fontenc}
\usepackage{lmodern}
\usepackage{xcolor}
\usepackage{natbib}
\usepackage{booktabs}

\usepackage{tcolorbox,wrapfig}
\usepackage[labelfont={bf},font=small]{caption}
\usepackage{subcaption}
\usepackage[ruled,vlined]{algorithm2e}
\tcbuselibrary{skins}
\usepackage[utf8]{inputenc} 
\usepackage[T1]{fontenc}    
\usepackage{hyperref}       
\usepackage{url}            
\usepackage{amsfonts}       
\usepackage{nicefrac}       
\usepackage{microtype}      
\usepackage{xcolor}         

\usepackage{longtable}

\usepackage{hyperref,graphicx,amsmath,amsthm,amsfonts,amssymb,bm,url,breakurl,epsfig,epsf,color,MnSymbol,mathbbol,fmtcount,semtrans,caption,multirow,comment, boldline}
\usepackage{wrapfig}
\usepackage{enumitem}
\usepackage{amssymb}

\setlist[itemize]{leftmargin=5mm}

\usepackage[utf8]{inputenc} 
\usepackage[T1]{fontenc}    
\usepackage{url}            
\usepackage{booktabs}       
\usepackage{amsfonts}       
\usepackage{nicefrac}       
\usepackage{microtype}      
\usepackage{dsfont}
\usepackage{xspace}
\usepackage[normalem]{ulem}

\newcommand{\mini}{\,\wedge\,}
\newcommand{\maxi}{\,\vee\,}

\newcommand{\pre}{\text{pre}\xspace}
\newcommand{\new}{\text{new}\xspace}

\newcommand{\Tfs}{\text{Transformers}\xspace}
\newcommand{\Attn}{\text{attention}\xspace}

\newcommand{\ftiled}{\textsc{\small Full-Tiled}\xspace}
\newcommand{\ptiled}{\textsc{\small Partial-Tiled}\xspace}
\newcommand{\embd}{\textsc{\small Embed-in-ImageNet}\xspace}

\newcommand{\promptt}{\textsc{\small Prompt-tuning-I}\xspace}
\newcommand{\prefixt}{\textsc{\small Prompt-tuning-II}\xspace}
\newcommand{\prefixtfull}{\textsc{\small Prompt-tuning-III}\xspace}

\newcommand{\ratl}{\texttt{rate}_{\lin}}

\newcommand{\Prml}{Prompt-attention\xspace}
\newcommand{\prml}{prompt-attention\xspace}

\newcommand{\WV}{\mtx{W}_{V}}

\newcommand{\WK}{\mtx{W}_{K}}
\newcommand{\WQ}{\mtx{W}_{Q}}

\newcommand{\Qb}{{\mtx{Q}}}
\newcommand{\vb}{{\vct{v}}}

\newcommand{\sft}[1]{\phi(#1)}
\newcommand{\sftd}[1]{{\phi}^{\prime}(#1)}
\newcommand{\X}{{\mtx{X}}}
\newcommand{\Xp}{{\mtx{X}}_{\Pbb}}
\newcommand{\Pbb}{{\mtx{P}}}
\newcommand{\acc}{\textsc{err}}
\newcommand{\err}{\textsc{err}}
\newcommand{\taub}{{\bar{\tau}}}
\newcommand{\taut}{\tau}
\newcommand{\wstab}{\bar{\w}_\star}
\newcommand{\qstab}{\bar{\qb}_\star}
\newcommand{\qtt}{\widetilde{\qb}}
\newcommand{\qbb}{\bar{\qb}}
\newcommand{\vct}[1]{\bm{#1}}
\newcommand{\mtx}[1]{\bm{#1}}

\newcommand{\isnr}[1]{\texttt{ISNR}(#1)}

\newcommand{\Rcc}{\mathcal{R}^c}

\usepackage{mathtools}

\usepackage{titlesec}

\usepackage{tikz}
\usepackage{pgfplots}
\usetikzlibrary{pgfplots.groupplots}

\usepackage{movie15}

\usepackage{caption}
\usepackage[bottom,hang,flushmargin]{footmisc} 

\newcommand{\tsn}[1]{{\left\vert\kern-0.25ex\left\vert\kern-0.25ex\left\vert #1 
    \right\vert\kern-0.25ex\right\vert\kern-0.25ex\right\vert}}

\definecolor{darkred}{RGB}{150,0,0}
\definecolor{darkgreen}{RGB}{0,150,0}
\definecolor{darkblue}{RGB}{0,0,200}
\hypersetup{colorlinks=true, linkcolor=darkred, citecolor=darkgreen, urlcolor=darkblue}

\newtheorem{observation}{Observation}
\newtheorem{theorem}{Theorem}

\newtheorem{assumption}{Assumption}

\newtheorem{lemma}{Lemma}
\newtheorem{fact}[subsection]{Fact}
\newtheorem{corollary}{Corollary}

\newtheorem{remark}{Remark}

\newcommand{\psinorm}[2]{\left\|{#1}\right\|_{\psi_{#2}}}

\newcommand{\xti}{\x_{i,t}}
\newcommand{\zti}{\z_{i,t}}

\newcommand{\del}{\delta}
\newcommand{\delw}{\delta^w}
\newcommand{\delq}{\delta^q}
\newcommand{\ctx}{context\xspace}
\newcommand{\qstar}{\qb_\star}

\newcommand{\wstar}{\w_\star}

\newcommand{\corqarg}[1]{R_{#1,\qstar}}
\newcommand{\corwarg}[1]{R_{#1,\wstar}}

\newcommand{\corqv}{\corqarg{\vb}}
\newcommand{\corwv}{\corwarg{\vb}}

\newcommand{\corq}{R_{\qstar}}
\newcommand{\corw}{R_{\wstar}}
\newcommand{\hcorq}{\widehat{R}_{\qstar}}
\newcommand{\hcorw}{\widehat{R}_{\wstar}}
\newcommand{\bro}{\bar{\rho}}
\newcommand{\what}{\widehat{\w}}

\newcommand{\qhat}{\widehat{\qb}}
\newcommand{\Ghat}{\widehat{\G}}
\newcommand{\Gtilde}{\widetilde{\G}}
\newcommand{\Chat}{\widehat{C}}

\newcommand{\qnorm}{Q}
\newcommand{\wnorm}{W}

\newcommand{\epsb}{\boldsymbol{\eps}}

\newcommand{\rhat}{\hat{r}}

\DeclareMathOperator{\tr}{tr}

\newcommand{\widesim}[2][1.5]{
  \mathrel{\overset{#2}{\scalebox{#1}[1]{$\sim$}}}
}
\newcommand{\cut}[1]{\textcolor{red}{}}
\newcommand{\W}{\vct{W}}

\newcommand{\sign}[1]{\texttt{sign}(#1)}

\newcommand{\Ob}{\vct{O}}
\newcommand{\Z}{\vct{Z}}

\newcommand{\Ub}{\vct{U}}
\newcommand{\G}{\vct{G}}
\newcommand{\li}{\left<}
\newcommand{\ri}{\right>}

\newcommand{\A}{\vct{A}}

\newcommand{\kb}{\vct{k}}

\newcommand{\st}{\star}
\newcommand{\bgl}{\big|}

\newcommand{\pb}{\vct{p}}

\newcommand{\thetab}{\boldsymbol{\theta}}
\newcommand{\thetah}{\boldsymbol{\hat{\theta}}}

\newcommand{\fb}{\vct{f}}

\newcommand{\x}{\vct{x}}

\newcommand{\ub}{\vct{u}}
\newcommand{\w}{{\vct{w}}}

\newcommand{\g}{{\vct{g}}}
\newcommand{\Lch}{\widehat{\Lc}}
\newcommand{\bb}{\vct{b}}

\newcommand{\e}{\vct{e}}
\newcommand{\eb}{\vct{e}}

\newcommand{\z}{\vct{z}}

\newcommand{\Iden}{\vct{I}}
\newcommand{\cb}{\vct{c}}  
\newcommand{\ab}{\vct{a}}
\newcommand{\fat}{f^{\textsc{att}}}
\newcommand{\flin}{f^{\textsc{lin}}}
\newcommand{\lin}{\textsc{lin}}
\newcommand{\fsat}{f^{\textsc{satt}}}
\newcommand{\flinatt}{f^{\textsc{lin-att}}}
\newcommand{\qb}{{\vct{q}}}
\newcommand{\qh}{{\vct{\hat{q}}}}

\newcommand{\h}{\mathbf{h}}
\newcommand{\hb}{\mathbf{h}}

\newcommand{\Sc}{{\mathcal{S}}}
\newcommand{\Bc}{{\mathcal{B}}}

\newcommand{\Dc}{\mathcal{D}}

\newcommand{\bSi}{\boldsymbol{\Sigma}}
\newcommand{\order}[1]{\mathcal{O}(#1)}
\newcommand{\ordet}[1]{{\tilde{\mathcal{O}}}(#1)}
\newcommand{\subg}[1]{\Sc\Nc(#1)}

\newcommand{\Nc}{\mathcal{N}}
\newcommand{\Rc}{\mathcal{R}}
\newcommand{\Nn}{\mathcal{N}}
\newcommand{\Lc}{\mathcal{L}}
\newcommand{\Cc}{\mathcal{C}}

\newcommand{\Ec}{\mathcal{E}}

\newcommand{\Qc}{\mathcal{Q}}

\newcommand{\beq}{\begin{equation}}
\newcommand{\eeq}{\end{equation}}
\newcommand{\bea}{\begin{align}}
\newcommand{\eea}{\end{align}}

\newcommand{\vp}{\vspace{5pt}}

\newcommand{\R}{\mathbb{R}}
\newcommand{\E}{\mathbb{E}}

\newcommand{\Pro}{\mathbb{P}}
\newcommand{\nn}{\notag}

\newcommand{\tn}[1]{\|#1\|}

\newcommand{\term}[1]{\text{Term}_{\rm{#1}}}

    \newcommand{\gammab}{\boldsymbol\gamma}

  \newcommand{\Sigmab}{\boldsymbol\Sigma}
    
  \newcommand{\eps}{\epsilon}

 \newcommand{\deltab}{\boldsymbol\delta}

\newcommand{\wh}{{\hat\w}}

\DeclarePairedDelimiterX{\inp}[2]{\langle}{\rangle}{#1, #2}

\newcommand{\wtt}{\widetilde\w}

\newcommand{\wt}{\widetilde}

\newcommand{\Id}{\mathds{I}}
\newcommand{\ones}{\mathds{1}}

\providecommand{\abs}[1]{\left\lvert#1\right\rvert}

\newcommand{\Xt}{\widetilde\X}
\newcommand{\yt}{\widetilde{y}}

\newcommand{\ohat}{\hat{o}}

\usepackage[textsize=tiny,textwidth=1.7cm]{todonotes}

\newcommand{\todoasr}[1]{\todo[color=cyan,size=\tiny]{Ankit: #1}}

\newcommand{\Whead}{\W_{\text{head}}}

\usepackage{lipsum}

\title{
On the Role of Attention in Prompt-tuning
}

\author{Samet Oymak$^{\ast}$, Ankit Singh Rawat$^{\diamond}$, Mahdi Soltanolkotabi$^{\dagger}$, Christos Thrampoulidis$^{\ddagger}$}
\affil{$^{\ast}$University of Michigan, UC  Riverside,
{\small \texttt{oymak@umich.edu}}
\vspace{-2mm}
}
\affil{$^{\diamond}$Google Research NYC,
{\small \texttt{ankitsrawat@google.com}}
\vspace{-2mm}
}
\affil{$^{\dagger}$University of Southern California,
{\small \texttt{soltanol@usc.edu}}
\vspace{-2mm}
}
\affil{$^{\ddagger}$University of British Columbia,
{\small \texttt{cthrampo@ece.ubc.ca}}
\footnote{Alphabetical order.}
}

\allowdisplaybreaks

\begin{document}

\maketitle

\begin{abstract}
\textit{Prompt-tuning} is an emerging strategy to adapt large language models (LLM) to downstream tasks by learning a (soft-)prompt parameter from data. Despite its success in LLMs, there is limited theoretical understanding of the power of prompt-tuning and the role of the attention mechanism in prompting. In this work, we explore prompt-tuning for one-layer attention architectures and study contextual mixture-models where each input token belongs to a context-relevant or -irrelevant set. We isolate the role of prompt-tuning through a self-contained \textit{prompt-attention} model. Our contributions are as follows: (1) We show that softmax-prompt-attention is provably more expressive than softmax-self-attention and linear-prompt-attention under our contextual data model. (2) We analyze the initial trajectory of gradient descent and show that it learns the prompt and prediction head with near-optimal sample complexity and demonstrate how prompt can provably attend to sparse context-relevant tokens. (3) \todoasr{Right now this third contribution mostly appears in the appendix, do we want to keep it as a main contribution in the abstract?}Assuming a known prompt but an unknown prediction head, we characterize the exact finite sample performance of prompt-attention which reveals the fundamental performance limits and the precise benefit of the context information. We also provide experiments that verify our theoretical insights on real datasets and demonstrate how prompt-tuning enables the model to attend to context-relevant information.
\end{abstract}

\section*{Todos}
\begin{itemize}
    \item Read
    \item Match theorem 2 and 3. Choose $\alpha$ in Theorem 3
    \item Theorem 2: make large enough finite $\gamma$
    \item Sharp for fixed $\wstar$ finite $n$: learn $\qstar$ with one step
    \item Infinite $T$
    \item Move Fig 6 to intro?
    \item Add numerical simulations on dynamics -- connect to 
    \item revisit, simplify for $\rho=0$: Discussion connect to others
    \item Add \cite{wei2021pretrained}
\end{itemize}






\section{Introduction}


Transformer models have achieved remarkable success in a wide array of machine learning domains spanning language modeling, vision, and decision making. 
Recently, one of the key techniques that has helped pave the way for the deployment of transformers to ever increasing application areas is their ability to adapt to multiple unseen tasks by conditioning their predictions through their inputs -- a technique known as prompt-tuning~\citep{lester2021power,li2021prefix}.
 Concretely, prompt-tuning provides a more efficient (cheaper/faster) alternative to fine-tuning the entire weights of the transformer by instead training (fewer) so-called prompt parameters that are appended on the input and can be thought of as an input interface. In fact, several recent works have demonstrated experimentally that prompt-tuning is not only more efficient, but often even becomes competitive to fine-tuning in terms of accuracy \citep{lester2021power,liu2023pre}. However, there is currently limited formal justification of such observations. This motivates the first question of this paper:
 \begin{center}
\emph{How does prompt-tuning compare to fine-tuning in terms of  expressive power? Are there scenarios prompt-tuning outperforms fine-tuning in that regard?}
\end{center}



 The core constituent of a transformer, and thus of prompt-tuning, is the attention mechanism. Through the attention layer, prompts get to interact with other input features, create/modify attention weights, and facilitate the model to attend on latent task-specific information. The standard attention layer relies on softmax nonlinearities. Operationally, the softmax function allows a model to selectively focus on certain parts of the input tokens when generating attention outputs. However, there is little rigorous understanding of attention-based prompt-tuning. Concretely,  
\begin{center}
\emph{What is the role of the softmax-attention in prompt-tuning in terms of optimization and generalization? How does it locate and extract relevant contextual information?}
\end{center}




\noindent\textbf{Contributions.} Our contributions are as follows:
\begin{itemize}
\item We show that a particular form of attention which we refer to as \emph{prompt-attention} naturally arises from the self-attention model with prompt-tuning\todoasr{Different from `prompt-tuning`, we are also updating $w$ in our analysis. Do we need a footnote somewhere highlighting/justifying this? Updating `w' is still cheaper compared to updating the whole model.}. We rigorously show that prompt-attention is superior compared to linear-prompt-attention baselines and also identify regimes where it is even more expressive than self-attention. This separation result  demonstrates provable scenarios where prompt-tuning can be superior to fine-tuning with self-attention. 


\item We develop new statistical foundations for gradient-based prompt-tuning: we study the optimization and generalization dynamics of the initial trajectory of gradient descent for optimizing prompt-attention. We show the first few iterations learn the prompt and prediction head with near-optimal sample complexity while achieving high accuracy.

\item Our results provide insights into the critical role of softmax in facilitating attention: we show how the initial trajectory of gradient descent utilizes softmax to provably attend to sparse context-relevant tokens, ignoring noisy/nuisance tokens.

\item We also characterize the exact finite sample performance of prompt-attention assuming known prompt but unknown prediction head. This reveals the fundamental performance limits and the precise benefit of the context information.

\item Our results allow studying various tradeoffs between different model parameters (i) the role of sparsity, i.e.,~the fraction of context-relevant tokens and (ii) the relative effects of the different constituents of context-relevant tokens.

\item Finally, we empirically validate our theoretical insights on both synthetic contextual-mixture datasets and image-classification datasets. We compared multiple variants of prompt-tuning against standard fine-tuning on the latter. Furthermore, we highlight the role of prompt-attention in selecting relevant tokens in the image classification setting.
\end{itemize}

\paragraph{Related works.} Attention, specifically the so-called self-\Attn, is the backbone mechanism of transformers~\citep{vaswani2017attention}. It differs from conventional models (e.g., multi-layer perceptrons (MLPs) and convolutional neural networks (CNNs)) in that it computes feature representations by globally modeling interactions between different parts of an input sequence. Despite tremendous empirical success~\citep[see, e.g.,][and references therein]{vaswani2017attention,brown2020language,saharia2022photorealistic,ramesh2022hierarchical,chatgpt,narayanan2021efficient,reed2022generalist}, the underlying mechanisms of the \Attn layer remain largely unknown: {How does it learn? What makes it better (and when) compared to conventional architectures?} \citet{Yun2020Are} show that self-attention based transformers with large enough depth are universal approximators of seq2seq functions. Focusing on the self-attention component, \citet{edelman2021inductive} show 
that self-\Attn can efficiently represent sparse functions of its input space, while \citet{sahiner2022unraveling,ergen2022convexifying} analyze convex-relaxations of Self-\Attn, and \citet{baldi2022quarks,dong2021attention} study the expressive ability of \Attn layers.
However, these works do \emph{not} characterize the optimization and generalization dynamics of \Attn. To the best of our knowledge, the only prior works attempting this are \citet{jelassi2022vision} and \citet{ICLRsub2022}. \citet{jelassi2022vision} assume a simplified \Attn structure in which the \Attn matrix is \emph{not} directly parameterized in terms of the input sequence. \citet{ICLRsub2022} study optimization/generalization of training a full self-attention model. Our paper differs in a variety of ways including: (1) the data model in \citet{ICLRsub2022} is different as it has no notion of context vector and the noise in the data is assumed to be bounded whereas our model captures the role of context and our noise model is sub-Gaussian. (2) unlike \citet{ICLRsub2022}, we develop a precise asymptotic analysis that reveals the precise role of various problem parameters. (3) different from \citet{ICLRsub2022}, which focuses on vanilla self-attention, our focus is on understanding prompt-tuning via prompt-attention and when it can potentially improve upon vanilla self-attention. 

\section{Problem setting}
\label{sec:problem-setting}

\subsection{Motivation: Prompt-tuning}
\label{sec:motivation}
Consider a single-head self-\Attn layer
\begin{align}\label{eq:self-attn}
\Ob_\pre = 
\sft{\X\WQ\WK^\top\X^\top}\X\WV\,,
\end{align}
with input $\X\in\R^{T\times d}$ consisting of $T$  tokens of dimension $d$ each, trainable parameters $(\WK,\WQ,\WV)$ and a softmax nonlinearity $\phi:\R^T\mapsto\R^T$, $[\phi(\vb)]_t=e^{v_t}/\sum_{t'\in[T]}e^{v_{t'}}$ that acts row-wise when its argument is a $T\times T$ matrix. We scalarize the output of the self-\Attn layer with a trainable linear head $\bar{\Ub}$ which yields
\begin{align}\label{eq:self-attn-scalar-output}
y_\pre=\inp{\bar{\Ub}}{\Ob_\pre}=\inp{\Ub}{\sft{\X\WQ\WK^\top\X^\top}\X}\,.
\end{align}
Note here that we implicitly subsume the value matrix $\WV$ in the linear head via $\Ub:=\bar{\Ub}\WV^\top$.

We assume that the model above is pre-trained so that $\WK,\WQ,\Ub$ are fixed. Our goal is to use the pretrained transformer on (potentially) new classification tasks. Towards this goal, we explore the use of prompt-tuning, introduced in \citet{li2021prefix, lester2021power} as an alternative to fine-tuning the existing transformer weights. 

Prompt-tuning appends a trainable prompt $\Pbb\in\R^{m\times d}$ to the input features $\X\in\R^{T\times d}$ with the goal of conditioning the transformer to solve the new classification task. Let $\Xp:=\begin{bmatrix}\Pbb\\\X\end{bmatrix}\in\R^{(T+m)\times d}$ be the new transformer input. The output of the \Attn-layer is thus is of the form
$$
\Ob=\sft{\Xp\WQ\WK^\top\X^\top}\X.
$$
Note that this is slightly different from  \eqref{eq:self-attn} in that now the layer computes a cross-\Attn between the augmented inputs $\Xp$ and the original inputs $\X$. This is also equivalent to self-attention on $\Xp$ after masking the prompt $\Pbb$ from keys. This masking is used to cleanly isolate the residual contribution of the prompt without impacting the frozen attention output. Concretely, let {$\Whead$} be the prediction head associated with the prompt tokens. As before, we scalarize the output by projecting with a linear head of size $(T+m)\times d$ as follows:
\begin{align}
&y=\inp{[{\Whead}^\top~\Ub^\top]^\top}{\sft{\Xp\WQ\WK^\top\X^\top}\X}\label{eq:prompt-output-general}
\\
&\hspace{-0.1in}= \underbrace{\inp{{\Whead}}{\sft{\Pbb\WQ\WK^\top\X^\top}\X}}_{\text{prompt-attention $y_\new$}}+\underbrace{\inp{\Ub}{\sft{\X\WQ\WK^\top\X^\top}\X}}_{\text{frozen self-attention $y_\pre$}} \nn.
\end{align}
Here, $y_\new$ captures the additive impact of prompt-tuning on the prediction. We denote the trainable parameters in the model above as $\thetab:=({\Whead},\Pbb)$
\footnote{{In our model, we  train the classifier head $\Whead$ in addition to the  prompt vectors $\Pbb$. Despite the additional training for the classifier head, the computational overhead remains minimal, and the overall scheme remains significantly more efficient compared to updating the entire model $\WQ,\WK,\Ub$.}} Since the $y_\new$ term becomes a self-contained module and the features attend directly to the prompt vector, we will refer to it as \emph{prompt-attention}.

Our goal is understanding the expressivity, training dynamics, and generalization properties of the above model. To simplify our analysis, we consider the following setting.

\begin{enumerate}[leftmargin=*]
    \item We focus our attention on the novel component $y_\new$ of the model output in \eqref{eq:prompt-output-general} so as to isolate and fully understand the capabilities of prompt-attention.
    \item We assume $\WK,\WQ\in\R^{d\times d}$ are full-rank.
    \item We assume a single trainable prompt $\qb\in\R^d$ i.e.,~$m=1$.
\end{enumerate}


\textbf{Prompt-attention model.} Using these assumptions and setting $\qb:=\WK\WQ^\top \Pbb^\top\in\R^d$ and $\w={\Whead^\top}\in\R^d$, we arrive at our core \emph{prompt-attention} model $\fat_{\thetab}$ (or simply $\fat$):
\begin{align}
    \fat_{\thetab}(\X) = \inp{\w}{\X^\top\phi\left(\X\qb\right)},~\thetab=(\w,\qb).\label{eq:PA}
\end{align}
We shall see that this model exhibits interesting properties to learn rich contextual relationships within the data and can even be more expressive than a single self-attention layer.

We remark that the model above is of interest even beyond the context of prompting: the \prml model in \eqref{eq:PA} is reminiscent of  the simplified model proposed in earlier seq2seq architectures \citep{bahdanau2014neural,xu2015show,chan2015listen} preceding self-attention and \Tfs \citep{vaswani2017attention}. Indeed, in  the simplified attention mechanism of \cite{bahdanau2014neural,xu2015show,chan2015listen}, the tokens' \emph{relevance scores} and corresponding \emph{attention weights} are determined by $\ab=\phi(\X\qb)$ 
in which $\qb$ is a trainable vector and $\phi$ is the softmax-score transformation. Note here that the trainable parameter $\qb$ corresponds exactly to the trainable prompt vector in \eqref{eq:PA}.

\subsection{Contextual data model}
\label{sec:data_model}
Consider supervised classification on IID data $(\X,y)\sim \Dc$ with features $\X\in\R^{T\times d}$ and binary label $y\in\{\pm1\}$.  

\noindent\textbf{Dataset model.}~We assume the following about an example $(\X,y)$ drawn from $\Dc$: The labels $y$ are distributed as $\Pro(y=1)=1-\Pro(y=-1)=\pi$; for simplicity, we set $\pi=1/2$ so that $\E[y]=0$. The tokens $\x_t,t\in[T]$ of input example $\X:=[\x_{1},\ldots,\x_{T}]$ are split into a \emph{\ctx-relevant} set $\Rc\subset [T]$ and \emph{\ctx-irrelevant} set $\Rcc:=[T]-\Rc$. Specifically, conditioned on the labels and relevance set $\Rc$, tokens $\x_t, t\in[T]$ within $\X$ are i.i.d.~as follows 
\begin{equation}
\x_t | y = \begin{cases}
\qstar + y \wstar, & 
t\in\Rc
~~~~~~\text{(relevant token)}
\\
-\delq\qstar-\delw y\wstar+\z_t, 
& 
t\not\in\Rc~~~~~~\text{(\textbf{ir}relevant token)}\,.
\end{cases}\label{eq:CGMM}\tag{DATA}
\end{equation}
In the above, $\qstar$ is a \ctx-vector that indicates token relevance and $\wstar$ is a regressor vector. $y,\del:=(\delq,\delw), (\z_t)_{t=1}^T$ are independent variables as follows: 

\begin{itemize}[leftmargin=*]
\item $\del=(\delq,\delw)\in\R^2$ is a bounded random variable that obeys $\delq\geq 0$. Thus, $\del$ reflects \emph{out-of-context} information within irrelevant tokens. However, $\del$ is allowed to expose label information through $\delw$. When $\del=(0,0)$ almost surely, we call the resulting distribution \textbf{core dataset model}.  

\item $\z_t$ are independent centered subgaussian and random variables with covariance $\bSi$ (see Ass. \ref{ass:noise}). When $\bSi=0$, we call the resulting distribution \textbf{discrete dataset model}.

\item We allow the relevance set $\Rc$ to be non-stochastic. This includes $\Rc$ being adversarial to the classification model.

\item We assume constant fraction $\zeta=|\Rc|/T\in(0,1)$ of  label-relevant tokens for each input example $\X$ drawn from $\Dc$. Thus,  $\zeta$ represents the sparsity of relevant signal tokens.
\end{itemize}
\noindent\textbf{Training dataset $\Sc:=(\X_i,y_i)_{i=1}^n$.} We draw $n$ i.i.d.~samples from $\Dc$ to form our training dataset $\Sc:=(\X_i,y_i)_{i=1}^n$. For $i$'th example $(\X_i,y_i)$, we denote the tokens by $(\xti)_{t=1}^T$, noise by $(\zti)_{t=1}^T$, relevance set by $\Rc_i$, and out-of-context variable by $\del_i=(\delq_i,\delw_i)$.

 Ideally, for $i$'th example, we would like to identify its \ctx-relevant set $\Rc_i$ and discard the rest. This would especially help when the signal-to-noise-ratio is small, i.e. $\zeta=|\Rc_i|/T\ll 1$. This is precisely the role of the \ctx-vector $\qstar$: Observe that, per our construction, relevant tokens have positive correlation with $\qstar$ whereas irrelevant tokens have non-positive correlation with $\qstar$ in expectation. Thus, by focusing attention onto tokens based on their $\qstar$ correlation, we can potentially select the relevant set.


 \begin{remark}[Model interpretation] \eqref{eq:CGMM} can be thought of as a simplified model for binary image classification with tokens being image patches of two types: ones revealing information about the label (set $\Rc$) and uninformative ones containing noise. Tokens in $\Rc$ contain information indicating: (i) class-membership via signed-regressor $y\wstar$ and (ii) \ctx-relevance via \ctx-vector $\qstar$. The signed-regressor differs across tokens of examples belonging to different classes $y\in\{\pm1\}$, while the \ctx-vector is common for all \ctx-relevant tokes across classes. For a concrete example, consider images each depicting multiple, say $|\Rc|$, birds of one type surrounded by label-irrelevant/noisy background. The goal is to classify images according to one of two types of birds. Here, think of ``\ctx'' as feature-information indicating corresponding pixels belong to ``bird'' (of either type) rather than ``background,'' while the ``regressor'' represents feature information useful to distinguish between two bird types. Alternatively, \eqref{eq:CGMM} may be modeling deep representations (rather than raw pixels) of the original images. Overall, simplified models similar to \eqref{eq:CGMM} have been used previously to analyze optimization and generalization dynamics of training fully-connected \citep{frei2022random} and convolutional models \citep{cao2022benign}. Specifically, \eqref{eq:CGMM} is an extension of the commonly used (sub)-Gaussian mixture model customized to the nature of \Attn: each example is tokenized and \ctx-relevant information is described in terms of both a regressor (differing between classes) and a \ctx (common across classes).
 \end{remark}

\subsection{Baseline Models}
We compare performance of the \prml model in \eqref{eq:PA} with the following three baseline models.

\noindent\textbf{The linear model} is parameterized by $\thetab=\w$ and outputs
\begin{align}\label{eq:linear_model}
\flin(\w) = \frac{1}{T}\w^\top\X^\top\ones_T=\frac{1}{T}\sum\nolimits_{t\in[T]}\w^\top\x_t.
\end{align}
Note  this corresponds to a \prml model with uniform attention weights $[\ab]_t=1/T, t\in[T].$

\noindent\textbf{The self-\Attn model} is a strict generalization of the linear model. Recalling \eqref{eq:self-attn-scalar-output}, let us merge the key-query weights $\W:=\W_Q\W_K^\top$ (without losing generality) 
and gather weights into $\thetab=(\Ub,\W)$; We then write it as
\begin{align}\label{eq:satt_model}
\fsat(\Ub,\W) = \frac{1}{T}\li\Ub,\sft{\X\W\X^\top}\X\ri.
\end{align}
Rather than using a $Td$ dimensional $\Ub$, we will also consider the simpler token-pooling via $\Ub=\ones_T\ub^\top$ for $\ub\in\R^d$.


\noindent\textbf{The linear-\Attn model}  parameterized by $\thetab=(\w,\qb)$ replaces the softmax score transformation in \eqref{eq:PA}  with a linear function and outputs
\begin{align}\label{eq:linatt_model}
\flinatt(\w,\qb) = \w^\top\X^\top\X\qb/T.
\end{align}

\subsection{Training}
Given $\Sc=(\X_i,y_i)_{i=1}^n$ drawn i.i.d.~from $\Dc$, we solve square-loss empirical risk minimization to obtain $\thetah=(\wh,\qh)$ 
\begin{align}\label{eq:loss}
\thetah=\arg\min_{\thetab} \Lch_{\Sc}(\thetab):=\frac{1}{2n}\sum_{i=1}^n (y_i-f_{\thetab}(\X_i))^2.
\end{align}
Within our theoretical investigation, we are interested in the following performance criteria for models $f\in \{\fat,\flinatt,\fsat\}$:
\begin{itemize}
\item \textbf{Classification error}: For a model $f_{\thetah}$ this is defined as $\err(\thetah):=\Pro(y f_{\thetah}(\X)<0)$.
\item~\textbf{Test risk:} $\Lc(\thetah)=\E_{(y,\X)\sim \Dc}[(y-f_{\thetah}(\X))^2]$.
\end{itemize}

\subsection{Assumptions and notations}






 First, we formally state our assumptions on the noisy tokens. The more general condition is that noise is subgaussian and satisfies a mild zero third-moment condition.

\let\origtheassumption\theassumption
\edef\oldassumption{\the\numexpr\value{assumption}+1}
\setcounter{assumption}{0}
\renewcommand{\theassumption}{\oldassumption.\alph{assumption}}

\begin{assumption}\label{ass:noise}
The noise vector $\z\sim\subg{\sigma}$ is centered $\sigma$-subGaussian, i.e. $\tn{\z}_{\psi_2}=\sigma$. Moreover, its distribution is symmetric and has zero-third moment, i.e. $\E\Big[\z\otimes \left(\z^\top \z\right)\Big]=0$. 
Let $\Sigmab:=\E[\z\z^\top]$ denote the noise covariance.
\end{assumption}

For some of our results it will be convenient to further assume that noise is Gaussian since this leads to precise formulas that are easily interpretable.

\begin{assumption}\label{ass:noise Gaussian}
The noise vector $\z\sim\Nn(0,\sigma^2\Id)$ is isotropic Gaussian with variance $\sigma^2.$
\end{assumption}

 \let\theassumption\origtheassumption

\edef\oldassumption{\the\numexpr\value{assumption}+1}

\setcounter{assumption}{0}
\renewcommand{\theassumption}{\oldassumption.\alph{assumption}}

Second, we require a mild assumption on the correlation between the context $\qstar$ and classifier $\wstar$ to guarantee that pure signal tokens $\qstar+y\wstar$ are correctly classified by the true regressor $\wstar,$ i.e. $y\wstar^\top(\qstar+y\wstar)>0$. For convenience, we denote
\[
\wnorm:=\tn{\wstar},~~~
\qnorm:=\tn{\qstar},~~~
\rho:={\qstar^\top\wstar}/{(\tn{\qstar}\tn{\wstar})}.
\]
\begin{assumption}\label{ass:rho general but small}
    Correlation  satisfies $\abs{\rho}<{W}/{Q}.$
\end{assumption}

We will also often
 consider the special case of zero correlation $\rho$ and thus state it as separate assumption below. This  orthogonality assumption, is useful for more tractable analysis as it helps decouple feature selection and prediction.
\begin{assumption}\label{ass:orthogonal}The context and classifier vectors are orthogonal, i.e. $\qstar\perp\wstar$.
\end{assumption}

\noindent\textbf{Notation.}~
We use boldface letters for vectors and matrices. $\ones_m$ represents an $m$-dimensional all-ones vector. For a vector $\vb$, $\tn{\vb}$ denotes its Euclidean norm and $\vb/\tn{\vb}$ its normalization. $\sft{\cdot}$ denotes the softmax transformation. $\mathcal{Q}(\cdot)$ denotes the tail function of the standard normal distribution. $\mini$ and $\maxi$ denote the minimum and maximum of two numbers, respectively.  $\ordet{}$ and $\gtrsim$ notations suppress logarithmic dependencies. Finally, $\propto$ denotes proportionality.


\begin{figure*}[t!]
    \begin{subfigure}{0.28\textwidth}
        \centering
        \includegraphics[scale=0.6]{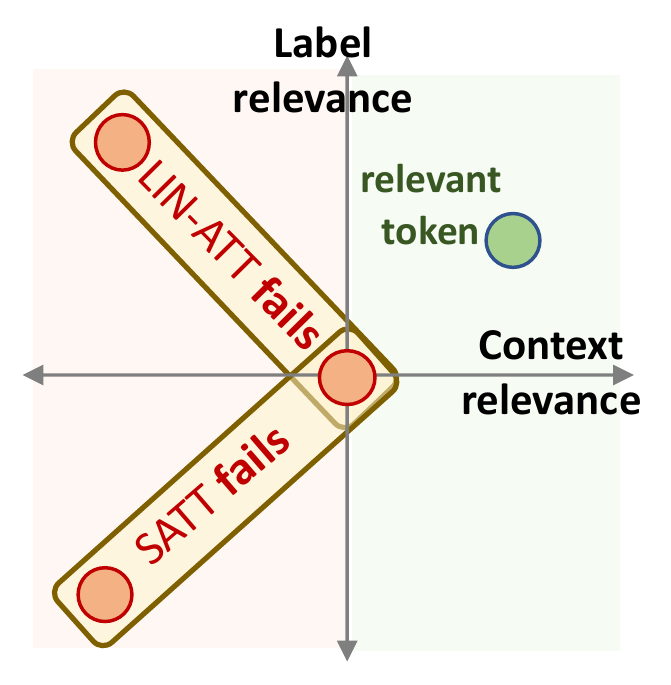}
        \caption{Failure modes}
        \label{fig:icl}
    \end{subfigure}%
    \begin{subfigure}{0.36\textwidth}
        \centering
        \includegraphics[scale=0.26]{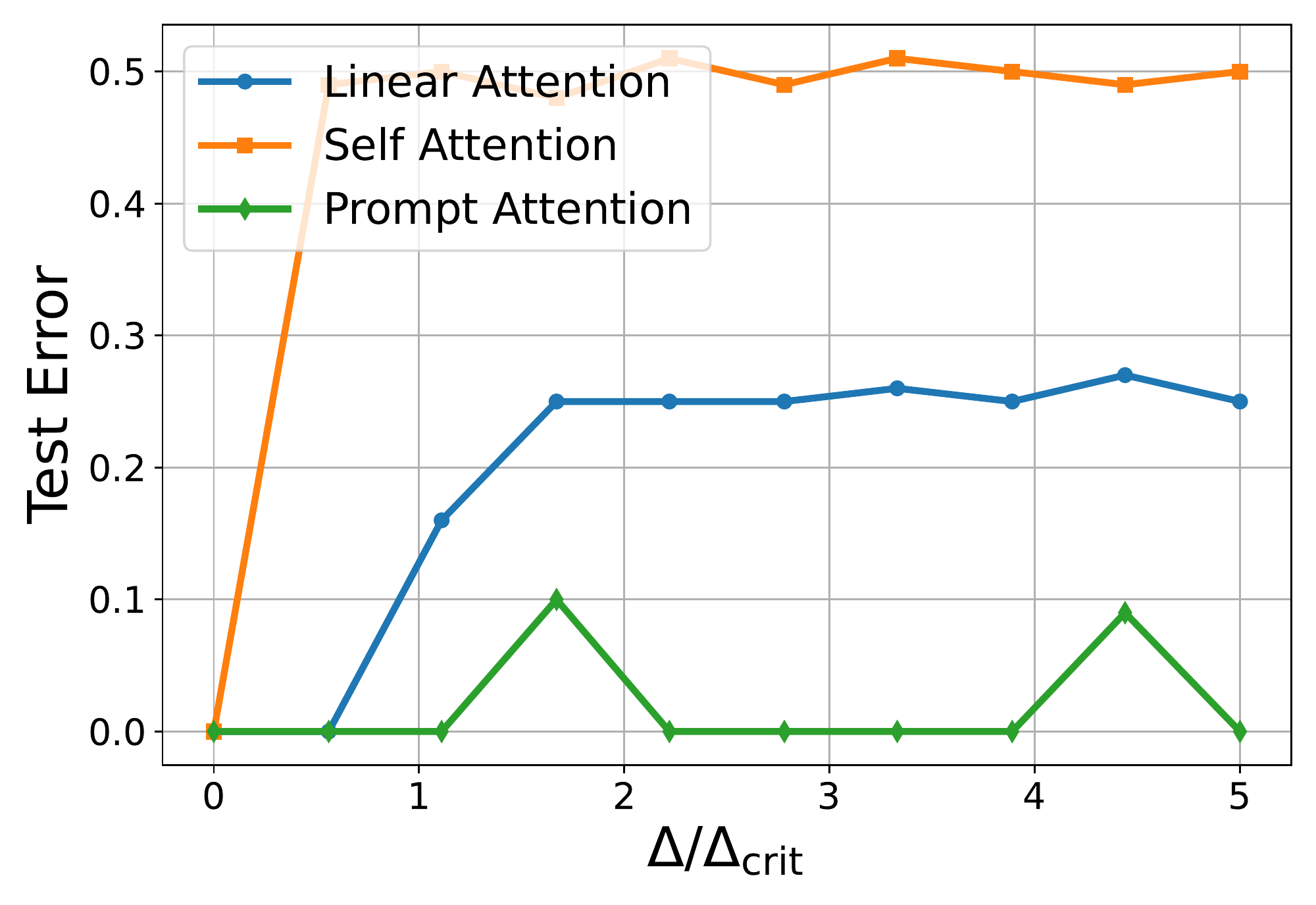}
        \caption{Varying $\Delta$ with $\delta_q=\delta_w$}
        \label{fig:thm1pp}
    \end{subfigure}%
    \begin{subfigure}{0.36\textwidth}
        \centering
        \includegraphics[scale=0.26]{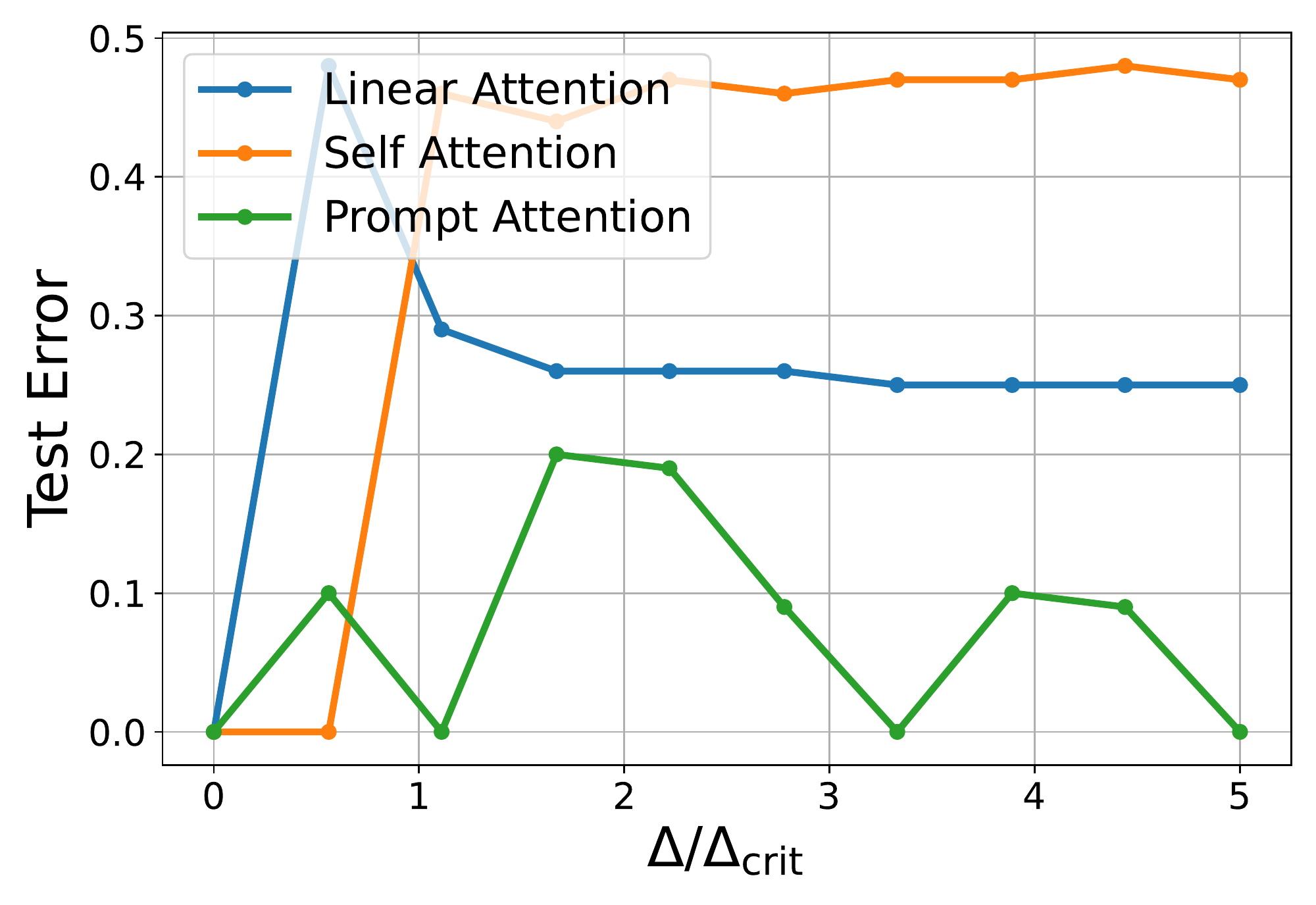}
        \caption{Varying $\Delta$ with $\delta_q=-\delta_w$}
        \label{fig:thm1pn}
    \end{subfigure}
    \caption{This figure summarizes and verifies the outcomes of Theorem 1. \textbf{Fig (a)} depicts the outcome of our Theorem \ref{separate thm}. Relevant token is at position $y\wstar+\qstar$ whereas red tokens (irrelevant) are in positions $-\delta(\qstar\pm y\wstar)$ with $\delta\in\{0,\Delta\}$. \textbf{Figures (b) \& (c)} plot the performance of our attention models under the contextual dataset with $\delta$ equally-likely over $\{0,\Delta\}$ {for a synthetic setup (cf.~Section~\ref{sec:exp-synthetic})}. We set $n=100$, $d=T=10$, $\zeta=0.4$ and train with 100 SGD epochs. We report the median test accuracy over 20 trials. Fig (b) sets $\delta=\delta_q=\delta_w$ and verifies self-attention has $\geq 50\%$ error (for $\Delta\geq \Delta_{\text{crit}} {=(1 - \zeta)^{-2}}$). Fig (c) sets $\delta=\delta_q=-\delta_w$ and verifies linear-prompt-attention has $\geq 25\%$ error {(when $\delta \geq \Delta_{\text{crit}} = \sqrt{\zeta/(1 - \zeta)}$ in this case)}.}
    \label{fig:thm1-verification}
\end{figure*}
\section{Contrasting \prml to baselines}

In this section, we establish separation results between \prml (cf.~\eqref{eq:PA}) and the baselines of self-attention (cf.~\eqref{eq:satt_model}) and linear attention (cf.~\eqref{eq:linatt_model}). For this, we focus on the discrete dataset model with noiseless irrelevant tokens ($\z_t=0, t\in[T]$).

We first observe that if $\del=(\delq,\delw)$ admits a single value, even a linear model can solve the discrete dataset model.
\begin{observation}[Linear model solves singleton] \label{obv 1}Suppose $(\delq,\delw)=(\Delta^q,\Delta^w)$ almost surely for $\Delta^q,\Delta^w\in\R$. Set $\wstar'=(\Iden-\qstab\qstab^\top)\wstar$.
As long as $\Delta^w\neq \zeta/(1-\zeta)$ and $\wstar'\neq 0$, $\flin(\wstar')$ or $\flin(-\wstar')$ achieves perfect accuracy.
\end{observation}
Thus, to investigate the expressivity of $\fat,\fsat,\flin$, $(\delq,\delw)$ would need to admit two or more values. Perhaps surprisingly, we prove that, as soon as, $(\delq,\delw)$ comes from a binary distribution, then both $\fsat$ and $\flin$ can indeed provably fail. Importantly, this happens in the regime $\delq\geq 0$ 
where \prml thrives. 


\begin{theorem}[Separation of population accuracies] \label{separate thm} Consider the discrete dataset model where we set $\bSi=0$ in \eqref{eq:CGMM}. The following statements hold:

\noindent1.~\textbf{\Prml:} 
Suppose $\rho^2<1$, $\delq\geq 0$, and $|\delw|\leq C$ almost surely. Define $\qstar'=(\Iden-\wstab\wstab^\top)\qstar,\wstar'=(\Iden-\qstab\qstab^\top)\wstar$. For $\Gamma>\frac{\log({C(1/\zeta-1)})}{\qnorm^2(1-\rho^2)}$, choosing $\thetab=(\wstar',\Gamma\qstar')$, $\fat_{\thetab}$ achieves perfect classification accuracy on \eqref{eq:CGMM}.

\noindent2.~\textbf{Self-attention:} In \eqref{eq:CGMM}, choose $(\delq,\delw)$ to be $(0,0)$ or $({\Delta,\Delta})$ equally-likely with $\Delta>(1-\zeta)^{-2}$. \vspace{-5pt}
\begin{itemize}
\item For any choice of $(\Ub=\ones\ub^\top,\W)$, $\fsat(\ones\ub^\top,\W)$ achieves 50\% accuracy (i.e.~random guess).\vspace{-5pt}
\item For any choice of $(\Ub,\W)$, there exists a \eqref{eq:CGMM} distribution with adversarial relevance set choices such that $\fsat(\Ub,\W)$ achieves 50\% accuracy.
\end{itemize}\vspace{-5pt}

\noindent3.~\textbf{Linear-attention:} In \eqref{eq:CGMM}, choose $(\delq,\delw)$ to be $(0,0)$ or $({\Delta,-\Delta})$ equally-likely with $\Delta>\sqrt{\zeta/(1-\zeta)}$. For any choice of $(\w,\qb)$, $\flinatt(\w,\qb)$ achieves at most 75\% accuracy.

\end{theorem}

{See Fig.~\ref{fig:thm1-verification} for an illustration of the main takeaways from Thm.~\ref{separate thm} and a numerical validation of its conclusion on synthetic data.} While surprising, the reason \prml can provably beat self-attention is because it is optimized for context-retrieval and can \emph{attend} perfectly on the relevant contextual information. In contrast, self-attention scores are fully feature-based; thus, context information is mixed with other features and can be lost during aggregation of the output. Also note that all results, with the exception of self-attention for general $\Ub$, hold for arbitrary choices of the relevance sets (including adversarial ones). The reason is that tokens are pooled and the particular choice of $\Rc$ does not matter. Only for $\fsat(\Ub,\W)$ we need to adapt the relevance set $\Rc$ to the output layer $\Ub$ (as well as $(y,\del)$ variables) to promote misclassification. Otherwise, with the hindsight knowledge of the relevance set, $\Ub$ can intelligently process individual tokens of the self-attention output to filter out ``confusing'' tokens. In fact, for the same failure dataset model, self-attention can achieve perfect accuracy by choosing $\Ub=\ones_{\Rc}\wstar'^\top$ where $\ones_{\Rc}$ is the vector of ones over the (known!) relevance set $\Rc$ (see Lemma \ref{lem satt success}). However, this is of course only known in hindsight.


%
\input{initial trajectory results.tex}
\section{Experiments}

First, we verify the utility of \prml via experiments on a synthetic setting that precisely follows the contextual data model from Section~\ref{sec:data_model}. Subsequently, we explore prompt-tuning on various image classification tasks that are motivated by the contextual data model and compare it with the standard fine-tuning method. Finally, we validate the utility of prompt vectors in distinguishing relevant tokens from irrelevant tokens via \prml under an image classification setting. 

\subsection{Synthetic experiments}
\label{sec:exp-synthetic}
Here, we generate a synthetic dataset according to the \textit{core dataset model}, i.e. we have $\delta = (\delta^q, \delta^w) = (0,0)$ for all examples in the dataset. In particular, we consider a setting with $T$ = 500, $d$=50, and $\zeta = 0.1$, i.e. each example has 500 tokens out of which 10\% tokens are relevant. As for the noisy tokens, they consists of i.i.d. $\mathcal{N}(0, \mathbf{I})$ vectors. Assuming that $\qstar \perp \wstar$ and $\sqrt{T}\wnorm = 3$, we generate $n = 10\cdot d$ training examples from the core dataset model for varying $\qnorm$. Fig.~\ref{fig:synthetic-results} showcases the performance of \prml (cf.~\eqref{eq:PA}) when combined with the estimates $\hat{\qb}$ and $\hat{\w}$ produced by gradient-based algorithm in \eqref{eq:algo}. We also showcase the performance of the linear model (cf.~\eqref{eq:linear_model}) and two oracle methods where we assume access to true $\qstar$ and true $(\qstar, \wstar)$, respectively, while applying the \prml. Note that \prml achieves a vanishing classification test error in this setting whereas a natural baseline (linear model) can fail to achieve a good performance. On the other hand, the \prml enabled by \eqref{eq:algo} successfully achieves a high accuracy as the context energy (defined by $\qnorm$) increases, validating the utility of \prml as well as our gradient-based algorithm in \eqref{eq:algo}.

{\color{blue}Finally, we also consider $\delta = (\delta^q, \delta^w) \neq (0, 0)$ in the aforementioned synthetic setup to validate Theorem~\ref{separate thm} in Fig.~\ref{fig:thm1-verification}.}

\begin{figure}[h!]
        \centering
        \includegraphics[scale=0.5]{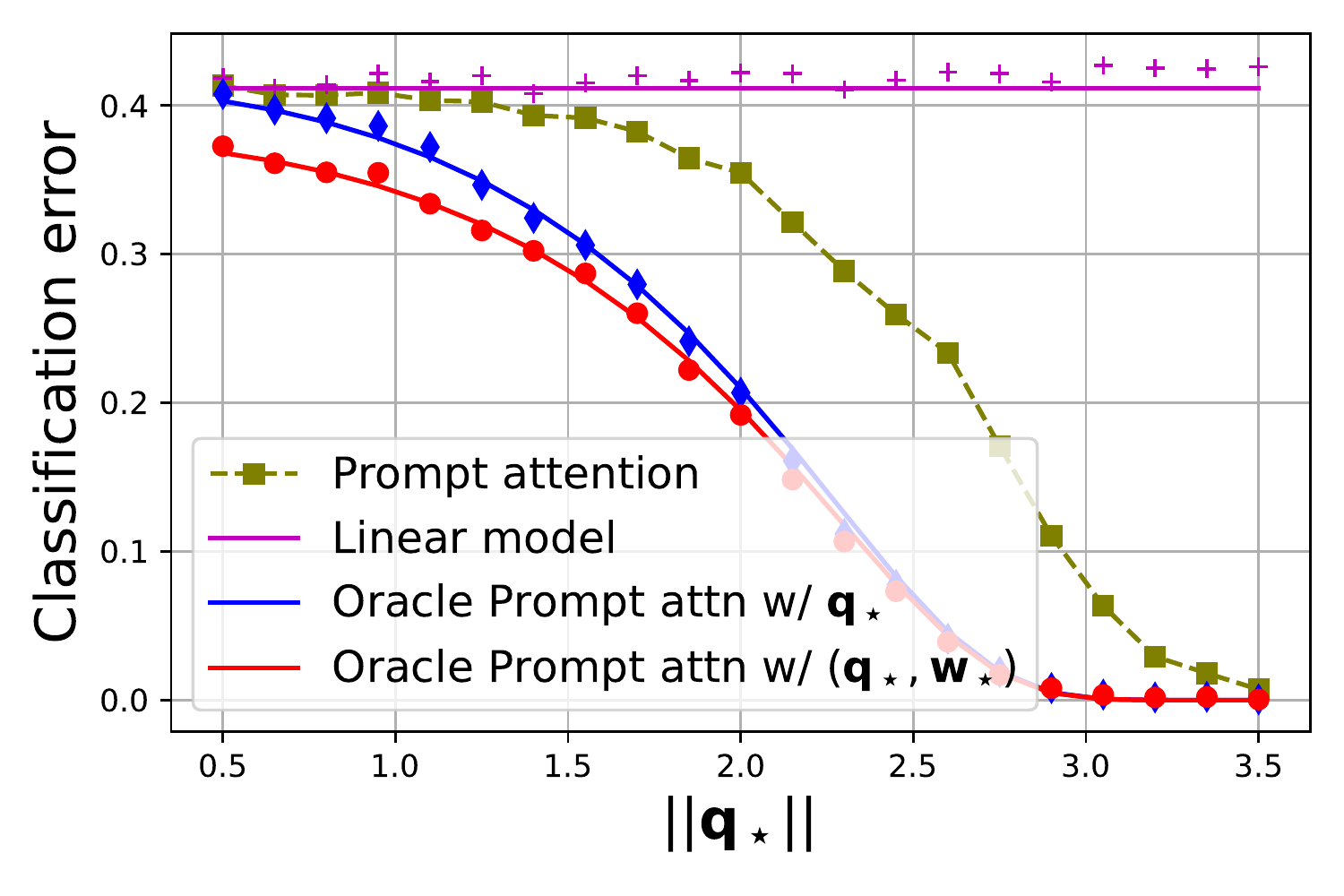}
    \caption{Performance of \prml on the synthetic setting described in Section~\ref{sec:exp-synthetic}. For \prml, we employ the algorithm in \eqref{eq:algo} to obtain $\hat{\qb}$ and $\hat{\w}$. For the baseline linear model and two oracle settings, we have closed-form expressions for their asymptotic test error (cf.~Theorem~\ref{separate thm}), which are depicted by solid lines. On the other hand, markers show the finite sample performance of these three methods. All finite sample performances are obtained by averaging over 20 independent runs.}
    \label{fig:synthetic-results}
\end{figure}

\begin{figure*}[t!]
    \begin{subfigure}{0.31\textwidth}
        \centering
        \includegraphics[scale=0.15]{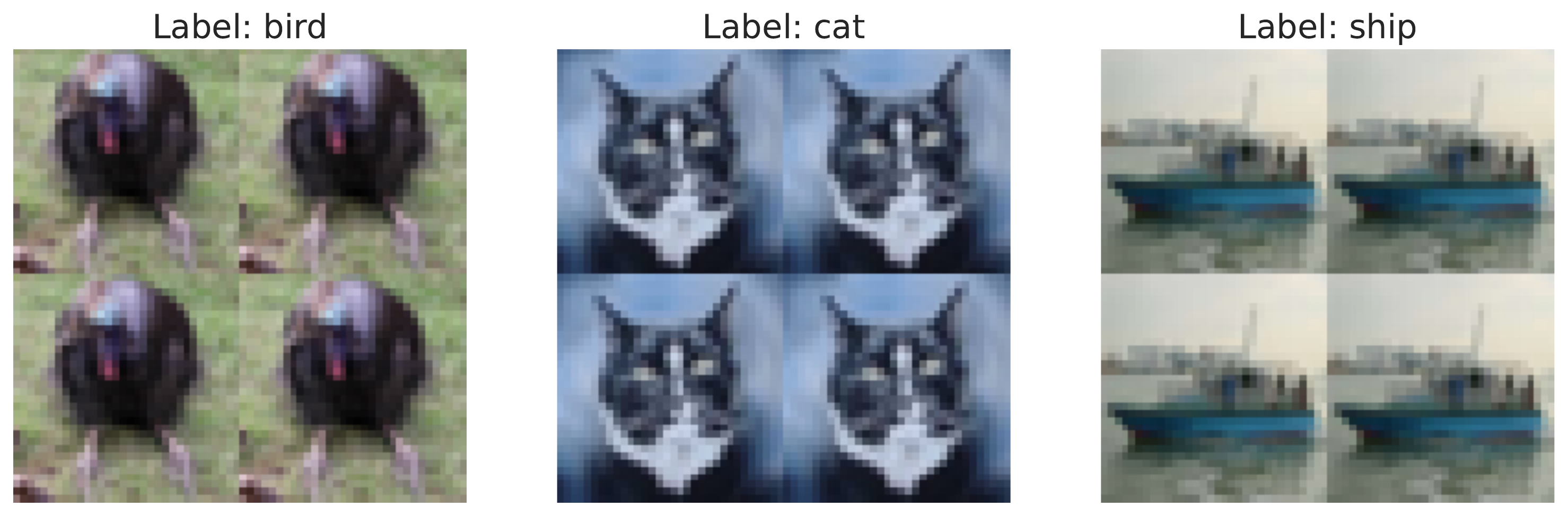}
        \caption{{\footnotesize\ftiled}}
        \label{fig:ftiled-ex}
    \end{subfigure}%
    \quad
    \begin{subfigure}{0.31\textwidth}
        \centering
        \includegraphics[scale=0.15]{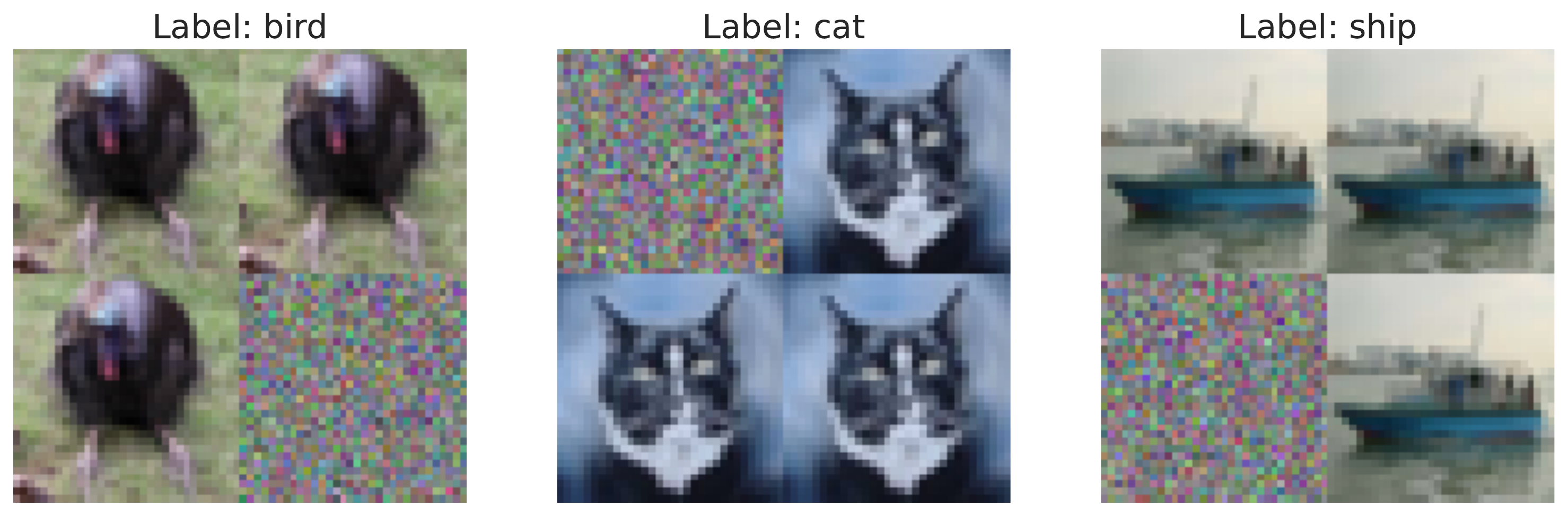}
        \caption{{\footnotesize\ptiled}}
        \label{fig:ptiled-ex}
    \end{subfigure}%
    \quad
    \begin{subfigure}{0.36\textwidth}
        \centering
        \includegraphics[scale=0.15]{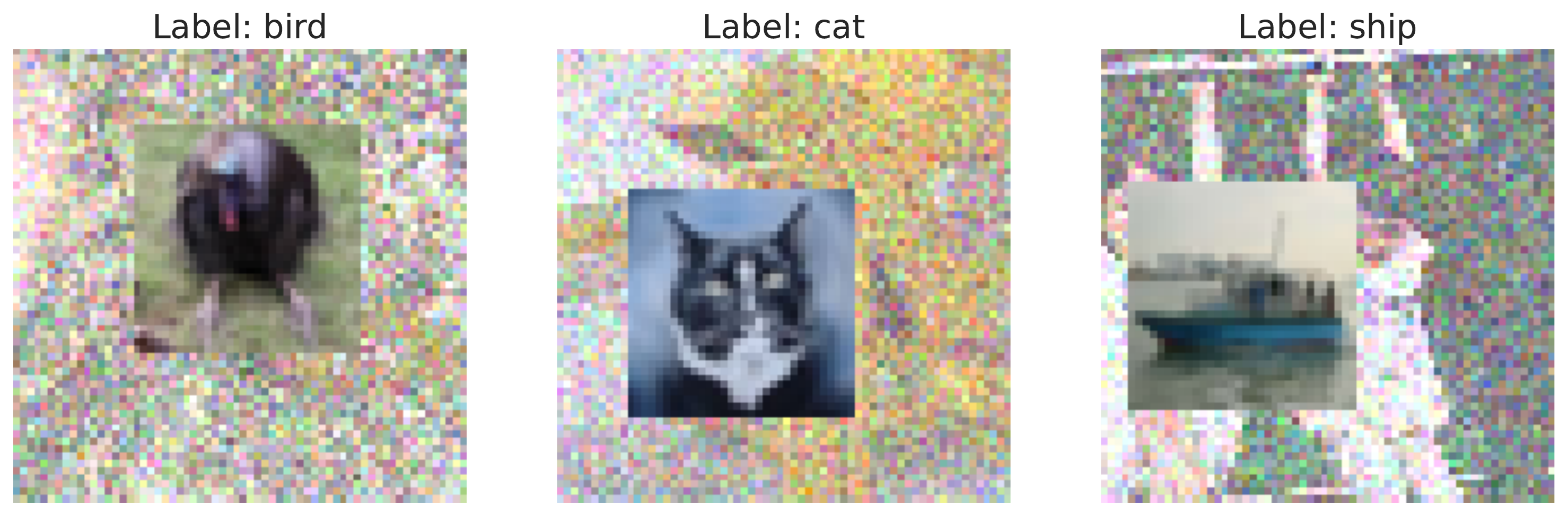}
        \caption{{\footnotesize\embd}}
        \label{fig:embd-ex}
    \end{subfigure}
    \caption{Illustration of different CIFAR-10 based datasets utilized in image classification experiments (cf.~Section~\ref{sec:exp-image}). Note that all three variants correspond to 10-way multiclass classification tasks corresponding to 10 original classes in CIFAR-10.}
    \label{fig:dataset-examples}
\end{figure*}

\begin{figure*}[t!]
    \begin{subfigure}{0.33\textwidth}
        \centering
        \includegraphics[scale=0.37]{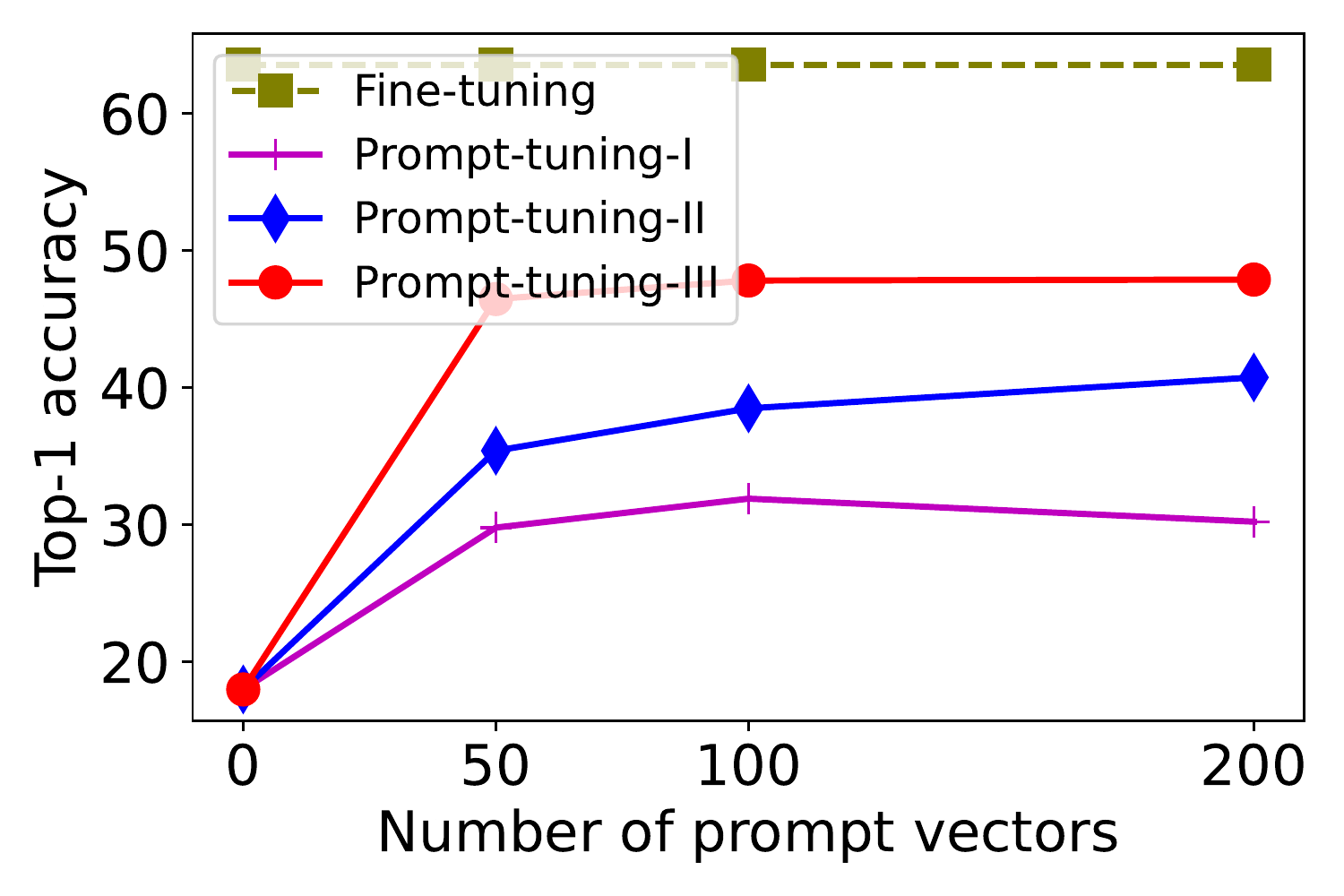}
        \caption{Full dataset}
        \label{fig:embd-res-cap100}
    \end{subfigure}%
    \begin{subfigure}{0.33\textwidth}
        \centering
        \includegraphics[scale=0.37]{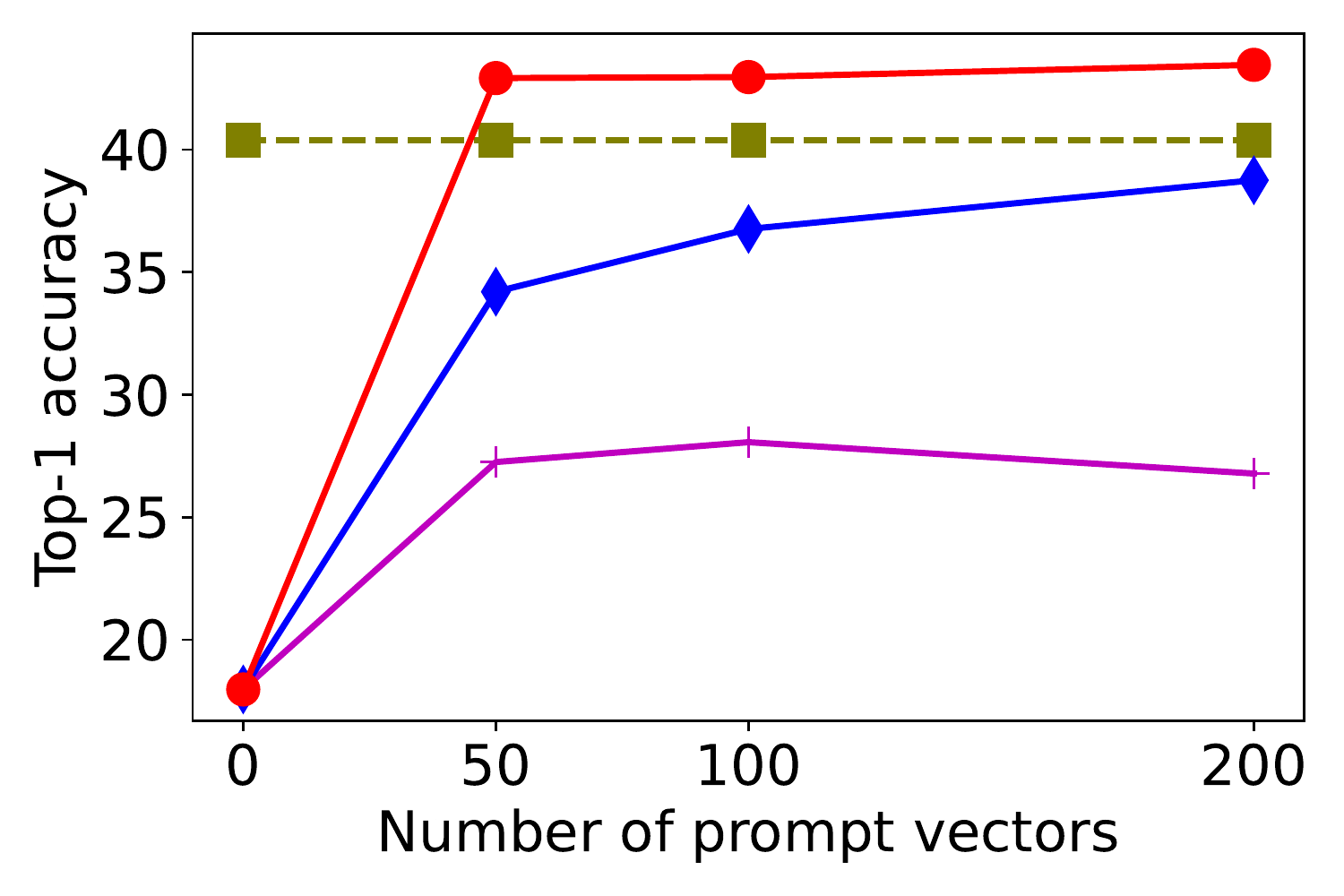}
        \caption{Capped 10\%}
        \label{fig:embd-res-cap10}
    \end{subfigure}%
    \begin{subfigure}{0.33\textwidth}
        \centering
        \includegraphics[scale=0.37]{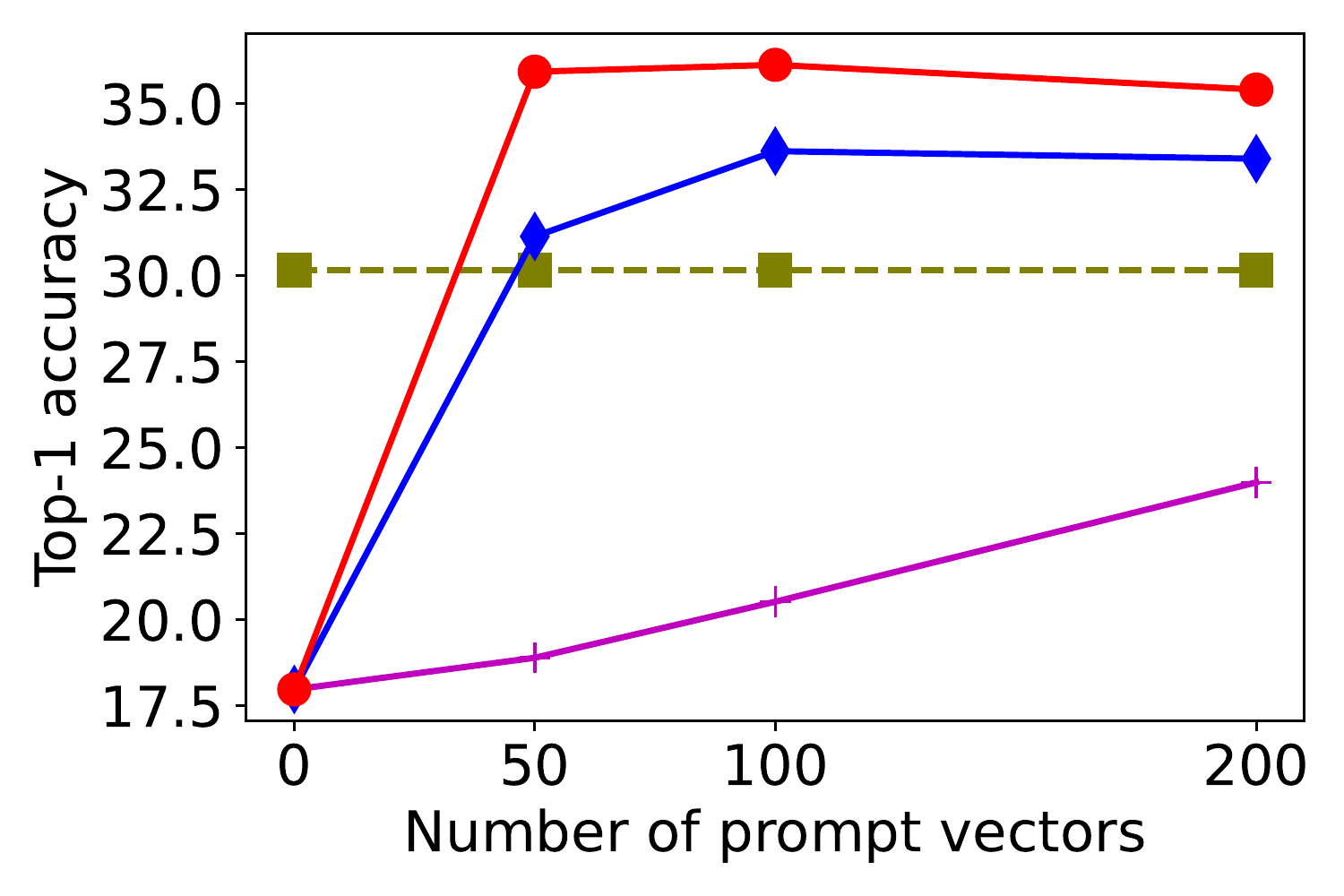}
        \caption{Capped 2\%}
        \label{fig:embd-res-cap2}
    \end{subfigure}
    \caption{Performance of fine-tuning vs. prompt-tuning on 10-way classification tasks defined by \embd dataset. Full dataset has 50K training examples. Capped 10\% and 2\% correspond to sub-sampled \textit{train} sets where we select exactly 500 and 100 examples per class from the full dataset. Note that number of prompt vectors equal to 0 corresponds to \textit{zero-shot} performance.}
    \label{fig:embd-results}
\end{figure*}

\begin{figure*}[t!]
    \begin{subfigure}{0.33\textwidth}
        \centering
        \includegraphics[scale=0.35]{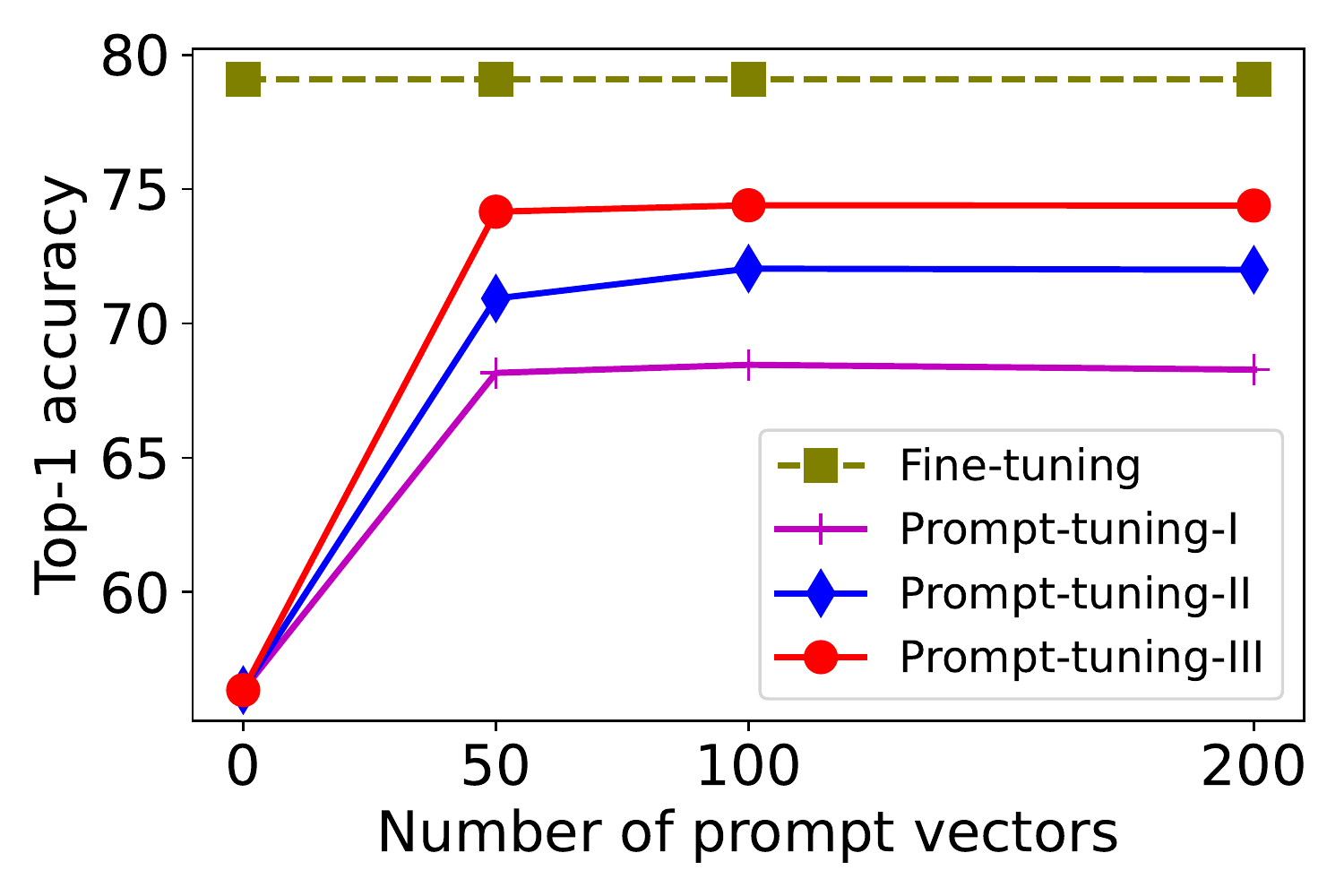}
        \caption{Full dataset}
        \label{fig:ptiled-res-cap100}
    \end{subfigure}%
    \begin{subfigure}{0.33\textwidth}
        \centering
        \includegraphics[scale=0.35]{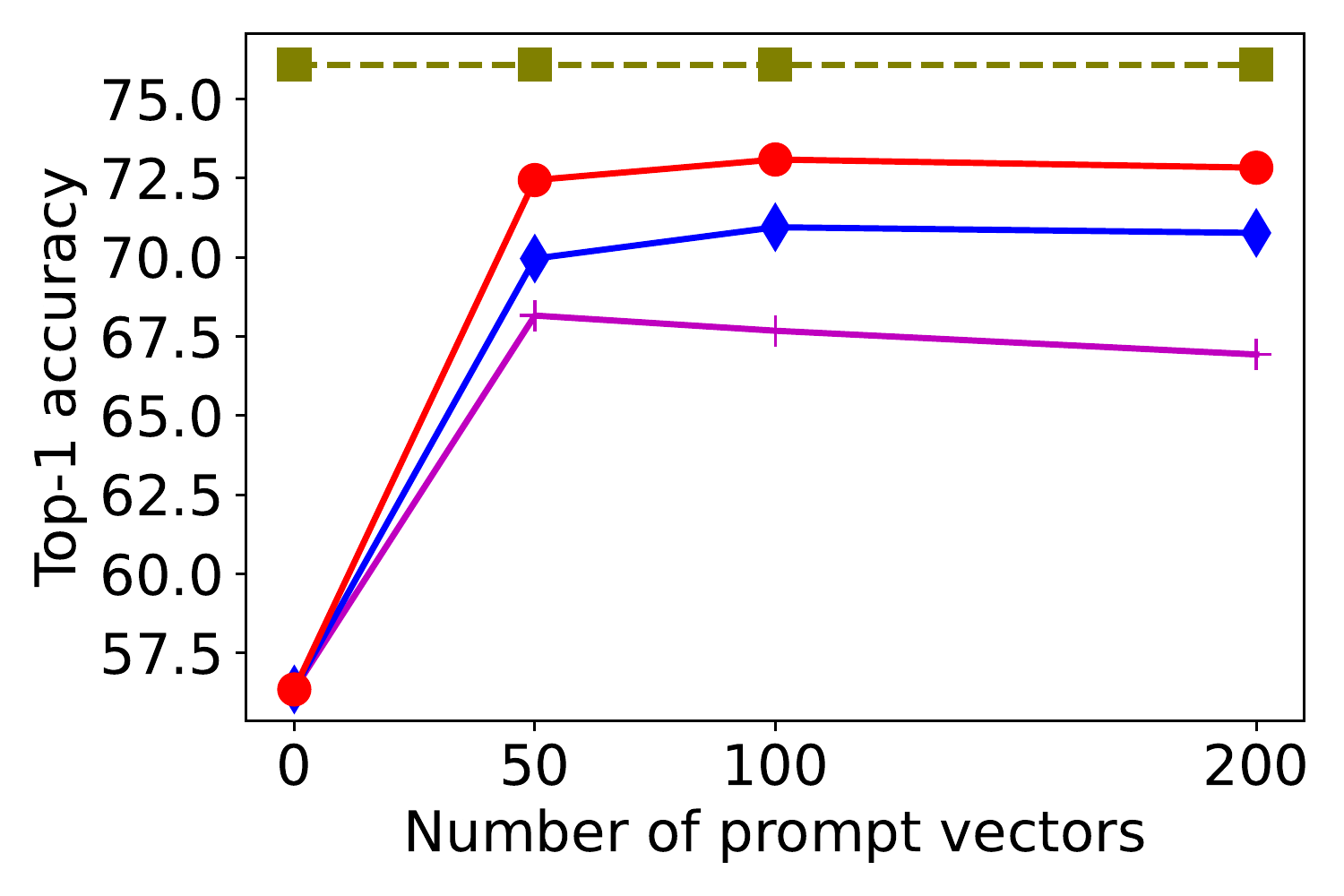}
        \caption{Capped 10\%}
        \label{fig:ptiled-res-cap10}
    \end{subfigure}%
    \begin{subfigure}{0.33\textwidth}
        \centering
        \includegraphics[scale=0.35]{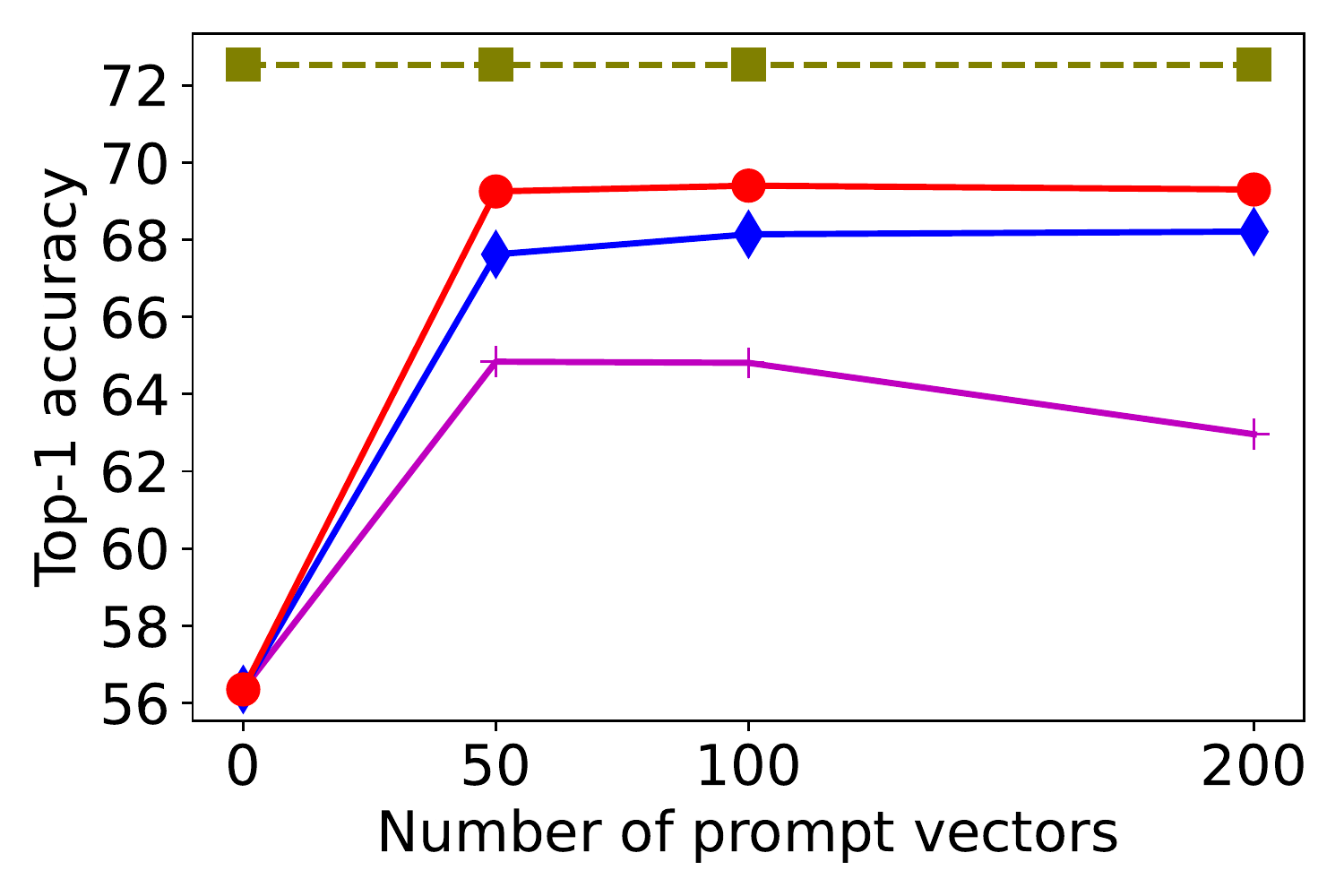}
        \caption{Capped 2\%}
        \label{fig:ptiled-res-cap2}
    \end{subfigure}
    \caption{Performance of fine-tuning vs. prompt-tuning on 10-way classification tasks defined by \ptiled dataset. Full dataset has 50K training examples. Capped 10\% and 2\% correspond to sub-sampled \textit{train} sets where we select exactly 500 and 100 examples per class from the full dataset. Note that number of prompt vectors equal to 0 corresponds to \textit{zero-shot} performance.}
    \label{fig:ptiled-results}
\end{figure*}

\subsection{Image classification experiments}
\label{sec:exp-image}

\subsubsection{Experimental setup}
\label{sec:setup-image}

\textbf{Dataset.} Motivated by our contextual data model, 
we construct the following three datasets based on CIFAR-10~\citep{krizhevsky2009learning} to conduct our evaluation 
(see Fig.~\ref{fig:dataset-examples} for examples). 
%
%
\begin{itemize}
\item {\ftiled.}~Each examples consists of a 64x64 images obtained by arranging a 32x32 image from CIFAR-10 in a tiling pattern with four tiles (cf.~Fig.~\ref{fig:ftiled-ex}).
\item {\ptiled.}~This is dataset is similar to \ftiled with the exception that each image has at-least $T$ out of 4 tiles replaced by patches of i.i.d. random Gaussian noise with mean zero and variance 0.2. Note that, for each example in the dataset, $T \in \{1, 2, 3\}$ is a random number as well as the location of the noisy tiles 
(cf.~Fig.~\ref{fig:ptiled-ex}).  
\item {\embd~\citep{karp2021neurips}.}~We construct an example by simply embedding a 32x32 image from CIFAR-10 at a random location in a 64x64 background corresponding to a randomly selected image from ImageNet~\citep{ILSVRC15}. We also add i.i.d. random Gaussian noise with mean zero and variance 0.2 to the background (cf.~Fig.~\ref{fig:embd-ex}).
\end{itemize}
By construction, each dataset has 50,000 train and 10,000 test examples corresponding to train and test set of CIFAR-10. We also consider data-limited settings where we keep the test set intact but subsample the train set by selecting a fixed number of images for each class. Note that all three datasets define 10-way multiclass classification tasks with CIFAR-10 classes as potential labels. \\

\noindent\textbf{Model architecture.}~We utilize a tiny variant of the Vision transformer model~\citep{dosovitskiy2021vit} for our experiments. This variant has 12 transformer layers with its hidden dimension, MLP intermediate dimension, and number of heads per attention layer being equal to 192, 768, and 3, respectively. The patch size in our study is set to be 4x4. The model itself (without counting the trainable parameters/weights during prompt-tuning) has approximately 5.44M parameters. We rely on the CLS token to obtain the classification logits.

\noindent\textbf{Training.}~We rely on Scenic library~\citep{dehghani2021scenic}\footnote{\url{https://github.com/google-research/scenic}} to conduct our experiments on image classification. Following the default settings in the library along with a coarse grid search, we employ Adam optimizer~\citep{kingma2014adam} with $\beta_1$ = 0.9,  $\beta_2$ = 0.999, weight decay = 0.1, and batch size = 128 while training a \textit{randomly initialized} model. Furthermore, we employ a linear warm-up of learning rate followed cosine learning rate schedule with base learning rate 3e-3. As for the fine-tuning and prompt-tuning experiments that (partially) initialize from a \textit{pre-trained} model, we rely on SGD with momentum parameter 0.9 and batch size = 128 to update trainable parameters. Again, we utilize a linear warm-up of learning rate followed by cosine learning rate schedule. Throughout our experiments, the base learning rates for fine-tuning and prompt-tuning are 1e-3 and 0.1, respectively. \\

\noindent\textbf{Methods.}~In our fine-tuning experiments, we update all pre-trained model parameters. 
As for prompt-tuning, we only update newly introduced (prompt) variables and keep the pre-trained network frozen. We consider three variants of prompt-tuning: 1) \promptt~\citep{lester2021power}, where we add trainable vector between CLS token embedding and first image (patch) embeddings only at the input; 2) \prefixt~\citep{li2021prefix}, where we add the \textit{same} trainable vectors between the CLS embedding and the first image embedding at the input of every transformer layer; and 3) \prefixtfull, where we add \textit{different} trainable vectors between the CLS embedding and the first image patch embedding at the input of every transformer layer. Note that the number of trainable parameters in \promptt and \prefixt do not scale with the number of layers whereas we have linear scaling with number of layers in 
\prefixtfull. 
Interestingly, all three prompt-tuning variant are identical when the number of layers is 1, which corresponds to the setup we theoretically analyzed in the paper. However, they exhibit remarkably different behavior for a multi-layer transformer model, as we show next.  \\

\noindent\subsubsection{Results}
Here, the main goal of our exploration is to highlight the different behavior of fine-tuning and prompt-tuning. We utilize a model trained on \ftiled dataset as the pre-trained model. This model achieves top-1 (in-domain) accuracy of 80.43 on \ftiled test set. In contrast, it achieves \textit{zero-shot} top-1 test accuracy of 56.35 and 17.97 on \ptiled and \embd, respectively. This alludes to the fact that \embd corresponds to a larger distribution shift from the pre-training distribution (\ftiled), as compared to \ptiled.

Fig.~\ref{fig:embd-results} and Fig.~\ref{fig:ptiled-results} 
showcase the performance of fine-tuning and prompt-tuning approaches on \embd and \ptiled, respectively.  
Note that fine-tuning outperforms prompt-tuning in a data-rich setting (cf.~Fig.~\ref{fig:embd-res-cap100} and \ref{fig:ptiled-res-cap100}). This is due to fine-tuning having the ability to update a large number of model parameters (5.4M in our case). In contrast, with $2000$ prompt vectors, \prefixtfull (the most expensive prompt-tuning method out of all three) only updates 460.8K parameters.

Interestingly, in the data-limited regimes, prompt-tuning becomes more competitive. 
In fact, the best performing prompt-tuning method outperforms the fine-tuning (cf.~Fig.~\ref{fig:embd-res-cap10} and \ref{fig:embd-res-cap2}) on \embd, where fine-tuning can easily overfit as it cannot leverage the benefits of the pre-training data due to a large distribution-shift between \ftiled and \embd. 

Part of the performance gap between \prefixtfull and \prefixt can be attributed to the larger number of trainable parameters available to \prefixtfull. 
Note that \prefixt consistently outperforms \promptt, even with the same number of trainable parameters. This alludes to the fact that optimization and architecture choices play a major role beyond just the number of trainable parameters. As mentioned earlier, our theoretical treatment for a single-layer model cannot distinguish among these different prompt-tuning approaches. As a result, we believe that our empirical observations point to multiple interesting avenues for future work.

\subsection{Attention weights for prompt vectors}
\label{sec:relevance-prompt-vecs}

Finally, we explore what role prompt-attention, i.e. the attention weights with prompt vectors as keys and image patches/tokens as values, plays towards underlying task. In Fig.~\ref{fig:image_and_attn_maps}
, we illustrate one representative example. The figure shows how average attention weights from prompt vectors to image tokens/patches evolve across transformer layers, when we employ \prefixtfull. Indeed, the figure verifies that prompt-attention helps locate the relevant tokens/patches from the irrelevant patches, validating our starting hypothesis in Section~\ref{sec:motivation} and \ref{sec:data_model}.
%

\begin{figure*}[h!]
  \begin{minipage}[b]{.3\textwidth}
  \centering
    \begin{subfigure}[b]{\textwidth}
      \centering
      \includegraphics[width=\textwidth,trim={0cm 0cm 0cm 0},clip]{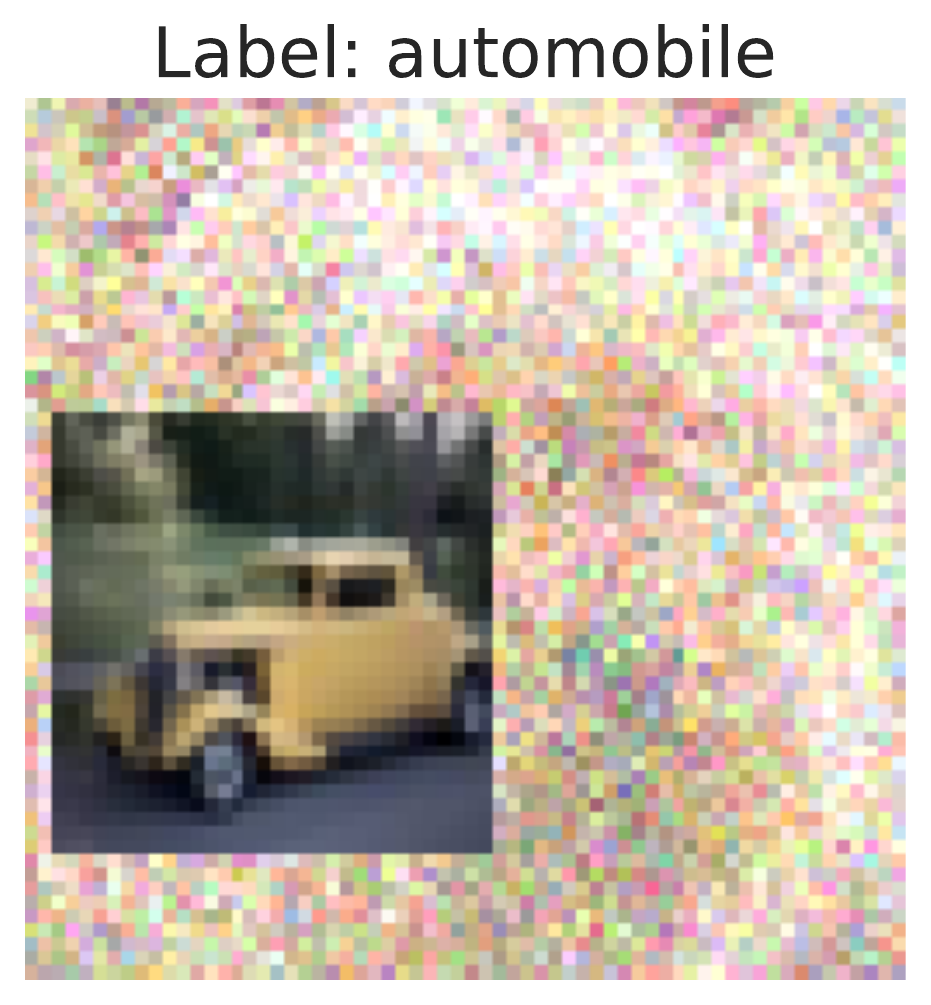}
      \caption{Input image}
      \label{fig:image1}
    \end{subfigure}
  \end{minipage}
  \hfill
  \begin{minipage}[b]{.69\textwidth}
  \begin{subfigure}[b]{\textwidth}
    \centering
    \includegraphics[width=0.8\textwidth,trim={0 0 0 0},clip]{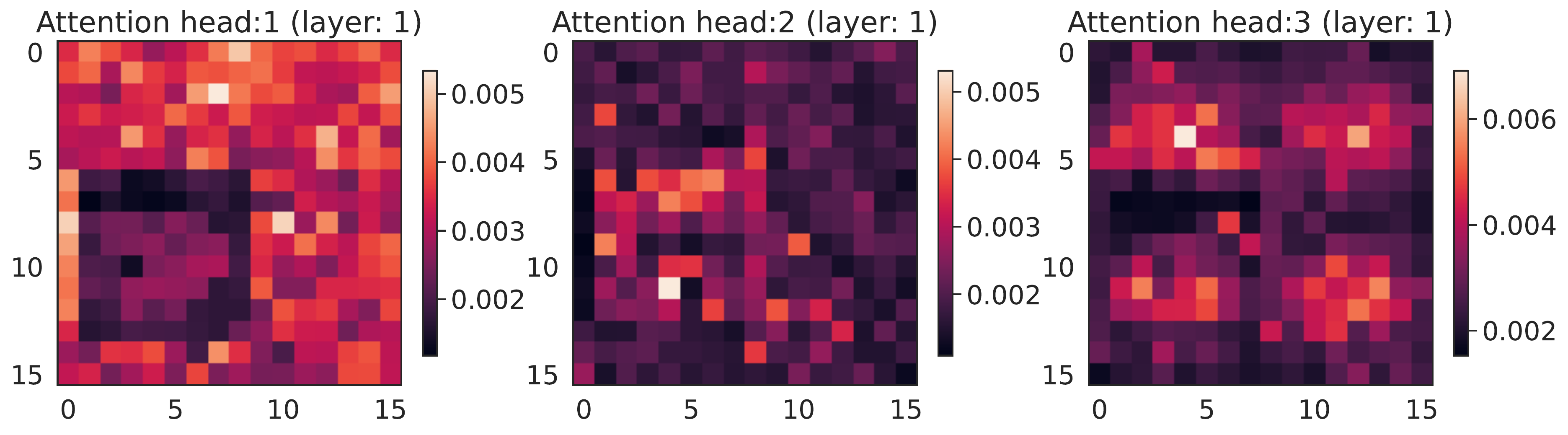}
  \end{subfigure}
    \begin{subfigure}[b]{\textwidth}
    \centering
    \includegraphics[width=0.8\textwidth,trim={0 0 0 0},clip]{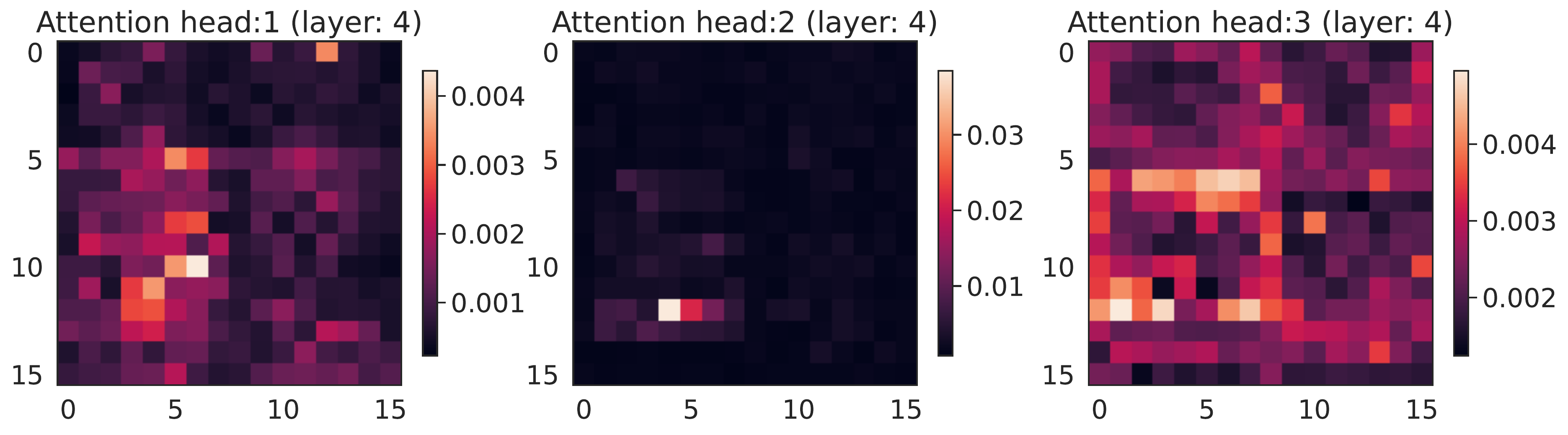}
  \end{subfigure}
    \begin{subfigure}[b]{\textwidth}
    \centering
    \includegraphics[width=0.8\textwidth,trim={0 0 0 0},clip]{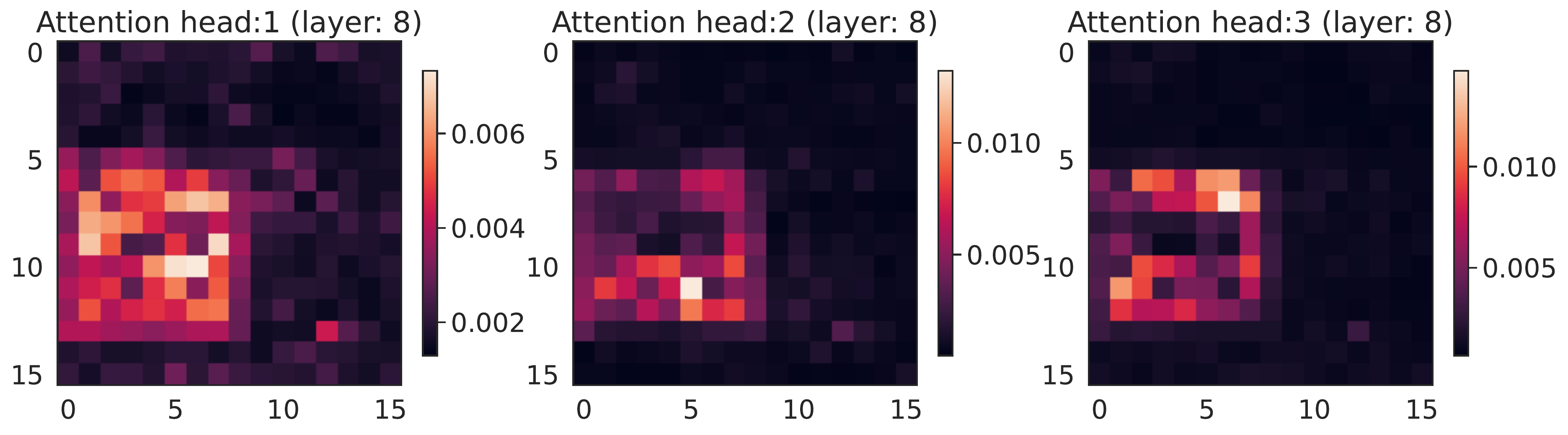}
  \end{subfigure}
    \begin{subfigure}[b]{\textwidth}
    \centering
    \includegraphics[width=0.8\textwidth,trim={0 0 0 0},clip]{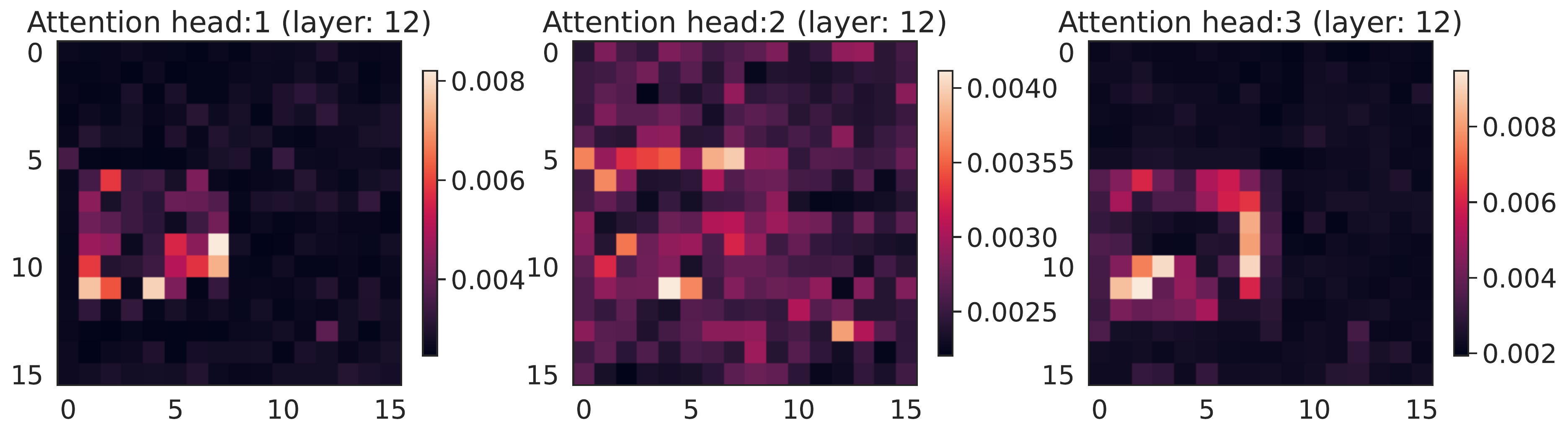}
    \caption{Average attention weights from prompts (keys) to image patches (values).}
    \label{fig:attn_maps1}
  \end{subfigure}
  \end{minipage}%
  \caption{Illustration of how attention weights progressive change from the first layer (Figure~\ref{fig:attn_maps1}-top) to the last layer (Figure~\ref{fig:attn_maps1}-bottom) in the transformer model for a given input image (Figure~\ref{fig:image1}) when we employ \prefixtfull. We plot average attention weights from 50 prompt vectors (keys) to 256 image patches (values). The attention weights for each attention head are naturally arranged in a 16 x 16 grid corresponding to the original locations of the patches in the image. Note that the attention weights in the early layer have a tiling pattern similar to that in \ftiled -- the dataset utilized by the pre-trained model. However, as we progress deeper into the transformer, attention weights begin to capture the relevant patch locations in the dataset of interest, i.e. \embd.}
  \label{fig:image_and_attn_maps}
\end{figure*}


\subsection{Comparison of prompt-tuning and fine-tuning the first self-attention layer}

\begin{table}[t!]
\centering
\caption{Comparison between prompt-tuning and fine-tuning only first layer self-attention weights (Full dataset).}\vspace{2mm}
\begin{tabular}{lrr}
\hline
\textbf{Method (trainable parameters)} & \embd  & \ptiled \\
\hline
\promptt w/ 100 prompt vectors (19.2K) & 31.89 & 68.46 \\
\prefixt  w/ 100 prompt vectors (19.2K) & 38.48 & 72.04 \\
\prefixtfull w/ 100 prompt vectors (230.8K) & 47.81 & 74.40 \\
Fine-tuning only first layer attention weights (148.2K) & 30.35 & 73.95 \\
\hline
\end{tabular}\label{tab single1}
\end{table}

\begin{table}[t!]
\centering
\caption{Comparison between prompt-tuning and fine-tuning only first layer self-attention weights (Capped 10\%).}\vspace{2mm}
\begin{tabular}{lrr}
\hline
\textbf{Method (trainable parameters)}  & \embd  & \ptiled \\
\hline
\promptt w/ 100 prompt vectors (19.2K) & 28.06 & 67.68 \\
\prefixt  w/ 100 prompt vectors (19.2K) & 36.76 & 70.95 \\
\prefixtfull w/ 100 prompt vectors (230.8K) & 42.96 & 73.09 \\
Fine-tuning only first layer attention weights (148.2K) & 18.53 & 70.44 \\
\hline
\end{tabular}\label{tab single2}
\end{table}

\begin{table}[t!]
\centering
\caption{Comparison between prompt-tuning and fine-tuning only first layer self-attention weights (Capped 2\%).}\vspace{2mm}
\begin{tabular}{lrr}
\hline
\textbf{Method (trainable parameters)}  & \embd  & \ptiled \\
\hline
\promptt w/ 100 prompt vectors (19.2K) & 20.52 & 64.81 \\
\prefixt w/ 100 prompt vectors (19.2K) & 33.62 & 68.14 \\
\prefixtfull w/ 100 prompt vectors (230.8K) & 36.13 & 69.40 \\
Fine-tuning only first layer attention weights (148.2K) & 15.46 & 65.04 \\
\hline
\end{tabular}\label{tab single3}
\end{table}

In Tables \ref{tab single1}, \ref{tab single2}, and \ref{tab single3}, we explored fine-tuning only first layer self-attention parameters for the underlying ViT model. This setting aligns well with the single-layer nature of our theoretical results. Similar to Fig.~\ref{fig:embd-results} (corresponding to \embd dataset) and Fig.~\ref{fig:ptiled-results} (corresponding to \ptiled dataset), we considered three settings: 1) Full dataset; 2) Capped 10\%; and 3) Capped 2\%, which progressively corresponds to smaller amount of (training) data during fine-tuning and prompt-tuning. 

The key takeaways are: 
\begin{enumerate}
    \item When there is a significant distribution-shift between from pre-training data (in case of \embd), even the simplest prompt-tuning, namely \promptt, significantly outperforms the fine-tuning first layer self-attention parameters.
    \item When the distribution-shift is small, prompt-tuning variants realize a better \textit{accuracy vs. training cost} trade-off, e.g. \prefixt outperforms fine-tuning first layer self-attention parameters in the Capped 10\% and Capped 2\% setting (while training only 19.2K rather than 148.2K parameters).
\end{enumerate}

\vspace{-6pt}
\section*{Discussion}\vspace{-3pt}
In light of remarkable success of attention architectures, we initiate a theoretical investigation of one of its core mechanisms: prompt-tuning. For a one-layer attention model, we motivate and analyze a simplified model for prompt-tuning, which we coin \prml. Through this model, we developed new statistical foundations for gradient-based prompt-tuning, characterized its optimization and generalization dynamics, and explored how it facilitates attending to context-relevant information. We also showed that under \eqref{eq:CGMM} model one-layer softmax-prompt-tuning can provably be superior to alternatives including one layer self-attention. Thorough experiments support our theory on how prompt-tuning attends to the context and how it can potentially outperform full fine-tuning. Our results also suggest many interesting future directions including 1) extension to deeper architectures by characterizing the role of softmax-attention in individual layers, 2) developing a stronger theoretical and empirical understanding of when/if prompt-tuning is superior to fine-tuning, 3) extending our model to include multiple prompt vectors (and perhaps extending \eqref{eq:CGMM} to include multiple context vectors), and 4) investigating the role of multiple attention heads in prompt-tuning.

\bibliographystyle{plainnat}
\bibliography{refs,transformers}

\newpage
\onecolumn
\appendix

\section{Sharp characterization of accuracy under known context}\label{sharp characterization} 
While the discrete dataset model is insightful, incorporating noise is crucial for understanding the fundamental limits of the benefits of context in attention. To this end, let us focus on the core dataset model where we set $\delq=\delw=0$ and explore the role of noise in population accuracy. Also assume that noise is standard normal, i.e. Assumption \ref{ass:noise Gaussian}. 

\noindent$\bullet$  \emph{Linear model.} The linear model aggregates tokens to obtain a simple Gaussian mixture distribution. Specifically, aggregated tokens are exactly distributed as $\frac{1}{T}\X^\top\ones_T\sim \Nn(\zeta\wstar,\frac{1-\zeta}{T}\Iden)$, 
leading to the following well-known result.
\begin{fact} \label{lin func fact}For linear models, optimal accuracy obeys $\min_{\w} \acc(\flin(\w))=Q(\sqrt{\frac{\zeta^2W^2T}{1-\zeta}})$ where $Q(\cdot)$ is the tail function of the standard normal distribution.
\end{fact}



\noindent$\bullet$ \emph{\Prml model.} Since \prml strictly generalizes the linear model, its accuracy is at least as good. The theorem below quantifies this and demonstrates how \ctx vector can enable an optimal weighting of relevant and irrelevant tokens to maximize accuracy. A general version of this theorem is proven under a non-asymptotic setting (finite $T,d$) as Theorem \ref{thm sharp and precise appendix}.

\begin{theorem}\label{thm sharp and precise} Consider the \prml model $\fat_{\thetab}$. Suppose $\wstar\perp\qstar$ and let $\taut,\taub>0$ be hyperparameters. Consider the following algorithm which uses the hindsight knowledge of $\qstar$ to estimate $\wstar$ and make prediction:

\noindent Set $\hat{\w}=(\Iden-\qstab\qstab) \nabla \Lc_{\w}(0,\taut\qstab)$ and $\thetab=(\hat{\w},\taub\qstab)$.
Suppose $\zeta^2W^2T,1-\zeta,\alpha:=n/d,e^{Q^2},e^\taut$ each lie between two positive absolute constants. Suppose $T$ is polynomially large in $n$ and these constants and $\ordet{\cdot}$ hides polynomial terms in $n$. Define \emph{inverse-signal-to-noise-ratio}: $\isnr{\alpha,\taut}=\frac{(1-\zeta)e^{2\taut(\taut-Q)}}{\alpha\zeta^2W^2T}$. 
In the limit $T,d\rightarrow\infty$, the test error converges in probability to $\Qc\left(\frac{e^{Q\taub-\taub^2}}{\sqrt{1+\isnr{\alpha,\taut}}}\cdot \sqrt{\frac{\zeta^2W^2T}{1-\zeta}}\right)$. In this limit, optimal hyperparameters are $\taut=\taub=Q/2$ and leads to optimal $\isnr{\alpha}:=\frac{(1-\zeta)e^{-Q^2/2}}{\alpha\zeta^2W^2T}$ and the error\vspace{-0pt}
\[
\err(\alpha,Q,W)=\Qc\left(\frac{e^{Q^2/4}}{\sqrt{1+\isnr{\alpha}}}\cdot \sqrt{\frac{\zeta^2W^2T}{1-\zeta}}\right)
\]
\end{theorem}

Here, a few remarks are in place. Note that $\ratl:=\sqrt{\frac{\zeta^2W^2T}{1-\zeta}}$ term is the population error rate of $\flin$ from Fact \ref{lin func fact}. In the limit $\alpha=n/d\rightarrow \infty$, the rate of $\fat$ is simply given by $e^{Q/4}\ratl$ demonstrating the strict superiority of prompt-attention. Moreover setting $Q=0$ (no prompt info), since feature-output of $\flin$ (i.e.~$\X^\top \phi(\X\ones)$) is (essentially) a binary Gaussian mixture distribution, our error-rate recovers the Bayes-optimal $\flin$ classifier which has a finite-sample rate of $\ratl/\sqrt{1+(1-\zeta)/(\alpha\zeta^2W^2T^2)}$. Prompt-tuning also strictly improves this because our $\isnr{\alpha}$ introduces an additional $e^{-Q^2/2}$ multiplier.

\section{Gradient calculations and concentration}

In this section, we focus on finite-sample analysis of Algorithm \ref{eq:algo}. Introduce the following shorthand notation analogous to the population counterparts in Section \ref{sec:population analysis main}:
\begin{align}
&\widehat\G_\qb(\qb,\w):=-\nabla_{\qb}\Lch_{\Dc}(\thetab) = \frac{1}{n}\sum_{i\in[n]} (y_i-f_{\thetab}(\X_i)) \X_i^\top\sftd{\X_i\qb}\,\X_i \w\,, \nn
\\
&\widehat\G_\w(\qb,\w):=-\nabla_{\w}\Lch_{\Dc}(\thetab) = \frac{1}{n}\sum_{i\in[n]} (y_i-f_{\thetab}(\X_i)) \X_i^\top\sft{(\X_i\qb)} .
\end{align}

\subsection{Gradient Calculations}
We begin with the gradient calculations for the first two steps of the algorithm.

For convenience, we make use of the following shorthands
\begin{subequations}\label{eq:gradient_term_notations}
\begin{align*}
\corq&:=\corqarg{\w}:=\w^\top\qb_\star,
\qquad\corw:=\corwarg{\w}:=\w^\top\w_\star,\qquad\alpha_i:=\alpha(\w,y_i):=\corq+y_i\corw, \\
\gammab_i&:=\gammab(\Z_i)=\frac{1}{T}\Z_i^\top\ones,\qquad
\beta_i:=\beta(\Z_i;\w)=\gammab_i^\top\w, \qquad
\hat\Sigmab_i:=\frac{1}{T}\Z_i^\top\Z_i\,,
\end{align*}
\end{subequations}
where $\Z_i\in\R^{(1-\zeta)T \times d}$ is the matrix of irrelevant tokens $\zti, t\in\Rcc$ for sample $i\in[n].$

\begin{lemma}\label{lem:w_der_finite}
 Under dataset model \eqref{eq:CGMM} and Assumption \ref{ass:noise}, we have
\begin{align}
\Ghat_\w(0,0)
&=  \zeta \w_\star + \zeta  \qb_\star \left(\frac{1}{n}\sum_{i\in[n]}y_i\right)+ \frac{1}{n}\sum_{i}y_i\gammab_i.
\end{align}
\end{lemma}

\begin{lemma}\label{lem:q_der_finite} Under dataset model \eqref{eq:CGMM} and Assumption \ref{ass:noise}, we have that
\begin{align}\label{eq:q_der_finite}
\Ghat_\qb(0,\w):=\frac{1}{n}\sum_{i=1}^{n}\left(y_i - \zeta \alpha_i-\beta_i\right) \left[
\left((\zeta-\zeta^2)\alpha_i-\zeta\beta_i\right)\left(\qstar + y_i\wstar\right)
+ \hat\Sigmab_i\w
- (\zeta\alpha_i+\beta_i)\gammab_i
\right]\,.
\end{align}
\end{lemma}

\subsubsection{Proof of Lemma \ref{lem:w_der_finite}}
By direct computation,
\begin{align}
\Ghat_\w(0,0):=-\nabla_{\w}\Lch_{\Dc}(\thetab) 
&= \frac{1}{nT}\sum_{i\in[n]}y_i\X_i^\top\ones_T = \frac{1}{nT}\sum_{i\in[n]}\sum_{t\in[T]}y_i\x_{i,t} \nn
\\
&=\frac{\zeta}{n}\big(\sum_{i\in[n]}y_i\big) \qb_\star
+\frac{\zeta}{n} \w_\star
+\frac{1}{n}\sum_{i\in[n]}y_i\gammab_i\,. \nn
\end{align}
\subsubsection{Proof of Lemma \ref{lem:q_der_finite}}
Note that $\phi^\prime(0)=\frac{1}{T}\Id-\frac{1}{T^2}\ones\ones^\top$; hence, for $\thetab=(0,\w)$:
\begin{align*} 
\Ghat_\qb(0,\w):=-\nabla_{\qb}\Lch_{\Dc}(\thetab) 
&= \frac{1}{n}\sum_{i\in[n]} \underbrace{\frac{1}{T}(y_i-f_{\thetab}(\X_i)) \X_i^\top\X_i\w}_{\term{1,i}} - \frac{1}{n}\sum_{i\in[n]} \underbrace{{\frac{1}{T^2}(y_i-f_{\thetab}(\X_i)) \X_i^\top\ones\ones^\top\X_i \w}}_{\term{2,i}}.
\end{align*}
Moreover, note that, 
\begin{align}
f_{\thetab}(\X_i)&=\frac{1}{T}\w^\top\X_i^\top\ones, \nn \\
\X_i^\top\X_i &= \zeta T(\qb_\star + y_i \w_\star) (\qb_\star + y_i \w_\star)^\top + \Z_i^\top\Z_i,\nn\\
\X_i^\top\ones &= \zeta T (\qb_\star + y_i \w_\star) + \Z_i^\top\ones\,,\nn
\end{align}
where recall the notation in Lemma \ref{lem:w_der_finite} for $\Z_i.$ Hence, using the lemma's notation (repeated here for convenience)
\begin{align*}
\alpha_i&:=\corq+y_i\corw, \quad
\beta_i:=\beta(\Z_i;\w):=\frac{1}{T}\ones^\top\Z_i\w, \quad
\gammab_i:=\gammab(\Z_i)=\frac{1}{T}\Z_i^\top\ones,\quad
\hat\Sigmab_i:=\frac{1}{T}\Z_i^\top\Z_i.
\end{align*}
we find that
\begin{align*}
y_i-f_{\thetab}(\X_i) &= y_i - \zeta \alpha_i-\beta_i
\\
\frac{1}{T}\X_i\X_i^\top\w&=\zeta \alpha_i\qstar + \zeta y_i \alpha_i \wstar+\hat\Sigmab_i\w
\\
\frac{1}{T^2}{\X_i^\top\ones\ones^\top\X_i\w}&=\zeta\left(\zeta  \alpha_i+\beta_i\right) \qstar + \zeta\left(\zeta  \alpha_i+\beta_i\right) y_i \wstar 
+\left(\zeta  \alpha_i+\beta_i\right)\gammab_i.
\end{align*}

With the above, each one of the two terms becomes:

\begin{align*}
\term{1,i}= \left(y_i - \zeta \alpha_i-\beta_i\right)\zeta \alpha_i \qstar + 
\left(y_i - \zeta \alpha_i-\beta_i\right)\zeta  \alpha_i y_i \wstar + \left(y_i - \zeta \alpha_i-\beta_i\right)\hat\Sigmab_i\w
\end{align*}

\begin{align*}
\term{2,i}= \zeta \left(y_i - \zeta \alpha_i-\beta_i\right) \left(\zeta\alpha_i+\beta_i\right)\qstar + 
\zeta \left(y_i - \zeta \alpha_i-\beta_i\right) \left(\zeta\alpha_i+\beta_i\right) y_i\wstar + \left(y_i - \zeta \alpha_i-\beta_i\right) \left(\zeta\alpha_i+\beta_i\right)\gammab_i\,.
\end{align*}

Combining the above:

\begin{align}
{\term{1,i}-\term{2,i}}=\left(y_i - \zeta \alpha_i-\beta_i\right) \left[
\zeta\left((1-\zeta)\alpha_i-\beta_i\right)\left(\qstar + y_i\wstar\right)
+ \hat\Sigmab_i\w
- (\zeta\alpha_i+\beta_i)\gammab_i
\right]\,.
\end{align}

\subsection{Concentration of Gradient \texorpdfstring{$\Ghat_\qb{(0,\w)}$}{} in the \texorpdfstring{$\qb$}{context} direction}

The main result of this section is the following lemma about concentration of gradient with respect to $\qb$.

\begin{lemma}[Concentration of $\Ghat_\qb(0,\w)$]\label{lem:conc_grad_q}
Fix any vectors $\vb,\w\in\R^d$. For convenience define $\corqv:=\vb^\top\qstar$ and $\corwv:=\vb^\top\wstar$ and recall $\corw,\corq$ notations from Lemma \ref{lem:q_der_finite}. Then, we can decompose
\[
\vb^\top\Ghat_\qb(0,\w) =\vb^\top\G_\qb(0,\w) + \vb^\top\Gtilde_\qb(0,\w),
\]
where the expectation term is given by 
    \begin{align}
\vb^\top\G_\qb(0,\w):=\E\left[\vb^\top\Ghat_\qb(0,\w)\right] 
&= \left((\zeta-\zeta^2)\left(\corw+{\w^\top\Sigmab\w}/{T}\right)-(\zeta^2-\zeta^3)\left(\corw^2 + \corq^2\right)\right)\,\corqv\nn
\\ 
&
+\left(\left((\zeta-\zeta^2)-2(\zeta^2-\zeta^3)\corw\right)\corq \right)\,\corwv -\left(\left(1+{2}/{T}\right)(\zeta-\zeta^2)\corq\right)\vb^\top\Sigmab\w\,,
\label{eq:mean_one_step_grad_q}
\end{align}
and the deviation term obeys
\begin{align}
\vb^\top\Gtilde_\qb(0,\w) 
&=\left[\left((\zeta-\zeta^2)-2(\zeta^2-\zeta^3)\corw\right)\corq\corqv+\left((\zeta-\zeta^2)\corw-(\zeta^2-\zeta^3)\left(\corw^2+\corq^2\right)\right)\corwv\right]\left(\frac{1}{n}\sum_{i=1}^{n}y_i\right) \nn
\\
&\qquad+\left[\left(-\zeta+2\zeta^2\right)\corq\corqv+\left(-\zeta+(-\zeta+2\zeta^2)\corw\right)\corwv+(1-\zeta)\w^\top\Sigmab\vb\right]\left(\frac{1}{n}\sum_{i=1}^{n}\left(\gammab_i^\top\w\right)\right) \nn
\\
&\qquad\qquad\qquad+\left[\zeta\corw - \zeta^2(\corq^2+\corw^2)+ \frac{(1-\zeta)}{T}\w^\top\Sigmab\w\right] \left( \frac{1}{n}\sum_{i=1}^{n}\vb^\top\gammab_i\right) \nn
\\
&\qquad\qquad\qquad+\left[\left(-\zeta+(-\zeta+2\zeta^2)\corw\right)\corqv+\left(-\zeta+2\zeta^2\right)\corq\corwv\right]\left(\frac{1}{n}\sum_{i=1}^{n}y_i\left(\gammab_i^\top\w\right)\right) \nn
\\
&\qquad\qquad\qquad+\left[\zeta\corq -2 \zeta^2\corq\corw\right]\left( \frac{1}{n}\sum_{i=1}^{n}y_i\left(\vb^\top\gammab_i\right)\right) \nn
\\
&\qquad
+\zeta\corqv \left(\frac{1}{n}\sum_{i=1}^{n}\left(\gammab_i^\top\w\right)^2-\frac{(1-\zeta)}{T}\w^\top\Sigmab\w\right) \nn
\\
&\qquad\qquad\qquad+\zeta\corwv \left(\frac{1}{n}\sum_{i=1}^{n}\left(\gammab_i^\top\w\right)^2y_i\right) \nn
\\
&\qquad\qquad\qquad+\left[1-2\zeta\corw\right] \left( \frac{1}{n}\sum_{i=1}^{n}y_i\left(\w^\top\gammab_i\right)\left(\vb^\top\gammab_i\right)\right) \nn
\\
&\qquad+\left(1-\zeta\corw\right)\left(\frac{1}{n}\sum_{i=1}^{n}y_i\vb^\top\hat\Sigmab_i\w \right) \nn
\\
&\qquad\qquad\qquad- 
\zeta \corq \left(\frac{1}{n}\sum_{i=1}^{n}\vb^\top\hat\Sigmab_i\w-(1-\zeta)\vb^\top\Sigmab\w\right) \nn
\\
&\qquad-\left[2\zeta\corq\right]\left( \frac{1}{n}\sum_{i=1}^{n}\left(\w^\top\gammab_i\right)\left(\vb^\top\gammab_i\right)-\frac{1-\zeta}{T}\vb^\top\Sigmab\w\right) \nn
\\
&\qquad\qquad\qquad- \left(\frac{1}{n}\sum_{i=1}^{n}\left(\vb^\top\gammab_i\right)\left(\left(\w^\top\gammab_i\right)^2-\frac{(1-\zeta)}{T}\w^\top\Sigmab\w\right)\right) \nn 
\\
&\qquad-\left(\frac{1}{n}\sum_{i=1}^{n}\left(\gammab_i^\top\w\right)\left(\vb^\top\hat\Sigmab_i\w-(1-\zeta)\w^\top\Sigmab\vb\right)\right).\label{eq:Ghat full}
\end{align}
Moreover, all random terms in \eqref{eq:Ghat full} are zero-mean and concentrate as prescribed by
Lemma \ref{lem:psi} below\,.
\end{lemma}

\begin{lemma}[Main concentration lemma]\label{lem:psi}
Let $y_i, i\in[n]$ be iid Rademacher random variables. Let $\Z_i\in\R^{(1-\zeta)T \times d}, i\in[n]$ be iid copies of a random matrix $\Z$. Each row  $\z_t, t\in[(1-\zeta)T]$ of $\Z$ is an iid copy of a random vector $\z$ satisfying Assumption \ref{ass:noise}. For convenience denote $\gammab_i:=\Z_i^\top\ones/T$ and $\hat\Sigmab_i:=\Z_i^\top\Z_i/T$. Then, the following statements are true for all vectors $\w,\vb\in\R^{d}$:
\begin{align*}
\psinorm{\frac{1}{n}\sum_{i\in[n]}y_i}{2}&\leq \frac{C}{\sqrt{n}}
\end{align*}
\begin{align*}
\psinorm{\frac{1}{n}\sum_{i\in[n]}\gammab_i^\top\w}{2} \vee \psinorm{\frac{1}{n}\sum_{i\in[n]}y_i\gammab_i^\top\w}{2} &\leq \frac{C\sigma\sqrt{1-\zeta}\|\w\|_2}{\sqrt{nT}}
\end{align*}
\begin{align*}
\psinorm{\frac{1}{n}\sum_{i\in[n]}\left(\gammab_i^\top\w\right)\left(\gammab_i^\top\vb\right) - \frac{1-\zeta}{T}\vb^\top\Sigmab\w}{1} \vee \psinorm{\frac{1}{n}\sum_{i\in[n]}y_i\left(\gammab_i^\top\w\right)\left(\gammab_i^\top\vb\right)}{1} &\leq \frac{C\sigma^2(1-\zeta)\|\w\|_2\|\vb\|_2}{T\sqrt{n}}
\end{align*}
\begin{align*}
\psinorm{\frac{1}{n}\sum_{i\in[n]}\w^\top\hat\Sigmab_i\vb-{(1-\zeta)}\w^\top\Sigmab\vb}{1} \vee \psinorm{\frac{1}{n}\sum_{i\in[n]}y_i\w^\top\hat\Sigmab_i\vb  }{1}&\leq \frac{C\sigma^2\sqrt{1-\zeta}\|\w\|_2\|\vb\|_2}{\sqrt{nT}}
\end{align*}
\begin{align*}
\psinorm{\frac{1}{n}\sum_{i\in[n]}\left(\left(\gammab_i^\top\w\right)^2-\frac{1-\zeta}{T}{\w^\top\Sigmab\w}\right)\gammab_i^\top\vb}{2/3}&\leq \frac{C\sigma^3(1-\zeta)^{3/2}\|\w\|_2^2\|\vb\|_2}{T^{3/2}\sqrt{n}}\log n
\end{align*}
\begin{align*}
\psinorm{\frac{1}{n}\sum_{i\in[n]}\left(\gammab_i^\top\w\right)\left(\w^\top\hat\Sigmab_i\vb-(1-\zeta)\w^\top\Sigmab\vb\right)}{2/3} &\leq \frac{ C \sigma^3({1-\zeta})\|\w\|_2^2\|\vb\|_2}{{T\sqrt{n}}}\log n\,.
\end{align*}
Also, all the random variables that appear above are zero mean.
\end{lemma}

\subsubsection{Proof of Lemma \ref{lem:conc_grad_q}}

We split \eqref{eq:q_der_finite} in four terms and handle each of them separately.

\noindent\textbf{$\bullet$~~ $\term{I}=\frac{1}{n}\sum_{i=1}^{n}\left(y_i - \zeta \alpha_i-\beta_i\right) 
\left((\zeta-\zeta^2)\alpha_i-\zeta\beta_i\right)\qstar$}\\
We first focus on
\begin{align*}
\term{I}=\frac{1}{n}\sum_{i=1}^{n}\left(y_i(1-\zeta\corw) -\beta_i- \zeta \corq\right) 
\left(y_i(\zeta-\zeta^2)\corw-\zeta\beta_i+(\zeta-\zeta^2)\corq\right)\qstar=:A\qstar.
\end{align*}
We we can express $A$ above conveniently as follows (recall $y_i^2=1$):
\begin{align*}
A&:=-\zeta(\zeta-\zeta^2)\corq^2+(1-\zeta\corw)(\zeta-\zeta^2)\corw +\left((1-\zeta\corw)(\zeta-\zeta^2)\corq-\zeta(\zeta-\zeta^2)\corw\corq\right)\left(\frac{1}{n}\sum_{i=1}^{n}y_i\right)
\\
&\qquad\qquad+\left(-(\zeta-\zeta^2)\corq+\zeta^2\corq\right)\left(\frac{1}{n}\sum_{i=1}^{n}\beta_i\right)
+\left(-(1-\zeta\corw)\zeta-(\zeta-\zeta^2)\corw\right)\left(\frac{1}{n}\sum_{i=1}^{n}y_i\beta_i\right)
+\zeta \left(\frac{1}{n}\sum_{i=1}^{n}\beta_i^2\right)
\\
&=(\zeta-\zeta^2)\corw-(\zeta^2-\zeta^3)\left(\corw^2+\corq^2\right)+\left((\zeta-\zeta^2)-2(\zeta^2-\zeta^3)\corw\right)\corq\left(\frac{1}{n}\sum_{i=1}^{n}y_i\right)
\\
&\qquad\qquad+\left(-\zeta+2\zeta^2\right)\corq\left(\frac{1}{n}\sum_{i=1}^{n}\beta_i\right)
+\left(-\zeta+(-\zeta+2\zeta^2)\corw\right)\left(\frac{1}{n}\sum_{i=1}^{n}y_i\beta_i\right)
\\
&\qquad\qquad+\zeta \left(\frac{1}{n}\sum_{i=1}^{n}\beta_i^2-\frac{(1-\zeta)}{T}\w^\top\Sigmab\w\right)+\frac{(\zeta-\zeta^2)}{T}\w^\top\Sigmab\w.
\end{align*}
From Lemma \ref{lem:psi}, all random terms above are zero mean. Hence,
\begin{align}
\E[A]&=-(\zeta^2-\zeta^3)\corq^2+(1-\zeta\corw)(\zeta-\zeta^2)\corw+(\zeta-\zeta^2)\frac{\w^\top\Sigmab\w}{T} \nn
\\
&=-(\zeta^2-\zeta^3)\corq^2+(\zeta-\zeta^2)\left(\corw+{\w^\top\Sigmab\w}/{T}\right)-(\zeta^2-\zeta^3)\corw^2 \nn
\\
&=(\zeta-\zeta^2)\left(\corw+{\w^\top\Sigmab\w}/{T}\right)-(\zeta^2-\zeta^3)\left(\corw^2 + \corq^2\right).\label{eq:Exp_term1}
\end{align}

\vp
\noindent\textbf{$\bullet$~~$\term{II}=\frac{1}{n}\sum_{i=1}^{n}\left(y_i - \zeta \alpha_i-\beta_i\right) 
\left((\zeta-\zeta^2)\alpha_i-\zeta\beta_i\right)y_i\wstar$}\\
\begin{align*}
\term{II}=\frac{1}{n}\sum_{i=1}^{n}\left(y_i(1-\zeta\corw) -\beta_i- \zeta \corq\right) 
\left(y_i(\zeta-\zeta^2)\corw-\zeta\beta_i+(\zeta-\zeta^2)\corq\right)y_i\wstar=B\wstar.
\end{align*}
We we can express $B$ above conveniently as:
\begin{align*}
B&:=\left((\zeta-\zeta^2)\corw-(\zeta^2-\zeta^3)\left(\corw^2+\corq^2\right)\right)\left(\frac{1}{n}\sum_{i\in[n]}y_i\right)  +\left((\zeta-\zeta^2)-2(\zeta^2-\zeta^3)\corw\right)\corq
\\
&\qquad\qquad+\left(-\zeta+2\zeta^2\right)\corq\left(\frac{1}{n}\sum_{i=1}^{n}\beta_iy_i\right)
+\left(-\zeta+(-\zeta+2\zeta^2)\corw\right)\left(\frac{1}{n}\sum_{i=1}^{n}\beta_i\right)
+\zeta \left(\frac{1}{n}\sum_{i=1}^{n}\beta_i^2y_i\right).
\end{align*}
All the random terms above are zero-mean. Hence,
\begin{align}
\E[B]&=\left((\zeta-\zeta^2)-2(\zeta^2-\zeta^3)\corw\right)\corq\,.
\label{eq:Exp_term2}
\end{align}


\vp
\noindent\textbf{$\bullet$~~ $\term{III}=\frac{1}{n}\sum_{i=1}^{n}\left(y_i - \zeta \alpha_i-\beta_i\right) 
 \hat\Sigmab_i\w$}\\

Fix any vector $\vb:$

\begin{align*}
\vb^\top\{\term{III}\} &= \left(\frac{1}{n}\sum_{i=1}^{n}y_i\vb^\top\hat\Sigmab_i\w\right) - \zeta \left(\corq+y_i\corw\right)\left(\frac{1}{n}\sum_{i=1}^{n}\vb^\top\hat\Sigmab_i\w\right)-\left(\frac{1}{n}\sum_{i=1}^{n}\left(\gammab_i^\top\w\right)\vb^\top\hat\Sigmab_i\w\right)\\
&= \left(1-\zeta\corw\right)\left(\frac{1}{n}\sum_{i=1}^{n}y_i\vb^\top\hat\Sigmab_i\w\right) - \zeta \corq\left(\frac{1}{n}\sum_{i=1}^{n}\vb^\top\hat\Sigmab_i\w\right)-\left(\frac{1}{n}\sum_{i=1}^{n}\left(\gammab_i^\top\w\right)\vb^\top\hat\Sigmab_i\w\right)
\\
&= \left(1-\zeta\corw\right)\left(\frac{1}{n}\sum_{i=1}^{n}y_i\vb^\top\hat\Sigmab_i\w\right) 
-
\zeta \corq \left(\frac{1}{n}\sum_{i=1}^{n}\vb^\top\hat\Sigmab_i\w-(1-\zeta)\vb^\top\Sigmab\w\right)-
 (\zeta-\zeta^2)\corq \vb^\top\Sigmab\w
\\ &\qquad\qquad-
\frac{1}{n}\sum_{i=1}^{n}\left(\gammab_i^\top\w\right)\left(\vb^\top\hat\Sigmab_i\w-(1-\zeta)\w^\top\Sigmab\vb\right)+(1-\zeta)(\w^\top\Sigmab\vb)\frac{1}{n}\sum_{i=1}^{n}(\gammab_i^\top\w)
\end{align*}

From Lemma \ref{lem:psi} all random terms above are zero mean. Hence,
\begin{align}
\E\left[\vb^\top\{\term{III}\}\right] = -(\zeta-\zeta^2)\corq \w^\top\Sigmab\vb .\label{eq:Exp_term3}
\end{align}


\vp
\noindent\textbf{$\bullet$~~ $\term{IV}=\frac{1}{n}\sum_{i=1}^{n}\left(y_i - \zeta \alpha_i-\beta_i\right)  (\zeta\alpha_i+\beta_i)\gammab_i$}\\

For fixed vector $\vb$, $
    \vb^\top\{\term{IV}\} = \frac{1}{n}\sum_{i=1}^{n}\left(y_i - \zeta \alpha_i-\beta_i\right)  (\zeta\alpha_i+\beta_i)\vb^\top\gammab_i.
$
Reorganizing, note that
\begin{align*}
\left(y_i - \zeta \alpha_i-\beta_i\right)  (\zeta\alpha_i+\beta_i)  &= \zeta y_i (\corq+y_i\corw) - \zeta^2(\corq^2+\corw^2+2y_i\corq\corw)-2\zeta\corq\beta_i-2\zeta\corw y_i\beta_i+y_i\beta_i-\beta_i^2\\
&= \left(\zeta\corq -2 \zeta^2\corq\corw\right) y_i + \left(\zeta\corw - \zeta^2(\corq^2+\corw^2)\right) -\left(2\zeta\corq\right)\beta_i+\left(1-2\zeta\corw\right) y_i\beta_i-\beta_i^2
\end{align*}

Overall,
\begin{align*}
    \vb^\top\{\term{IV}\} &=\left(\zeta\corw - \zeta^2(\corq^2+\corw^2)\right) \left( \frac{1}{n}\sum_{i=1}^{n}\vb^\top\gammab_i\right)+\left(\zeta\corq -2 \zeta^2\corq\corw\right)\left( \frac{1}{n}\sum_{i=1}^{n}y_i\left(\vb^\top\gammab_i\right)\right)
    \\
    &+\left(1-2\zeta\corw\right) \left( \frac{1}{n}\sum_{i=1}^{n}y_i\left(\w^\top\gammab_i\right)\left(\vb^\top\gammab_i\right)\right)
    \\
    &-\left(2\zeta\corq\right)\left( \frac{1}{n}\sum_{i=1}^{n}\left(\w^\top\gammab_i\right)\left(\vb^\top\gammab_i\right)-\frac{1-\zeta}{T}\vb^\top\Sigmab\w\right)-\left(2\zeta\corq\right)\frac{1-\zeta}{T}\vb^\top\Sigmab\w
    \\
    &- \frac{1}{n}\sum_{i=1}^{n}\left(\vb^\top\gammab_i\right)\left(\left(\w^\top\gammab_i\right)^2-\frac{(1-\zeta)}{T}\w^\top\Sigmab\w\right)
+\frac{(1-\zeta)}{T}\w^\top\Sigmab\w \left(\frac{1}{n}\sum_{i=1}^{n}\left(\vb^\top\gammab_i\right)\right)\,.
\end{align*}
According to Lemma \ref{lem:psi} all random terms above are zero mean. Thus,
\begin{align}
\E\left[\vb^\top\{\term{IV}\}\right] &= -2\zeta\corq\frac{1-\zeta}{T}\w^\top\Sigmab\vb
\label{eq:Exp_term4}
\end{align}

\vp
\noindent\textbf{$\bullet$~~ Combined} \\
The desired identities \eqref{eq:mean_one_step_grad_q} and \eqref{eq:Ghat full} follow by combining all the terms above.

\subsubsection{Proof of Lemma \ref{lem:psi}}

\vp
\noindent\underline{First bound:} Obvious by boundedness (hence, sub-gaussianity) of $y_i$ and Fact \ref{fact:psi_additive}.

\vp
\noindent\underline{Second bound:} For convenience set $\wt T=(1-\zeta)T$ and assume wlog that $\Rcc=[\wt T].$  Recall that $$\beta_i=\frac{1}{T}\sum_{t=1}^{\wt T} \z_{i,t}^\top\w=\frac{1-\zeta}{\wt T} \sum_{t=1}^{\wt T}\z_{i,t}^\top\w.
$$
Also
for all $t$: $\psinorm{\z_{i,t}^\top\w}{2}\leq K\|\w\|_2$. Thus, from Fact \ref{fact:psi_additive}:
\begin{align}\label{eq:beta_psi2}
\psinorm{\beta_i}{2} \leq \frac{C\sigma(1-\zeta)\|\w\|_2}{\sqrt{\wt T}} = \frac{C\sigma\sqrt{(1-\zeta)}\|\w\|_2}{\sqrt{ T}}\,.
\end{align}
The bound then follows by applying Fact \ref{fact:psi_additive} once more.  

For the second term in this bound  recall that $y_i\in\{\pm{1}\}$ and $\beta_i=\sum_t \z_{i,t}^\top\w/T$. Also, for all $i\in[n]$: $y_i,\{\z_{i,t}\}_{t}$ are zero-mean and independent. Thus (i) $\E[y_i\beta_i]=0$, and (ii)  {$\{y_i\z_{i,t}\widesim{D} \z_{i,t} \text{ and } y_i\z_{i,t} \perp y_i\z_{i,t'} \} \implies y_i\beta_i\widesim{D} \beta_i$.} Thus, the same bound as the first term holds.

\vp
\noindent\underline{Third bound:} It is easy to compute 
\begin{align}\label{eq:beta2_expectation}
\E[\left(\gammab_i^\top\w\right)\left(\gammab_i^\top\vb\right)]=\frac{1}{T^2}\sum_{t=1}^{\wt T}\sum_{t'=1}^{\wt T}\E[\w^\top\z_{i,t}\z_{i,t'}\vb] = \frac{1-\zeta}{T}\vb^\top\Sigmab\w,
\end{align} 
and, using \eqref{eq:beta_psi2}
\begin{align}\label{eq:beta2_psi1}
\psinorm{\left(\gammab_i^\top\w\right)\left(\gammab_i^\top\vb\right)-\E[\left(\gammab_i^\top\w\right)\left(\gammab_i^\top\vb\right)]}{1} \leq C \psinorm{\left(\gammab_i^\top\w\right)\left(\gammab_i^\top\vb\right)}{1} \leq  C \psinorm{\gammab_i^\top\w}{2}\psinorm{\gammab_i^\top\vb}{2} = \frac{C\sigma^2{(1-\zeta)}\|\w\|_2\|\vb\|_2}{{T}}.
\end{align}
Since $\left(\gammab_i^\top\w\right)\left(\gammab_i^\top\vb\right), i\in[n]$ are independent, the desired bound on the first term follows from Fact \ref{fact:psi_additive}.

Consider now the second term. By independence of $y_i, \gammab_i$ it holds that $\E[y_i\left(\gammab_i^\top\w\right)\left(\gammab_i^\top\vb\right)]=0$. Arguing as we did above for the second bound, $y_i\gammab_i^\top\w\widesim{D} \gammab_i^\top\w$. Hence, the subexponential bound is the same as for the first term.

\vp
\noindent\underline{Fourth bound:} First, it is easy to compute that for all $i\in[n]$:
$$
\E[\w^\top\Sigmab_i\vb]=\frac{1}{T}\E[\w^\top\Z_i^\top\Z_i\vb] = \frac{1}{T}\sum_{t=1}^{\wt T}\E[\w^\top\z_{i,t}^\top\z_{i,t}\vb] = \frac{\wt T}{T}\w^\top\Sigmab\vb = (1-\zeta)\w^\top\Sigmab\vb.
$$
Thus,
\begin{align*}
\w^\top\hat\Sigmab_i\vb-\E[\w^\top\hat\Sigmab_i\vb]=\w^\top\hat\Sigmab_i\vb-(1-\zeta)\w^\top\Sigmab \vb
=\frac{1}{T}\sum_{t=1}^{\wt T}\Big((\z_{i,t}^\top\w)(\z_{i,t}^\top\vb)-\w^\top\Sigmab \vb\Big)
\end{align*}
and so
\begin{align*}
\frac{1}{n}\sum_{i=1}^{n}\w^\top\hat\Sigmab_i\vb-(1-\zeta)\w^\top\Sigmab \vb
=\frac{1}{nT}\sum_{i=1}^n\sum_{t=1}^{\wt T}\Big((\z_{i,t}^\top\w)(\z_{i,t}^\top\vb)-\w^\top\Sigmab \vb\Big).
\end{align*}
Now, each random variable in the double sum above is independent and such that
\begin{align}\label{eq:wSv}
\|(\z_{i,t}^\top\w)(\z_{i,t}^\top\vb)-\w^\top\Sigmab \vb\|_{\psi_1}\le 2\|(\z_{i,t}^\top\w)(\z_{i,t}^\top\vb)\|_{\psi_1}\le 2\|\z_{i,t}^\top\w\|_{\psi_2}\|\z_{i,t}^\top\vb\|_{\psi_2}\le C\sigma^2\|\w\|_2\|\vb\|_2.
\end{align}
Hence, from Fact \ref{fact:psi_additive},
\begin{align*}
\psinorm{\frac{1}{n}\sum_{i=1}^{n}\w^\top\hat\Sigmab_i\vb-(1-\zeta)\w^\top\Sigmab \vb}{1}\leq \frac{C\sigma^2\sqrt{1-\zeta}\|\w\|_2\|\vb\|_2}{\sqrt{nT}}\,.
\end{align*}

The bound for the second term follows along the same lines. The two key observations are that (i) $\E[y_i\w^\top\hat\Sigmab_i\vb]=0$ because $y_i$ and $\z_{i,t}$ are independent, and, (ii) 
\begin{align*}
y_i\w^\top\hat\Sigmab_i\vb=
\frac{1}{T}\w^\top y_i\Z_i^\top\Z_i\vb \widesim{D} \frac{1}{T}\w^\top\wt\Z_i^\top\Z_i\vb
=\frac{1}{T}\sum_{t=1}^{\wt T}(\wt\z_{i,t}^\top\w)(\z_{i,t}^\top\vb)
\end{align*}
where $\wt\Z_i$ is an independent copy of $\Z_i$.

\vp
\noindent\underline{Fifth bound:} From \eqref{eq:beta2_psi1} and \eqref{eq:beta2_expectation}, we have for all $i\in[n]$ that
\begin{align*}\label{eq:beta2_psi1}
\psinorm{\left(\gammab_i^\top\w\right)\left(\gammab_i^\top\vb\right)-\frac{1-\zeta}{T}\w^\top\Sigmab\w}{1} \leq \frac{C\sigma^2{(1-\zeta)}\|\w\|_2^2}{{T}}.
\end{align*}
Moreover, recall from Eq. \eqref{eq:beta_psi2} that
$$
\psinorm{\gammab_i^\top\vb}{2}\leq  \frac{C\sigma\sqrt{1-\zeta}\|\vb\|_2}{\sqrt{ T}}.
$$
Combining the above two displays and applying Fact \ref{fact:orlitz_inequality} we find for all $i\in[n]$ that
\begin{align}
\psinorm{\left(\left(\gammab_i^\top\w\right)^2-\frac{1-\zeta}{T}{\w^\top\Sigmab\w}\right)\gammab_i^\top\vb}{2/3}&\leq \frac{C\sigma^3(1-\zeta)^{3/2}\|\w\|_2^2\|\vb\|_2}{T^{3/2}}.
\end{align}
The desired bound follows from the above after using Fact \ref{fact:psi_heavy_Talagrand}.

\vp
\noindent\underline{Sixth bound:}
From Eq. \eqref{eq:wSv}:
$$
\psinorm{\w^\top\hat\Sigmab_i\vb - (1-\zeta)\w^\top\Sigmab\vb}{1} \leq \frac{C\sigma^2\sqrt{1-\zeta}\|\w\|_2\|\vb\|_2}{\sqrt{T}}
$$
and from Eq. \eqref{eq:beta_psi2}
$$
\psinorm{\gammab_i^\top\w}{2}\leq  \frac{C\sigma\sqrt{1-\zeta}\|\w\|_2}{\sqrt{ T}}.
$$
Next we use Fact \ref{fact:orlitz_inequality} with $\alpha=2$ and $\beta=1$ to find that 
\begin{align*}
\psinorm{\left(\gammab_i^\top\w\right)\left(\w^\top\hat\Sigmab_i\vb-(1-\zeta)\w^\top\Sigmab\vb\right)}{2/3} \leq 
 \frac{ C K^3({1-\zeta})\|\w\|_2^2\|\vb\|_2}{{T}}\,.
\end{align*}
Next we use Fact \ref{fact:psi_heavy_Talagrand} which allows us to conclude that 
\begin{align*}
\psinorm{\frac{1}{n}\sum_{i=1}^n \left(\gammab_i^\top\w\right)\left(\w^\top\hat\Sigmab_i\vb-(1-\zeta)\w^\top\Sigmab\vb\right)}{2/3}\le  \frac{ C K^3({1-\zeta})\|\w\|_2^2\|\vb\|_2}{{T\sqrt{n}}}\log n\,.
\end{align*}
Finally, the zero-mean property follows since
\begin{align*}
\E[\left(\ones^\top\Z_i\w\right)\left(\w^\top\Z_i\Z_i^\top\vb\right)] &= 
\sum_{t=1}^{\wt T}\sum_{t'=1}^{\wt T}\E[\left(\z_{i,t}^\top\w\right)\left(\w^\top\z_{i,t'}\z_{i,t'}^\top\vb\right)]
= 
\sum_{t=1}^{\wt T}\sum_{t'=1}^{\wt T}\E[\w^\top\z_{i,t}\tr\left({\z_{i,t'}\z_{i,t'}^\top\vb\w^\top}\right)]
\\
&=
\wt T^2\E[\tr\left(\z^\top\w\right)\tr\left({\z\z^\top\vb\w^\top}\right)] = \wt T^2\E[\tr\left( (\z^\top\w) \otimes (\z\z^\top\vb\w^\top)\right)]
\\
&=
\wt T^2\tr\left( \E[(\z^\top\otimes \z\z^\top)] (\w\otimes\vb\w^\top)\right) = 0\,,
\end{align*}
where the last equality follows by the zero third moment property in Assumption \ref{ass:noise}.

\section{Finite-sample gradient analysis} 

{In the following, we assume without explicit further reference that Ass. \ref{ass:rho general but small} holds (i.e. $\abs{\rho}<\wnorm/\qnorm$) and additionally that $\qnorm>\wnorm$. To simplify the results we further assume $\sigma\propto 1$ and $W\gtrsim 1$.}

\subsection{First gradient step}

The lemma below studies the deviation of the first-step of GD $\what_1$ with respect to its population counterpart $\w_1$.  Provided that $n$ and $n\zeta T /d$ are larger than appropriate functions of other problem parameters, then the deviations are of small multiplicative order. 
\begin{lemma}[First gradient step]\label{lem:1st gradient step}
Consider the one-step population and finite updates $\w_1=\eta\G_\w(0,0)$ and $\what_1=\eta\Ghat_\w(0,0)$, respectively. For convenience denote 
\[\corw=\w_1^\top\wstar,\quad \corq=\w_1^\top\qstar, \quad \hcorw=\what_1^\top\wstar, \quad\hcorq=\what_1^\top\qstar\,.
\]
For any $u>0$ and any small constant $c_0>0$, there exist absolute constants $c,c'>0$ and large enough  constant $C=C(c_0)>0$ such that if
\begin{align}\label{eq:large n lemma}
    {\sqrt{n}}&{\geq C u \frac{\qnorm}{\wnorm}}
    \qquad\text{and}\qquad
    \sqrt{n\zeta T}\geq C u \frac{\sigma}{\wnorm}\sqrt{\zeta^{-1}-1}\, ,
\end{align}
then, with probability at least $1-c'e^{-cu^2}$
    \begin{align}\label{eq:w1hat corr}
        \abs{\hcorw-\corw}\leq c_0\corw\qquad\text{and}\qquad\abs{\hcorq-\corq}\leq {c_0\eta \zeta \qnorm\wnorm}\,.
    \end{align}
Additionally, if 
    \begin{align}
    \sqrt{n}&\geq Cu\frac{\qnorm}{\wnorm}
    \qquad\text{and}\qquad
    \sqrt{n\zeta T}\geq C{(1+u)}\frac{\sigma}{\wnorm}\sqrt{\zeta^{-1}-1}  \sqrt{d}
\end{align}
then with probability {$1-e^{-cu^2}-e^{-cdu^2}$}
\begin{align}\label{eq:w1hat norm}
    \abs{\tn{\what_1}-\tn{\w_1}}\leq{c_0{\eta\zeta\wnorm}}.
\end{align}
\end{lemma}
\begin{proof}
Note that the conclusions of the lemma are all homogeneous in $\eta$. Hence, without loss of generality, set $\eta=1$. Also recall by Lemma \ref{lem:w_der_finite} that
    \begin{align}\label{eq:what_1 first eq}
\what_1=\Ghat_\w(0,0)
&=  \zeta \w_\star + \zeta  \qb_\star \left(\frac{1}{n}\sum_{i\in[n]}y_i\right)+ 
\frac{1}{n}\sum_{i}y_i\gammab_i\,,
\end{align}
and  $\w_1=\G_\w(0,0)=\zeta\wstar$; thus, $\corw=\zeta\wnorm^2$.
From these, and also using Lemma \ref{lem:psi} , for any $u>0$ with probability at least $1-2e^{-cu^2}$
\begin{align}
\abs{\hcorw-\corw}&=\abs{\wstar^\top(\what_1-\w_1)}\leq \zeta  \abs{\rho}\wnorm\qnorm \abs{\frac{1}{n}\sum_{i\in[n]}y_i}+ 
\abs{\frac{1}{n}\sum_{i}y_i\gammab_i^\top\wstar}\leq  \frac{Cu\zeta   \abs{\rho}\wnorm\qnorm}{\sqrt{n}} + \frac{Cu\sigma\sqrt{1-\zeta}\sqrt{\zeta}\wnorm}{\sqrt{n\zeta T}}\nn
\\
&\leq c_0{\zeta \wnorm^2}=c_0{\corw}.\nn
\end{align}
where the last inequality follows by assuming $n,\zeta T$ large enough as in the condition of the lemma {and using $\abs{\rho}\leq 1$}. 
Similarly,  
\begin{align*}
\abs{\hcorq-\corq}=\abs{\qstar^\top(\what_1-\w_1)}&\leq \zeta  \qnorm^2 \abs{\frac{1}{n}\sum_{i\in[n]}y_i}+ 
\abs{\frac{1}{n}\sum_{i}y_i\gammab_i^\top\qstar}
\leq  \frac{Cu\zeta  \qnorm^2}{\sqrt{n}} + \frac{Cu\sigma\sqrt{1-\zeta}\sqrt{\zeta}\qnorm}{\sqrt{n\zeta T}}\nn\leq 
c_0{\zeta\qnorm\wnorm}
\end{align*}
for sufficiently large $n$ per \eqref{eq:large n lemma}.
Finally,  {with probability at least $1-e^{-cu^2}-e^{-cdu^2}$}
\begin{align}
    \tn{\what_1}-\tn{\w_1}\leq \tn{\what_1-\w_1}\leq \frac{Cu\zeta\qnorm }{\sqrt{n}} + \frac{C(1+u)\sigma\sqrt{1-\zeta}\sqrt{\zeta}\sqrt{d}}{\sqrt{n\zeta T}}\leq {c_0{\zeta W}}\,.
    \label{eq:w1 norm diff}
\end{align}
where, again, the last inequality follows by assuming $n,\zeta T$ large enough as stated in the lemma. In the second inequality, we used from Lemma \ref{lem:psi} that $\sum_{i\in[n]}y_i\gammab_i/n$ is $C\sigma\sqrt{1-\zeta}/\sqrt{nT}$-subGaussian and applied Fact \ref{fact:norm subg} to get a high-probability bound on its euclidean norm.
\end{proof}

\subsection{Second gradient step}

Next, we move on to the second gradient update in the direction of $\Ghat_\qb(0,\what_1).$ Recall our goal of controlling the relevance scores of signal and noisy tokens. The first lemma below takes a step in this direction by computing the signal and noise relevance scores assuming access to the population gradient $\G_\qb(0,\what_1):=\E\left[\Ghat_\qb(0,\what_1)\right].$
\begin{lemma}[$\Ghat_\qb(0,\cdot)$ control: Expectation term]\label{lem:finite q1 expectation}Let $\G_\qb(0,\w_1)=\E\left[\Ghat_\qb(0,\w_1)\right]$ be the expectation of a gradient step in the $\qb$-direction evaluated at $(0,\what_1)$ and recall that $\what_1=\eta\Ghat_\w(0,0)$ for $\eta>0.$ Suppose $\what_1$ satisfies \eqref{eq:w1hat corr} and \eqref{eq:w1hat norm} {for sufficiently small enough constant $c_0>0$.}
Further assume that the step-size $\eta$ satisfies the following for sufficiently small absolute constant $c_\eta>0$:
    \begin{align}\label{eq:eta small deviation}
    \eta=\frac{c_\eta}{\sigma^2\qnorm^{2}} \,.
    \end{align}
Then, for $y\in\{\pm1\}$ it holds that

\begin{align}
\left(\qstar+y\wstar\right)^\top\G_\qb(0,\hat\w_1) &\geq  {\eta\zeta(\zeta-\zeta^2)\wnorm^2\qnorm\left({\left(1/4-\rho^2/8-3c_0\right)\qnorm - \left(9/4\abs{\rho}+c_0\right)\wnorm} \right)\,,}
   \label{eq:finite signal exp}
\end{align}
and
 {
\begin{align}\label{eq:finite norm exp}
\tn{\G_\qb(0,\hat\w_1)}\leq\eta\zeta\left(\zeta-\zeta^2\right)\,\wnorm^2\qnorm\left(13/4+2c_0\right)\,.
\end{align}
}
\end{lemma}

\begin{proof}
Fix any $\vb$ and recall the notation of Lemma \ref{lem:1st gradient step}. With these, we have from Lemma \ref{lem:conc_grad_q} that
\begin{align}
    \vb^\top\G_\qb(0,\what_1)&=\left((\zeta-\zeta^2)\left(\hcorw+{\sigma^2\tn{\what_1}^2}/{T}\right)-(\zeta^2-\zeta^3)\left(\hcorw^2 + \hcorq^2\right)\right)\,\vb^\top\qstar\nn
\\ 
&\qquad\qquad 
+\left((\zeta-\zeta^2)-2(\zeta^2-\zeta^3)\hcorw\right)\hcorq \,\vb^\top\wstar\nn
-\sigma^2\left(1+{2}/{T}\right)(\zeta-\zeta^2)\hcorq\vb^\top\what_1
\\
&=(\zeta-\zeta^2)\hcorw\left(1+{\sigma^2\tn{\what_1}^2}/({T}\hcorw)-\zeta\left(\hcorw + \hcorq^2/\hcorw\right)\right)\,\vb^\top\qstar\nn
\\ 
&\qquad\qquad 
+(\zeta-\zeta^2)\hcorq \,\left(1-2\zeta\hcorw\right)\,\vb^\top\wstar\nn
-\sigma^2\left(1+{2}/{T}\right)(\zeta-\zeta^2)\hcorq\vb^\top\what_1
\\
&=(\zeta-\zeta^2)\hcorw\underbrace{\left(1+{\sigma^2\tn{\what_1}^2}/({T}\hcorw)-\zeta\left(\hcorw + \hcorq^2/\hcorw\right)\right)}_{:=\Chat_1}\,\vb^\top\qstar\nn
\\ 
&\qquad
+(\zeta-\zeta^2)\hcorq \,\underbrace{\left(1-2\zeta\hcorw-\eta\zeta\sigma^2\left(1+{2}/{T}\right)\right)}_{:=\Chat_2}\,\vb^\top\wstar -\underbrace{\eta\sigma^2\left(1+{2}/{T}\right)}_{:=\Chat_3}(\zeta-\zeta^2)\hcorq\vb^\top\deltab_1\label{eq:vG}
\end{align}
where, in the last line we used \eqref{eq:what_1 first eq} and set \[
\deltab_1:=(\what_1-\w_1)/\eta=\zeta  \qb_\star \left(\frac{1}{n}\sum_{i\in[n]}y_i\right)+ 
\frac{1}{n}\sum_{i}y_i\gammab_i.
\]

Recall from \eqref{eq:w1hat corr} and \eqref{eq:w1hat norm} that $\corw/2\leq \hcorw\leq 3\corw/2$,  {$|\hcorq-\corq|\leq c_0\eta\zeta\qnorm\wnorm$ and $\tn{\what_1}\leq \tn{\w_1}+c_0\eta\zeta\wnorm$.} Also, from \eqref{eq:w1hat norm}  we have that {
$
\tn{\deltab_1}\leq c_0\zeta\wnorm,.
$
}
Here and onwards, $c_0>0$ is a small enough absolute constant (smaller than $1/2$) whose value may change from line to line. 
Further recall $\corw=\eta\zeta\wnorm^2$, $\corq=\eta\zeta\rho\wnorm\qnorm$ and $\tn{\w_1}=\eta\zeta\wnorm$. With these, we can set small enough step size $\eta\propto \sigma^{-2}\qnorm^{-2}$   such that 
\[\Chat_1\in[1/2,3/2],\quad \Chat_2\in[-1/8,1], \quad \Chat_3\in[0,(c_3/\zeta)/\qnorm^2]\,,
\]
for constant $c_3>0$ to be made small enough later in the proof.  

From the above, we can compute  {
\[
\Chat_3\abs{\hcorq}\tn{\deltab_1}\leq \frac{c_3}{\zeta\qnorm^2}\eta\zeta(c_0\qnorm\wnorm+\rho\wnorm\qnorm)(c_0\zeta\qnorm)\leq c_3c_0\eta\zeta\,\qnorm\,.
\]}
 {Moreover,
\begin{align*}
    \hat C_1 \hcorw \qnorm^2 &\geq \eta\zeta\wnorm^2\qnorm^2/4
    \\
    \hat C_2 \hcorq \rho\wnorm\qnorm &= \hat C_2 \rho\wnorm\qnorm\corq + C_2 \rho\wnorm\qnorm (\hcorq-\corq) \geq 
    -\eta\zeta\rho^2\qnorm^2\wnorm^2/8 -  c_0\eta\zeta\qnorm^2\wnorm^2
        \\
        &\geq  
    -\eta\zeta\rho^2\qnorm^2\wnorm^2/8 - c_0\eta\zeta\wnorm^2\qnorm^2\,.
\end{align*}
}
Putting the above displays together we find that
\begin{align*}
    \qstar^\top\G_\qb(0,\hat\w_1)&\geq \eta\zeta(\zeta-\zeta^2)\left(\wnorm^2\qnorm^2/4- {\rho^2\qnorm^2\wnorm^2/8 - c_0\wnorm^2\qnorm^2-c_3c_0\qnorm^2}\right)\,.
\end{align*}
 {Similarly, we can compute that
\begin{align*}
    \abs{\hat C_1 \hcorw \rho \qnorm \wnorm} &\leq (9/4)\eta\zeta\wnorm^3\qnorm\abs{\rho}
    \\
    \abs{\hat C_2 \hcorq \wnorm^2} &\leq \eta\zeta(\abs{\rho}\wnorm^3\qnorm+c_0\qnorm\wnorm^3)\,.
\end{align*}
}
Hence, for $y\in\{\pm1\}$
\begin{align}
\abs{y\wstar^\top\G_\qb(0,\hat\w_1)}\leq \eta\zeta\left(\zeta-\zeta^2\right)
 {4\wnorm^3\qnorm\abs{\rho}+c_0\qnorm\wnorm^3+c_3c_0\qnorm\wnorm}\,.\nn
\end{align}
The above two displays put together yield \eqref{eq:finite signal exp}. Specifically, we also use the simplifying assumption  $\qnorm\gtrsim \wnorm \gtrsim 1$ and pick $c_3$ small enough so that $c_3\wnorm^{-2}\qnorm^{-1}\leq 1$.

The norm bound in \eqref{eq:finite norm exp} follows again starting from \eqref{eq:vG} and using similar arguments as above to show:
\begin{align*}
    \tn{\G_\qb(0,\hat\w_1)}&\leq \eta\zeta(\zeta-\zeta^2)\left(\Chat_1\corw\qnorm+\abs{\Chat_2}\abs{\hcorq}\wnorm+\Chat_3\abs{\hcorq}\tn{\deltab_1}\right)\\
    &\leq \eta\zeta(\zeta-\zeta^2)\left((9/4)\wnorm^2\qnorm+(\abs{\rho}+c_0)\wnorm^2\qnorm+c_3c_0\qnorm\right)\,.
\end{align*}

\end{proof}

The next lemma controls the effect on the relevance scores of the deviation  term $\Gtilde_\qb(0,\what)=\Ghat_\qb(0,\what)-\G_\qb(0,\what).$
\begin{lemma}[$\Ghat_\qb(0,\cdot)$ control: Deviation term]
    \label{lem:finite q1 deviation}Let $\Gtilde_\qb(0,\what_1):=\Ghat_\qb(0,\what_1)-\G_\qb(0,\what_1)$ and suppose $\what_1$ satisfies \eqref{eq:w1hat corr} and \eqref{eq:w1hat norm}.
    Also assume $\sigma=1$. Fix any $u>0$ and any small constant $c_1>0$. 
    Then, there exists small enough constant $c_\eta$ (dependent on $c_1$) such that if step-size $\eta$ is small enough as per \eqref{eq:eta small deviation} the  following statements hold.  
    
    With probability at least $1-c'e^{-cu^{2/3}}$ for positive constants $C,c',c>0$ it holds for signal tokens that 
\begin{align}    
\abs{(\qstar+y\wstar)^\top\Gtilde_\qb(0,\what_1)}\leq u\, C\,c_1\,{\qnorm}\left(\qnorm\maxi \frac{\sigma\maxi \sigma^2}{\sqrt{T}}\maxi\frac{\sigma^3}{T}\right)\frac{\log(n)}{\sqrt{n}},\,\quad y\in\{\pm1\}\label{eq:finite signal dev}
\end{align}
Moreover, with probability at least $1-c'd\,e^{-cu^{2/3}}$ it holds that

\begin{align}\label{eq:finite norm dev}
    \tn{\Gtilde_\qb(0,\hat\w_1)}\leq u C\,c_1\,\left({\qnorm} \maxi \frac{\sigma\maxi \sigma^2}{\sqrt{T}}\maxi\frac{\sigma^3}{T}\right)\frac{\log(n){\sqrt{d}}}{\sqrt{n}}\,.
    \end{align}


\end{lemma}
\begin{proof} 
     We study each one of the terms of $\Gtilde_\qb(0,\what_1)$ in \eqref{eq:Ghat full} separately. We repeat the terms here for convenience, also noting the substitutions $\corw\leftarrow\hcorw=\what_1^\top\wstar$ and $\corq\leftarrow\hcorq=\what_1^\top\qstar.$
\begin{align} 
    \vb^\top\Gtilde_\qb(0,\what_1)&=\left[\left((\zeta-\zeta^2)-2(\zeta^2-\zeta^3)\hcorw\right)\hcorq\corqv+\left((\zeta-\zeta^2)\hcorw-(\zeta^2-\zeta^3)\left(\hcorw^2+\hcorq^2\right)\right)\corwv\right]\left(\frac{1}{n}\sum_{i=1}^{n}y_i\right) \nn
\\
&\qquad+\left[\left(-\zeta+2\zeta^2\right)\hcorq\corqv+\left(-\zeta+(-\zeta+2\zeta^2)\hcorw\right)\corwv+(1-\zeta)\what_1^\top\Sigmab\vb\right]\left(\frac{1}{n}\sum_{i=1}^{n}\left(\gammab_i^\top\what_1\right)\right) \nn
\\
&\qquad\qquad\qquad+\left[\zeta\hcorw - \zeta^2(\hcorq^2+\hcorw^2)+ \frac{(1-\zeta)}{T}\what_1^\top\Sigmab\what_1\right] \left( \frac{1}{n}\sum_{i=1}^{n}\vb^\top\gammab_i\right) \nn
\\
&\qquad\qquad\qquad+\left[\left(-\zeta+(-\zeta+2\zeta^2)\hcorw\right)\corqv+\left(-\zeta+2\zeta^2\right)\hcorq\corwv\right]\left(\frac{1}{n}\sum_{i=1}^{n}y_i\left(\gammab_i^\top\what_1\right)\right) \nn
\\
&\qquad\qquad\qquad+\left[\zeta\hcorq -2 \zeta^2\hcorq\hcorw\right]\left( \frac{1}{n}\sum_{i=1}^{n}y_i\left(\vb^\top\gammab_i\right)\right) \nn
\\
&\qquad
+\zeta\corqv \left(\frac{1}{n}\sum_{i=1}^{n}\left(\gammab_i^\top\what_1\right)^2-\frac{(1-\zeta)}{T}\what_1^\top\Sigmab\what_1\right) \nn
\\
&\qquad\qquad\qquad+\zeta\corwv \left(\frac{1}{n}\sum_{i=1}^{n}\left(\gammab_i^\top\what_1\right)^2y_i\right) \nn
\\
&\qquad\qquad\qquad+\left[1-2\zeta\hcorw\right] \left( \frac{1}{n}\sum_{i=1}^{n}y_i\left(\what_1^\top\gammab_i\right)\left(\vb^\top\gammab_i\right)\right) \nn
\\
&\qquad+\left(1-\zeta\hcorw\right)\left(\frac{1}{n}\sum_{i=1}^{n}y_i\vb^\top\hat\Sigmab_i\what_1 \right) \nn
\\
&\qquad\qquad\qquad- 
\zeta \hcorq \left(\frac{1}{n}\sum_{i=1}^{n}\vb^\top\hat\Sigmab_i\what_1-(1-\zeta)\vb^\top\Sigmab\what_1\right) \nn
\\
&\qquad-\left[2\zeta\hcorq\right]\left( \frac{1}{n}\sum_{i=1}^{n}\left(\what_1^\top\gammab_i\right)\left(\vb^\top\gammab_i\right)-\frac{1-\zeta}{T}\vb^\top\Sigmab\what_1\right) \nn
\\
&\qquad- \left(\frac{1}{n}\sum_{i=1}^{n}\left(\vb^\top\gammab_i\right)\left(\left(\what_1^\top\gammab_i\right)^2-\frac{(1-\zeta)}{T}\what_1^\top\Sigmab\what_1\right)\right) \nn 
\\
&\qquad-\left(\frac{1}{n}\sum_{i=1}^{n}\left(\gammab_i^\top\what_1\right)\left(\vb^\top\hat\Sigmab_i\what_1-(1-\zeta)\what_1^\top\Sigmab\vb\right)\right) \,.\label{eq:Ghat full what}
\end{align}
Recall from the lemma assumption that \eqref{eq:w1hat corr} holds and from $\corw=\eta\zeta\wnorm^2$ and $\corq=\eta\zeta\rho\wnorm\qnorm$, that 
$1/2 \eta\zeta\wnorm^2\leq \hcorw\leq 3/2 \eta\zeta\wnorm^2$ and $\abs{\hcorq}\leq {\eta\zeta(\abs{\rho}+c_0)\wnorm\qnorm}$. The observation is that we can choose step-size $\eta$ small enough (as stated in \eqref{eq:eta small deviation}) to bound (in absolute value) all the coefficients in \eqref{eq:Ghat full what} (aka all terms in square brackets) that include  $\hcorw,\hcorq$.  Therefore, for any small positive constant $c_1>0$ it can be checked that there is sufficiently small constant $c_\eta$ that determines step-size $\eta$ in \eqref{eq:eta small deviation} such that
\begin{align}    
\abs{\vb^\top\Gtilde_\qb(0,\what_1)}&\leq c_1\left[\abs{\corqv}+\abs{\corwv}\right]\abs{\frac{1}{n}\sum_{i=1}^{n}y_i} \nn
\\
&\qquad+\left[c_1\abs{\corqv}+{(1+c_1)}\abs{\corwv}+(1-\zeta)\sigma^2\abs{\what_1^\top\vb}\right]\abs{\frac{1}{n}\sum_{i=1}^{n}\left(\gammab_i^\top\what_1\right)}\nn
\\
&\qquad\qquad\qquad+\left[c_1+ \frac{(1-\zeta)}{T}\sigma^2\tn{\what_1}^2\right] \abs{\frac{1}{n}\sum_{i=1}^{n}\vb^\top\gammab_i} \nn
\\
&\qquad\qquad\qquad+\left[{(1+c_1)}\abs{\corqv}+c_1\abs{\corwv}\right]\abs{\frac{1}{n}\sum_{i=1}^{n}y_i\left(\gammab_i^\top\what_1\right)} \nn
\\
&\qquad\qquad\qquad+\left[c_1\right]\abs{\frac{1}{n}\sum_{i=1}^{n}y_i\left(\vb^\top\gammab_i\right)}\nn
\\
&\qquad
+\zeta\abs{\corqv} \abs{\frac{1}{n}\sum_{i=1}^{n}\left(\gammab_i^\top\what_1\right)^2-\frac{(1-\zeta)}{T}\what_1^\top\Sigmab\what_1} \nn
\\
&\qquad\qquad\qquad+\zeta\abs{\corwv} \abs{\frac{1}{n}\sum_{i=1}^{n}\left(\gammab_i^\top\what_1\right)^2y_i} \nn
\\
&\qquad\qquad\qquad+\left[1+c_1\right] \abs{ \frac{1}{n}\sum_{i=1}^{n}y_i\left(\what_1^\top\gammab_i\right)\left(\vb^\top\gammab_i\right)} \nn
\\
&\qquad+\left[1+c_1\right]\abs{\frac{1}{n}\sum_{i=1}^{n}y_i\vb^\top\hat\Sigmab_i\what_1}\nn
\\
&\qquad\qquad\qquad + c_1\abs{\frac{1}{n}\sum_{i=1}^{n}\vb^\top\hat\Sigmab_i\what_1-(1-\zeta)\vb^\top\Sigmab\what_1}\nn
\\
&\qquad+\left[c_1\right]\abs{ \frac{1}{n}\sum_{i=1}^{n}\left(\what_1^\top\gammab_i\right)\left(\vb^\top\gammab_i\right)-\frac{1-\zeta}{T}\vb^\top\Sigmab\what_1} \nn
\\
&\qquad + \abs{\frac{1}{n}\sum_{i=1}^{n}\left(\vb^\top\gammab_i\right)\left(\left(\what_1^\top\gammab_i\right)^2-\frac{(1-\zeta)}{T}\what_1^\top\Sigmab\what_1\right)} \nn 
\\
&\qquad+\abs{\frac{1}{n}\sum_{i=1}^{n}\left(\gammab_i^\top\what_1\right)\left(\vb^\top\hat\Sigmab_i\what_1-(1-\zeta)\what_1^\top\Sigmab\vb\right)}\,.\label{eq:Ghat full what simple}
\end{align}
Now, we use successively Lemma \ref{lem:psi} to bound the random terms. Also note that $\abs{\corqv}\leq \qnorm\tn{\vb}$, $\abs{\corwv}\leq \wnorm\tn{\vb}$ and ${\tn{\what_1}\leq \eta\zeta(1+c_0)\wnorm=:\eta\zeta M}.$ Here, we denote $M:=(1+c_0)\wnorm$ for convenience. With these, for any $u>0$, with  probability at least $1-c'e^{-cu^{2/3}}$ we have

\begin{align}    
&\abs{\vb^\top\Gtilde_\qb(0,\what_1)}\leq u\cdot\frac{C}{\sqrt{n}}\,\tn{\vb}\,c_1\left(\wnorm+\qnorm\right) \nn
\\
&+u\cdot\frac{C\sigma\sqrt{1-\zeta}}{\sqrt{nT}}\,\tn{\vb}\,\left( \left({(1+c_1)}(2\wnorm+2\qnorm)+c_1\eta\zeta(1-\zeta)\sigma^2 M\right)\eta\zeta M + \left(2c_1 + \sigma^2(1-\zeta)\eta^2\zeta^2 M^2/T\right) \right) \nn
\\
&+u\cdot\frac{C\sigma^2(1-\zeta)}{T\sqrt{n}}\,\tn{\vb}\,\left(\eta^2\zeta^3M^2\qnorm + \eta^2\zeta^3\wnorm M^2+(1+c_1)\eta\zeta M\right) \nn
\\
&+u\cdot\frac{C\sigma^2\sqrt{1-\zeta}}{\sqrt{nT}}\,\tn{\vb}\,\left((1+2c_1)\eta\zeta M\right)\nn
\\
&+u\,\frac{C \sigma^2(1-\zeta)}{T\sqrt{n}}\,\tn{\vb}\,\left(c_1\eta\zeta M\right) \nn
\\
&+
u\,\frac{C\sigma^3(1-\zeta)^{3/2}\log(n)}{T^{3/2}\sqrt{n}}\,\tn{\vb}\,\left( \eta^2\zeta^2 M^2\right)
 \nn 
\\
&+
u\,\frac{C\sigma^3(1-\zeta)\log(n)}{T\sqrt{n}}\,\tn{\vb}\,\eta^2\zeta^2 M^2
\,.\label{eq:Ghat full what simple with concentration}
\end{align}

Now, again  using small step size $\eta$ as per \eqref{eq:eta small deviation} (recall that $M=(1+c_0)\wnorm\lesssim\qnorm$), this can be further simplified to the following (here the value of constant $c_1$ might be different from \eqref{eq:Ghat full what simple with concentration})
\begin{align}    
\abs{\vb^\top\Gtilde_\qb(0,\what_1)}&\leq u\cdot\frac{C}{\sqrt{n}}\,\tn{\vb}\,c_1\left(\wnorm+\qnorm\right)
+u\cdot\frac{C\sigma\sqrt{1-\zeta}}{\sqrt{nT}}\,\tn{\vb}\,c_1\nn
+u\cdot\frac{C\sigma^2(1-\zeta)}{T\sqrt{n}}\,\tn{\vb}\,c_1 \nn
\\
&\quad+u\cdot\frac{C\sigma^2\sqrt{1-\zeta}}{\sqrt{nT}}\,\tn{\vb}\,c_1
+
u\,\frac{C\sigma^3(1-\zeta)^{3/2}\log(n)}{T^{3/2}\sqrt{n}}\,\tn{\vb}\,c_1
 \nn 
+
u\,\frac{C\sigma^3(1-\zeta)\log(n)}{T\sqrt{n}}\,\tn{\vb}\,c_1
\,
\\
&\leq u\cdot\tn{\vb}\cdot \frac{C}{\sqrt{n}} \cdot c_1 \left(\qnorm\maxi\frac{\sigma\maxi \sigma^2}{\sqrt{T}}\maxi\frac{\sigma^2}{T}\maxi\frac{\sigma^3\log(n)}{T^{3/2}}\maxi\frac{\sigma^3\log(n)}{T}\right) \nn
\\
&\leq u\cdot\tn{\vb}\cdot \frac{C\,\log(n)}{\sqrt{n}} \cdot c_1 \left(\qnorm\maxi\frac{\sigma\maxi \sigma^2}{\sqrt{T}}\maxi\frac{\sigma^3}{T}\right) \,\label{eq:Ghat full what simple with concentration very simple to use}.
\end{align}

Now, we can compute the deviation of the relevance scores. For signal tokens we have for both $y\in\{\pm1\}$, and all $u>0$ with probability at least $1-c'e^{-cu^{2/3}}$, there exist constant $C>0$ such that
\begin{align}    
\abs{(\qstar+y\wstar)^\top\Gtilde_\qb(0,\what_1)}\leq u\, C\,c_1\,\qnorm\left(\qnorm\maxi \frac{\sigma\maxi \sigma^2}{\sqrt{T}}\maxi\frac{\sigma^3}{T}\right)\frac{\log(n)}{\sqrt{n}} \label{eq:finite signal dev 2}
\end{align}

Similarly, since \eqref{eq:Ghat full what simple with concentration very simple to use} holds for all $\vb$, we can apply it for all standard basis vectors $\vb=\eb_j, j\in[n]$ and union bounding yields
for all $u>0$ with probability at least $1-c'de^{-cu^{2/3}},$
\[
\tn{\Gtilde_\qb(0,\what_1)}\leq u\, C\,c_1\,\left(\qnorm \maxi \frac{\sigma\maxi \sigma^2}{\sqrt{T}}\maxi\frac{\sigma^3}{T}\right)\frac{\log(n)\,{\sqrt{d}}}{\sqrt{n}}\,.
\]
\end{proof}

\subsection{First and second gradient steps combined: Learning the relevant features}

With lemmas \ref{lem:1st gradient step}, \ref{lem:finite q1 expectation}, and \ref{lem:finite q1 deviation} at hand, we are now ready to put things together stating our final bounds for relevance scores. The finding is presented as a stand-alone lemma below.
\begin{lemma}[Put things together]\label{putogether1}
    Consider the finite-sample gradient step
    $\qhat_1=\Ghat_\qb(0,\what_1)$, where recall that $\what_1=\eta\Ghat_\w(0,0).$
    Fix any $u_0,u_1,u_2,u_3>0$ and any small constants $\wt{c}_0, \wt{c}_1>0$.
    Suppose step-size $\eta$ of first gradient step satisfies \eqref{eq:eta small deviation} for sufficiently small constant $c_\eta=c_\eta(\wt{c}_1)>0$  and further assume
\begin{align}\label{eq:large enough n}
    {\sqrt{n}}&{\geq u_0\cdot C_0}\,\frac{\qnorm}{\wnorm}
    \qquad\text{and}\qquad
    \sqrt{\frac{n\zeta T}{d}}\geq   {(1+u_0)} \cdot C_0  \frac{\sigma}{\wnorm}\sqrt{1/\zeta-1}.
\end{align}
for some large enough  constant $C_0=C_0(\wt{c_0})>0$. Finally, make the following mild assumption (for simplicity),
$
\frac{\sigma\maxi \sigma^2}{\sqrt{T}}\maxi\frac{\sigma^3}{T} \leq { \qnorm.}$
For a fresh dataset $(\X_i,y_i)_{i\in[n]}$  consider the signal relevance scores 
$
\rhat_{i,t}:=r_{i,0}:=\left(\qstar+y_i\wstar\right)^\top\Ghat_\qb(0,\hat\w_1)$ for $t\in\Rc$. 
There exist positive absolute constants $c_i,c_i',i=0,1,2,3$ such that the following statements hold.
\begin{itemize}
    \item With probability at least $1-c_0'e^{-c_0  u_0^2}-c_1' e^{c_1u_1^{2/3}},$ the signal relevance scores satisfy
    \begin{align}
    \label{eqcomb1}
        \min_{i\in[n]} r_{i,0} \geq  \eta\zeta
        (\zeta-\zeta^2)\wnorm^2\qnorm{\left(\left(1/4-\rho^2/8-2\wt{c_0}\right)\qnorm - \left(9/4\abs{\rho}+2\wt{c_0}\right)\wnorm \right)
        - 
        u_1\, \wt{c}_1\, \qnorm^2\frac{\log(n)}{\sqrt{n}}}=:B(u_1)\,.
    \end{align}
    \item With probability at least $1-c_0'e^{-c_0 u_0^2}-c_3'd e^{c_3u_3^{2/3}},$ the norm of  $\Ghat_\qb(0,\hat\w_1)$ satisfies
        \begin{align}\label{eqcomb3}
\tn{\Ghat_\qb(0,\hat\w_1)}\leq\eta\zeta\left(\zeta-\zeta^2\right)\,\wnorm^2\qnorm\left(13/4+2\wt{c_0}\right)+u_3 \wt{c_1} \sigma \qnorm \frac{\log(n)\sqrt{d}}{\sqrt{n}}\,.
\end{align}
\end{itemize}

\end{lemma}
\begin{proof} The lemma follows by collecting \eqref{eq:finite signal exp}, \eqref{eq:finite signal dev}, \eqref{eq:finite norm exp}, \eqref{eq:finite norm dev}, and applying union bound for the noise terms over $i\in[n],t\in\Rcc$.
\end{proof}

From the lemma above, we can show that provided $\qnorm$ is large enough with respect to $\wnorm$ and $\rho$ is small enough (see \eqref{eq:key finite q_1 condition} below) then the normalized signal relevance score is with high probability over the training set proportional to $\qnorm.$

{\begin{corollary}\label{cor:correlation}
Suppose there exists positive constant $\alpha\in(0,3/16)$ such that
\begin{align}\label{eq:key finite q_1 condition}
    A_Q:=\left(3/16-\rho^2/8\right)\qnorm - \left(9/4\abs{\rho}+1/16\right)\wnorm \geq \alpha \cdot Q\,.
\end{align}
Then, for sufficiently small step-size $\eta\propto \qnorm^{-2}$ and large enough $n$ there exist constants $c_0',c_0,c_1',c_1,c_3,c_3'$ such that with probability at least 
\begin{align}
&1-{c_0'\exp\left(-c_0 n\,\left(\,(W/Q)^2\mini \frac{\zeta^2\wnorm^2 T}{d}\right) \right)}-c_1'\exp\left({-c_1n^{1/3}(W/Q)^{4/3}\zeta^{4/3}\big/\log^{2/3}(n)}\right)\nn\\
&\qquad\qquad-c_3'd\exp\left({-c_3(n/d)^{1/3}(W/Q)^{4/3}\zeta^{4/3}\big/\log^{2/3}(n)}\right)\,, \label{eq:prob cor 1}
\end{align}
it holds that
\[
(\qstar+y\wstar)^\top\left(\frac{\Ghat_\qb(0,\what_1)}{\tn{\Ghat_\qb(0,\what_1)}}\right)\geq (\alpha/14) \cdot \qnorm\,.
\]
\end{corollary}
}

\begin{proof}
    Suppose \eqref{eq:key finite q_1 condition} holds and recall the notation $A_Q$ defined therein. 
Further suppose $n$ is large enough such that 
\eqref{eq:large enough n} holds so that we can invoke Lemma \ref{putogether1}. Therein set \[u_1\propto\eta\wnorm^2\zeta^2\frac{\sqrt{n}}{\log(n)} \propto(\wnorm/\qnorm)^2\zeta^2\frac{\sqrt{n}}{\log(n)} ,\]
\[
 u_3 \propto \eta\wnorm^2\zeta^2\frac{\sqrt{n}}{\log(n)\sqrt{d}}
 \propto (\wnorm/\qnorm)^2\zeta^2\frac{\sqrt{n}}{\log(n)\sqrt{d}}
 \]
and
\[
u_0 \propto \sqrt{n}\left((\wnorm/\qnorm)\mini \zeta\wnorm\sqrt{T/d}\right)\,.
\]
Note that the latter condition is consistent with \eqref{eq:large enough n}. On the other hand, the former two conditions are chosen so that the following two hold.

First,
    \[
\eta\zeta(\zeta-\zeta^2)\wnorm^2\qnorm \cdot A_Q > 2 u_1 \wt c_1\qnorm^2 \frac{\log{n}}{\sqrt{n}}\,.
    \]
    Thus, \eqref{eqcomb1} and setting $2\wt{c_0}=1/16$ give
    \[
        (\qstar+y\wstar)^\top\Ghat_\qb(0,\what_1) \geq \frac{1}{2}\eta\zeta(\zeta-\zeta^2)\wnorm^2\qnorm A_Q
    \]   
    Second,
    \[
\eta\zeta(\zeta-\zeta^2)\wnorm^2\qnorm > u_3 \wt c_1\qnorm \frac{\log{n}\sqrt{d}}{\sqrt{n}}
    \]
    Thus, from \eqref{eqcomb3}
    \[
\tn{\Ghat_\qb(0,\hat\w_1)}\leq2\eta\zeta\left(\zeta-\zeta^2\right)\,\wnorm^2\qnorm\left(13/4+1/16\right).
\]

Combining the above we conclude with the desired: the signal correlation is lower bounded by
\[
{(\qstar+y\wstar)^\top\Ghat_\qb(0,\what_1)}\big/{\tn{\Ghat_\qb(0,\what_1)}}\geq A_Q / 14 \geq (\alpha/14)\cdot \qnorm\,.
\]
\end{proof}

Having shown that the second gradient step has high signal relevance score, we can use it to show in the lemma below that the attention features are close to the signal features with high-probability.  This property will play a key role in finalizing the finite-sample analysis. Indeed, the rate of the event in \eqref{eq:correlation Xq1} will turn out to govern the error rate of our algorithm. 

\begin{lemma}[From learning the context to learning good features]
\label{lem:error rate given q1}
Suppose the second gradient step $\qb_1$ has correlation coefficient $\geq \alpha$ with $\qstar+y\wstar$ for any $y\in\{\pm1\}$, i.e.,
    $
        (\qstar+y\wstar)^\top\qb_1\big/\tn{\qb_1}\geq \alpha Q.
    $
Then, for any $\eps>0$ there exists sufficiently large $\gamma_*(\eps)$ and absolute constant $c>0$ such that setting $\qb_1^\gamma=\gamma\frac{\qb_1}{\tn{\qb_1}}$ for any $\gamma\geq \gamma_*(\eps)$, we have
\begin{align}
    \Pro_{(\X,y)}(\tn{\X^\top \phi(\X\qb_1^\gamma)-(\qstar+y\wstar)}\leq \eps)\geq 1-2Te^{-c\frac{\alpha^2Q^2}{\sigma^2}}.\label{eq:correlation Xq1}
\end{align}
\end{lemma}
\begin{proof} Recall $\X$ consists or $\zeta T$ relevant and $(1-\zeta)T$ irrelevant tokens. For relevant tokens, we have by assumption that 
$
B:=(\qstar+y\wstar)^\top \qbb_1\geq \alpha Q,
$
where we denote $\qbb_1:=\qb_1/\tn{\qb_1},$ for convenience.
Let $\z_1,\dots,\z_{(1-\zeta)T}$ denote the irrelevant tokens, which are $\sigma$-subgaussian. In order to ensure that softmax-attention perfectly selects the relevant tokens, we require
\[
M:=\max_{1\leq t\leq (1-\zeta)T} \z_t^T\qbb_1\leq {\alpha Q}/{2} = B/2.
\]
 Condition on this event, which holds with at least probability $1-Te^{-\frac{c\alpha^2\qnorm^2}{\sigma^2}}$. Then, the attention coefficients $a_t=[\phi(\gamma\X\qb_1)]_t$ are as follows:
 \[
 t\;\text{relevant:}\;\;a_t = \frac{1}{\zeta T + (1-\zeta) T e^{\gamma (M-B)} } \geq  \frac{1}{\zeta T + (1-\zeta) T e^{-\gamma B/2}}=:\frac{1}{\zeta T} a_R\,.
 \]
  \[
 t\;\text{irrelevant:}\;\;a_t = \frac{1}{\zeta T e^{\gamma(B-M)} + (1-\zeta) T  } \leq  \frac{1}{\zeta T e^{\gamma B/2} + (1-\zeta) T }=:\frac{1}{(1-\zeta)T} a_I =\frac{1}{(1-\zeta)T} \cdot\frac{1}{1+\frac{\zeta}{1-\zeta}e^{\gamma B/2}}\,.
 \]
Therefore,
\[
    \tn{\X^\top \phi(\gamma\X\qb_1)-(\qstar+y\wstar)} \leq (\qnorm+\wnorm)\left(1-a_R\right) + a_I \max_{1\leq t\leq (1-\zeta) T}\tn{\z_t}.
\]
To continue, further condition on the event 
\begin{align}\label{eq:norm z Q}
\max_{1\leq t\leq (1-\zeta) T}\tn{\z_t} \leq C\sigma\sqrt{d} + \alpha\qnorm
\end{align}
 which holds with probability at least $1-(1-\zeta)Te^{-c\alpha^2\qnorm^2/\sigma^2}.$
 Also note that $1-a_R=a_I$. Hence, with probability at least $1-2Te^{-c\alpha^2\qnorm^2/(8\sigma^2)}$
 \[
    \tn{\X^\top \phi(\gamma\X\qb_1)-(\qstar+y\wstar)} \leq \left((1+\alpha)\qnorm+\wnorm + C\sigma\sqrt{d}\right) a_I.
\]
The right hand-side above can be made smaller than $\eps$ by choosing $\gamma$ large enough (depending on $\eps, \alpha, \qnorm, \wnorm, \sigma$ and $d$). This completes the proof.
\end{proof}

Combining Lemma \ref{lem:error rate given q1} and Corollary \ref{cor:correlation} we arrive at the following result, which we state as a stand-alone theorem since it summarizes the effect of the first two-gradient steps on learning good (aka relevant) features.

\begin{theorem}[Summary result of first two GD steps]\label{thm:step 1 and 2 together}
Suppose $\qnorm, \wnorm$ and $\rho$ are such that there exists positive constant $\alpha\in(0,3/16)$ for which 
\[
    A_Q:=\left(3/16-\rho^2/8\right)\qnorm - \left(9/4\abs{\rho}+1/16\right)\wnorm \geq \alpha \cdot Q\,.
\]
Fix any $\eps>0$. For sufficiently small step-size $\eta\lesssim \qnorm^{-2}$, large enough $n$ and  sufficiently large $\gamma_\star=\gamma_\star(\eps)$ such that the following statements are true about the first two gradient steps:
\begin{align*}
    \what_1&:=\eta\Ghat_\w(0,0)
    \\  \qhat_1^\gamma&:=\Ghat_\qb(0,\what_1),\quad\text{for any } \gamma\geq \gamma_*\,.
\end{align*}
There exist constants $c_0',c_0,c_1',c_1,c_3,c_3',c$ such that with probability at least 
\begin{align}
&1-{c_0'\exp\left(-c_0\,n\, \left((W/Q)^2\mini \frac{\zeta^2\wnorm^2 T}{d}\right) \right)}-c_1'\exp\left({-c_1n^{1/3}(W/Q)^{4/3}\zeta^{4/3}\big/\log^{2/3}(n)}\right)\nn\\
&\qquad\qquad-c_3'd\exp\left({-c_3(n/d)^{1/3}(W/Q)^{4/3}\zeta^{4/3}\big/\log^{2/3}(n)}\right)\,, \label{eq:prob thm steps 1 and 2}
\end{align}
it holds for any fresh sample $(\X,y)$ that
\begin{align}
    \Pro(\tn{\X^\top \phi(\X\qb_1^\gamma)-(\qstar+y\wstar)}\leq \eps)\geq 1-2Te^{-c\frac{\alpha^2Q^2}{\sigma^2}}.\label{eq:correlation Xq1}
\end{align}
Moreover, it holds that
\[
\tn{\what_1} \leq \frac{c \zeta\wnorm}{\qnorm^2}\,.
\]

\end{theorem}
\input{finite sample gradient.tex}
\subsection{Third gradient step}
With the characterization of the quality of learnt features in Theorem \ref{thm:step 1 and 2 together}, we are now ready to turn our attention to the third gradient step. For this last step, it turns out  all we need is that $\what_2:=\eta\Ghat_\w(\qhat_1,\what_1)$ has a strictly positive correlation with $\wstar$. This is indeed the case and the result is formalized in the lemma below.

\begin{lemma}[Third step]\label{lem:third step}
Suppose that the first and second gradient steps $\what_1, \qhat_1$ are such that
the following hold. First, for  absolute constant $c_\eta>0$
\[
\tn{\what_1}\leq c_\eta \zeta \wnorm / \qnorm^2\,.
\]
Second, for $\eps>0$ and any fresh datapoint $(\X,y)$ there exists $\delta$ that does not depend on $\eps$ such that
\[
\Pro_{(\X,y)}\left(\tn{\X^\top\sft{\X\qhat_1}- (\qstar+y\wstar)}\leq \eps \right) \geq 1-\delta.
\]
Consider the third gradient step $\what_2:=\eta \Ghat_\qb(\qhat_1,\what_1)$. There exists absolute constants $c,C$ such that, {for all sufficiently small $\eps$ and $c_\eta$,} with probability at least $1-n\delta-2e^{cn}$,
\[
\frac{\what_2^\top\wstar}{\tn{\what_2}}\geq C\,\frac{\wnorm^2}{\qnorm}\,.
\]  
\end{lemma}

\begin{proof}
Note that the lemma's conclusion is insensitive to the choice of step-size $\eta$ for the third gradient step. Thus, without loss of generality, assume below that $\eta=1$.

Recall the gradient formula
\begin{align*}
\Ghat_\w(\qb,\w):=-\nabla_{\w}\Lch_{\Dc}(\thetab) = \frac{1}{n}\sum_{i\in[n]} (y_i-f_{\thetab}(\X_i)) \X_i^\top\sft{(\X_i\qb)}
\end{align*}
and denote for convenience
\[
\epsb_i:= \X_i^\top\sft{\X_i\qb}- (\qstar+y_i\wstar),\qquad i\in[n].
\]

With this notation, we can conveniently rewrite the third gradient step evaluated at $\qb:=\qb^\gamma:=\gamma\,\Ghat_\qb(0,\what_1)$ for any $\gamma>0$ as follows:
\begin{align}
\nn\what_2&:=\Ghat_\w(\qb,\what_1)=\frac{1}{n}\sum_{i\in[n]} y_i \X_i^\top\sft{\X_i\qb}-\frac{1}{n}\sum_{i\in[n]}\left(\what_1^T\X_i^\top\sft{\X_i\qb}\right)\X_i^\top\sft{\X_i\qb}
\\
\nn&=\wstar+ \qstar \left(\frac{1}{n}\sum_{i\in[n]} y_i\right) + \frac{1}{n}\sum_{i\in[n]} y_i \left(\X_i^\top\sft{\X_i\qb} - (\qstar+y_i\wstar) \right)
\\\nn
&\qquad\qquad-\frac{1}{n}\sum_{i\in[n]}\left(\what_1^T\X_i^\top\sft{\X_i\qb}\right) (\qstar+y_i\wstar) -\frac{1}{n}\sum_{i\in[n]}\left(\what_1^T\X_i^\top\sft{\X_i\qb}\right)\left(\X_i^\top\sft{\X_i\qb}- (\qstar+y_i\wstar)\right)\,\nn
\\\nn
&=\wstar+ \qstar \left(\frac{1}{n}\sum_{i\in[n]} y_i\right) + \frac{1}{n}\sum_{i\in[n]} y_i \epsb_i 
\\&- \frac{1}{n}\sum_{i\in[n]}\what_1^\top\epsb_i\epsb_i
- \frac{1}{n}\sum_{i\in[n]}\what_1^\top(\qstar+y_i\wstar)\epsb_i
- \frac{1}{n}\sum_{i\in[n]}\what_1^\top\epsb_i(\qstar+y_i\wstar)
-\frac{1}{n}\sum_{i\in[n]}\what_1^\top(\qstar+y_i\wstar)(\qstar+y_i\wstar)\,. \label{eq:w2 nice}
\end{align}

In order to control the correlation $\what_2^\top\wstar/\tn{\what_2}$ it suffices to control $\what_2^\top\vb$ for arbitrary $\vb\in\R^d$. In view of the expression above, it will suffice bounding the two random terms below:
\begin{align*}
    \term{I}&:=\abs{\frac{1}{n}\sum_{i\in[n]} y_i}
    \\
    \term{II}&:= \tn{\epsb_i}=\tn{\X_i^\top\sft{\X_i\qb} - (\qstar+y\wstar) } 
\end{align*}
For the first term, we have with probability at least $1-2e^{u_1^2/2}$ that 
\[\term{I}\leq u_1/\sqrt{n}.\]
For the second term, we know by assumption that with probability at least $1-n \delta$,
\[\term{II}\leq \eps.\]
Putting the above together, with probability at least $1-2e^{-u^2/2}-n\delta$, we have that
\begin{align*}
    \wstar^\top\what_2 &\geq \wnorm^2 - \abs{\rho}\qnorm\wnorm \,\frac{u_1}{\sqrt{n} }- \eps\,\wnorm - \frac{3}{2}\eta\zeta\left( \eps^2 \wnorm^2 + 2\eps \wnorm^2\left(\qnorm+\wnorm\right) +\wnorm^2\left(\qnorm+\wnorm\right)^2\right)
    \\
    &\geq \wnorm^2\left( 1- \abs{\rho}\,\frac{\qnorm}{\wnorm} \,\frac{u_1}{\sqrt{n} }- \frac{\eps}{\wnorm} - \frac{3}{2}\eta\zeta\left( \eps^2 + 2\eps \left(\qnorm+\wnorm\right) +\left(\qnorm+\wnorm\right)^2\right)\right)\,
\end{align*}
where we also used the lemma's assumption on $\tn{\what_1}\leq (3/2)\eta\zeta\wnorm$.

To further lower bound $\wstar^\top\what_2$, recall that $\abs{\rho}\leq \wnorm/\qnorm$, $\wnorm\gtrsim 1$ and that $\eps$ can be made arbitrarily small constant. Further pick 
\begin{align}\label{eq: third step u1 eta}
u_1 = c_1 \sqrt{n} \qquad \text{and} \qquad \eta = c_\eta/\qnorm^2
\end{align}
for sufficiently small constants $c_1$ and $c_\eta$. With these, we guarantee with probability at least $1-2e^{-cn}-n\delta$ that 
\begin{align}
    \what_2^\top\wstar \gtrsim \wnorm^2\,.
\end{align}

Next, we use similar arguments to bound $\tn{\what_2}.$ Conditioning on the event where the bounds derived above hold for $\term{I}$ and $\term{II}$, we have from \eqref{eq:w2 nice} that
\begin{align}
    \tn{\what_2} \leq \wnorm + \qnorm \frac{u_1}{\sqrt{n}} + \eps + \frac{3}{2}\eta\zeta\wnorm\left(\eps^2+2\eps\left(\qnorm+\wnorm\right)+\left(\qnorm+\wnorm\right)^2\right)\nn
    \lesssim \qnorm ,
\end{align}
where in the second inequality, we chose $u_1, \eta$ as in \eqref{eq: third step u1 eta} and used again that $\eps$ is arbitrarily small constant, as well as, $\qnorm > \wnorm \gtrsim 1.$

All the above combined, shows that 
\[
\frac{\what_2^\top\wstar}{\tn{\what_2}} \gtrsim \frac{\wnorm^2}{\qnorm}.
\]
This completes the proof.
\end{proof}

\subsection{De-biasing step}
\begin{lemma}[Debiasing predictions]\label{debiasing} For some $\eps>0$, suppose $\qb_1$ is such that a test example $(y,\X)$ satisfies
\[
\Pro_{(\X,y)}\left(\tn{\X^\top\sft{\X\qhat_1}- (\qstar+y\wstar)}\leq \eps \right) \geq 1-\delta
\]
and that $\w_2$ is such that
\[
\frac{\w_2^\top \wstar}{\tn{\w_2}} > 4\eps.
\]
Given a fresh dataset $\Sc=(y_i,\X_i)_{i=1}^n$, set $b=\frac{1}{n}\sum_{i=1}^n f_{\thetab}(\X_i)=\frac{1}{n}\sum_{i=1}^n \w_2^\top \vb_i$ where $\vb_i:=\X_i^\top \phi(\X_i\qb_1)$. Set the debiased classifier $f'_{\thetab}(\X)=f_{\thetab}(\X)-b$.  
Suppose $n\geq 8\log\left(\frac{2}{\delta n}\right)$. Then, with probability $1-2\delta n$ over $\Sc$, the test error of $f'_{\thetab}$ obeys
\[
\err(f'_{\thetab})\leq \delta.
\]
\end{lemma}
\begin{proof} 
First, let us prove the following intermediate statement: With probability $1-2\delta n$ over $\Sc$, for a new test sample $(y,\X)$, with probability $1-\delta$,
\begin{align}
|yf'_{\thetab}(\X)-\w_2^\top \wstar|\leq\sqrt{\frac{2\log(2/\delta n)}{n}}\cdot \w_2^\top \wstar+ 2{\eps\tn{\w_2}}.\label{debias bound}
\end{align}
To see the above, start with observing that, with probability $1-n\delta$ over the dataset $(y_i,\X_i)_{i=1}^n$, for each $\vb_i$,
\[
|\w_2^\top\vb_i-\w_2^\top(\qstar+y_i\wstar)|\leq \eps\tn{\w_2}.
\]
Set $\bar{b}=\w_2^\top\qstar$ abd $\bar{y}=|\frac{1}{n}\sum_{i=1}^n y_i|$. With probability $1-\delta n$, $\bar{y}\leq \sqrt{\frac{2\log(2/\delta n)}{n}}$. Combining, with overall probability at least $1-2\delta n$, the classifier bias obeys
\[
|b-\bar{b}|\leq \left|\w_2^\top\wstar\right| \sqrt{\frac{2\log(2/\delta n)}{n}}+\eps\tn{\w_2}.
\]
To finalize, for a new sample $(y,\X)$, with probability $1-\delta$, we have that $|\w_2^\top\vb-\w_2^\top(\qstar+y\wstar)|\leq \eps\tn{\w_2}$ where $\vb=\X^\top \phi(\X\qb_1)$. Thus, the prediction $f'(\X)=f(\X)-b$ obeys
\begin{align}\label{above debias eq}
|yf'_{\thetab}(\X)- y (f_{\thetab}(\X)-\bar{b})|\leq |b-\bar{b}|\leq\left|\w_2^\top\wstar\right| \sqrt{\frac{2\log(2/\delta n)}{n}}+\eps\tn{\w_2}.
\end{align}
To conclude with \eqref{debias bound}, note that
\[
| y (f_{\thetab}(\X)-\bar{b})-\w_2^\top\wstar|\leq \eps\tn{\w_2},
\]
and apply triangle inequality with \eqref{above debias eq}. 

To prove the statement of the lemma, note that, when $n\geq 8\log(2/\delta n)$ and $\w_2^\top \wstar> 4\eps\tn{\w_2}$, a test sample (with $\geq 1-\delta$ probability) obeys
\begin{align}
y f'_{\thetab}(\X)\geq \w_2^\top\wstar-\sqrt{\frac{2\log(2/\delta n)}{n}}\w_2^\top\wstar-2\eps\tn{\w_2}\geq 0.5\w_2^\top\wstar-2\eps\tn{\w_2}>0.
\end{align}
Thus, the classifier makes the correct decision with the same probability.
\end{proof}

\subsection{Finishing the finite sample analysis}

\begin{theorem}[Main theorem: Finite-sample]\label{thm:main finite appendix}
    Suppose $\qnorm, \wnorm$ and $\rho$ are such that there exists positive constant $\alpha\in(0,3/16)$ for which 
\begin{align}\label{eq:SNRcondition finite}
    \left(3/16-\rho^2/8\right)\qnorm - \left(9/4\abs{\rho}+1/16\right)\wnorm \geq \alpha \cdot Q\,.
\end{align}
Fix any $\eps>0$. For sufficiently small step-size $\eta\lesssim\qnorm^{-2},$  sufficiently large step-size $\gamma=\gamma(\eps)$, and large enough $n$, there exist constants $c_j',c_j,j=0,1,2,3,4$ such that the following statements hold with 
probability at least 
\begin{align}\label{long prob bound}
&1-{c_0'\exp\left(-c_0\,n\, \left((W/Q)^2\mini \left(\frac{\zeta^2\wnorm^2 T}{d} \mini 1\right)\right) \right)}-c_1'\exp\left({-c_1n^{1/3}(W/Q)^{4/3}\zeta^{4/3}\big/\log^{2/3}(n)}\right)\nn\\
&\qquad\qquad-c_3'd\exp\left({-c_3(n/d)^{1/3}(W/Q)^{4/3}\zeta^{4/3}\big/\log^{2/3}(n)}\right) - 2nT\exp\left(-c_4\alpha^2\qnorm^2\right)\,,
\end{align}
over the training set:

\noindent~~~\textbf{1. Prompt attends to relevant tokens:} For any test sample $(\X,y)$, with probability at least $1- 2T\exp\left(-c_4\alpha^2\qnorm^2\right)$, the attention coefficients $a_t=\left[\sft{\X\qhat_1}\right]_t$ after the second gradient step satisfy:
\begin{align}
    a_t \begin{cases}
    \geq \frac{1-\eps}{\zeta T} &\text{ $t$ relevant\,}
    \\
    \leq \frac{\eps}{(1-\zeta) T} &\text{ $t$ irrelevant\,.}
    \end{cases}
\end{align}

\noindent~~~\textbf{2. Prompt learns relevant features:} The prompt attention mechanism outputs relevant tokens with the same probability. Concretely, 
\[
\Pro_{(\X,y)}\left(\tn{\X^\top\sft{\X\qhat_1}- (\qstar+y\wstar)}\leq \eps \right) \geq 1- 2T\exp\left(-c_4\alpha^2\qnorm^2\right)\,.
\]

\noindent~~~\textbf{3. Test error:} The test error of the model $f_{\thetab}'$ satisfies
\[
\err(f'_{\thetab})\leq 2T\exp\left(-c_4\alpha^2\qnorm^2\right)\,.
\]
\end{theorem}

\begin{proof}
    The theorem follows by combining Theorem \ref{thm:step 1 and 2 together}, Lemma \ref{lem:third step} and Lemma \ref{debiasing}. 
\end{proof}

\section{Proofs for Population-gradient Analysis in Section \ref{sec:population analysis main}}\label{sec:proofs grad population}

This section includes the missing proofs of all the results in Section \ref{sec:population analysis main} regarding population analysis of Algorithm \ref{eq:algo}.

\subsection{Proof of Lemma \ref{lem:q_der_pop}}\label{sec:proof lemma q population}
We repeat here the lemma for convenience also stated for general (not necessarily isotropic) noise covariance $\Sigmab$.
\begin{lemma}The second population gradient step $\qb_1=\gamma\G_\w(\w_1,0)$ satisfies the following for $\alpha:=\eta\zeta$
\begin{align}
\G_\qb(0,\alpha\wstar)
&= \left((\zeta-\zeta^2)\left(\alpha\wnorm^2+\alpha^2{\wstar^\top\Sigmab\wstar}/{T}\right)-\alpha^2(\zeta^2-\zeta^3)\left(\wnorm^4 + \left(\wstar^\top\qstar\right)^2\right)\right)\,\qstar\nn
\\ 
&\qquad\qquad 
+\left(\left((\zeta-\zeta^2)-2(\zeta^2-\zeta^3)\alpha\wnorm^2\right)\alpha\left(\wstar^\top\qstar\right) \right)\,\wstar\nn
\\
&\qquad\qquad -\left(\left(1+{2}/{T}\right)(\zeta-\zeta^2)\alpha\left(\wstar^\top\qstar\right)\right)\alpha\Sigmab\wstar
\label{eq:pop_one_step_grad_q appendix}
\end{align}
\end{lemma}
\begin{proof}
The lemma follows immediately from 
 Eqn. \eqref{eq:mean_one_step_grad_q} of Lemma \ref{lem:conc_grad_q} by recognizing that for $\w=\alpha\wstar$ it holds $\corq=\alpha\qstar^\top\wstar$ and $\corw=\alpha\wnorm^2$.    
\end{proof}

\subsection{Corollary \ref{cor1}}
\begin{corollary} \label{cor1} 
Suppose small enough step-size $\eta$ obeying 
\begin{subequations}\label{eq:eta small pop}
\begin{align}\label{eq:eta_1}
\eta\left( \zeta^2\left(\wnorm^2+\qnorm^2\right)-{\zeta \cdot\sigma^2/T}\right)&\leq {1}/{2}.
\\
\label{eq:eta_2}
\eta\zeta\left(2\zeta\wnorm^2+(1+2/T)\sigma^2\right)&\leq {5}/{4}.
\\
\label{eq:eta_3}
\eta\zeta\left(\sigma^2/T\right)\leq{1}/2.
\end{align}
\end{subequations}
Then, for $C_1\in[1/2,3/2]$ and $ C_2\in[-1/4,1]$, we have that
\begin{align*}
{\qb_1}&= \gamma {\eta\zeta(\zeta-\zeta^2)}\wnorm\left(C_1 \wnorm\qb_\star + C_2 \rho \qnorm\wstar\right)\,.
\end{align*}
In particular, 
$
\qstar^\top\qb_1=\gamma{\eta\zeta(\zeta-\zeta^2)}\wnorm^2\qnorm^2\left(C_1 + C_2 \rho^2\right)$
and $
\wstar^\top\qb_1=\gamma{\eta\zeta(\zeta-\zeta^2)}\wnorm^3\qnorm\rho\left(C_1 + C_2\right).
$
\end{corollary}

\begin{proof}
Set $\alpha=\eta\zeta$ and 
\begin{subequations}
\begin{align}
3/2 \geq C_1&:=\left(1+\alpha\sigma^2/{T}\right)-\alpha\zeta\left(\wnorm^2 + \rho^2\qnorm^2\right)\geq 1/2.\label{eq:C1}
\\
1\geq C_2&:=1-2\alpha\zeta\wnorm^2-(1+2/T)\alpha\sigma^2\geq -1/4.\label{eq:C2}
\end{align}
\end{subequations}
The gradient formula  follows directly from \eqref{eq:pop_one_step_grad_q}. For the lower/upper bounds on $C_1,C_2$ use \eqref{eq:eta_1}, \eqref{eq:eta_2} and \eqref{eq:eta_3}.
\end{proof}

\begin{remark}[Condition on correlation]
To classify correctly the signal tokens, we need 
\begin{align}
y\wstar^\top(\qstar+y\wstar)>0 &\quad\Longleftrightarrow\quad
y\rho\qnorm+\wnorm>0  \Longleftrightarrow\quad
\abs{\rho}<\wnorm/\qnorm\label{eq:rho small}
\end{align}
Note that if \eqref{eq:rho small} holds, then 
\begin{align}\label{eq:C1>C2}
C_1 \geq 1+\alpha\sigma^2/T-2\alpha\zeta\wnorm^2 = C_2 + (1+3/T)\alpha\sigma^2.
\end{align}
\end{remark}

\subsection{Proof of Theorem \ref{thm:main grad pop}}\label{sec:proof of gradient pop}

We start by showing that for any $\eps>0$, and all sufficiently large $\gamma\geq\gamma_\star(\eps)$, it holds
\begin{align}
    \Pro_{(\X,y)\sim\Dc}(\tn{\X^\top \phi(\X\qb_1^\gamma)-(\qstar+y\wstar)}\leq \eps)\geq 1-2Te^{-c\frac{\alpha^2Q^2}{\sigma^2}}=:1-\delta.\label{eq:correlation Xq1 for asymptotic}
\end{align}
We can get this by applying Lemma \ref{lem:error rate given q1} provided only that we show the correlation of the normalized gradient-step $\qb_1^\gamma/\|\qb_1^\gamma\|$ ($=\bar\qb_1$ below) with signal-relevant tokens is at least $\alpha\qnorm$. Concretely, we have  from Corollary \ref{cor1} that 
$
\qb_1^\gamma=\gamma \qb_1:= \gamma 
\eta\zeta(\zeta-\zeta^2)\wnorm\left(
C_1   \wnorm\qb_\star + C_2 \rho \qnorm\wstar\right)
$
with $3/2\geq C_1\geq 1/2$, $1\geq C_2\geq -1/4$. Define for convenience $\qb_1:=\eta\zeta(\zeta-\zeta^2)\wnorm\left(
C_1   \wnorm\qb_\star + C_2 \rho \qnorm\wstar\right)$ and consider the normalized gradient step
\[
\bar\qb_1:=\frac{\qb_1}{\|\qb_1\|_2}=\frac{C_1\wnorm\qb_\star+C_2 \rho \qnorm\wstar}{\sqrt{C_1^2\wnorm^2\qnorm^2+C_2^2\rho^2\qnorm^2\wnorm^2+2C_1C_2\rho^2\wnorm^2\qnorm^2}}=
\frac{C_1\wnorm\qb_\star+C_2 \rho \qnorm\wstar}{\qnorm\wnorm\sqrt{C_1^2+C_2\rho^2(C_2+2C_1)}}\,.
\]
We can lower-bound its correlation with a signal token $\qstar+y\wstar$ as follows:
\begin{align*}
\bar\qb_1^T(\qstar+y\wstar)&=
\frac{C_1\wnorm(\qnorm^2+y\rho\wnorm\qnorm)+C_2 \rho \qnorm(\rho\wnorm\qnorm+y\wnorm^2)}{\qnorm\wnorm\sqrt{C_1^2+C_2\rho^2(C_2+2C_1)}}
=
\frac{C_1(\qnorm+y\rho\wnorm)+C_2 \rho (\rho\qnorm+y\wnorm)}{\sqrt{C_1^2+C_2\rho^2(C_2+2C_1)}}
\\
&=\frac{(C_1+C_2\rho^2)\qnorm+y\rho(C_1+C_2)\wnorm)}{\sqrt{C_1^2+C_2\rho^2(C_2+2C_1)}}
\\
&\geq\frac{(C_1+C_2\rho^2)\qnorm+y\rho(C_1+C_2)\wnorm)}{C_1\sqrt{1+3\rho^2}}
\geq \frac{1+(C_2/C_1)\rho^2}{
\sqrt{1+3\rho^2}
}\qnorm - \frac{\abs{\rho}\left(1+C_2/C_1\right)}{\sqrt{1+3\rho^2}}
\\
&\geq \frac{\big(1+(C_2/C_1)\rho^2\big)\qnorm-\left(\abs{\rho}(1+C_2/C_1)\right)\wnorm}{\sqrt{1+3\rho^2}} 
\\
&\geq \frac{(1-\rho^2/2)\qnorm -2\abs{\rho}\wnorm}{
\sqrt{1+3\rho^2}}
\\
&\geq \alpha\qnorm\,,
    \end{align*}
where: (i) the inequality $\sqrt{C_1^2+\rho^2 C_2(C_2+2C_1)}\leq C_1\sqrt{1+3\rho^2}$ used in in the third line follows because $C_1>0$ and $C_2\leq C_1$ from \eqref{eq:C1>C2}; (ii) the penultimate inequality uses $C_2/C_1\in[-1/2,1]$ (for the lower bound recall $C_2\geq-1/4, C_1\geq 1/2$); (iii) the last inequality is because of the theorem's assumption in   \eqref{eq:clean noise condition}.

Next, recall that
\begin{align}
\w_2^\gamma:=\G_\w(0,\qb_1^\gamma)&=\E_{(\X,y)\sim\Dc}\left[y\X^\top\sft{\X\qb_1^\gamma}\right] = \E_{(\X,y)\sim\Dc}\left[y\vb(\X)\right]\,,
\label{eq:w2 pop expand popop}
\end{align}
where we set $\vb(\X)=\X^\top\sft{\X\qb_1^\gamma}$ for convenience. Thus, the prediction of the model with parameters $\thetab=(\w_2^\gamma,\qb_1^\gamma)$ for test datapoint $(\Xt,\yt)$ is 
\begin{align}
\hat o:=\yt f_{\thetab}(\Xt) = \yt\inp{\w_2^\gamma}{\vb(\Xt)}= \underbrace{\inp{\w_2^\gamma}{\yt\qstar+\wstar}}_{\ohat_1} 
 + \underbrace{\yt\inp{\w_2^\gamma}{\vb(\Xt)-(\qstar+\yt\wstar)}}_{\ohat_2}\,.
\end{align}
Let $\Ec$ denote the event for which $\|\vb(\Xt)-(\qstar+\yt\wstar)\|\leq \eps$, which has probability at least $1-\delta$ by \eqref{eq:correlation Xq1 for asymptotic}. 
Note that
\begin{align}
\acc(f_{\thetab})=\Pro\left(\ohat<0\right) \leq     \Pro\left(\ohat<0\,|\,\Ec\right) + \Pr(\Ec) =  \Pro\left(\ohat<0\,|\,\Ec\right) + \delta\,.
\end{align}
Hence, our goal below is to bound $\Pro\left(\ohat<0\,|\,\Ec\right)$. {In fact, we will show that 
$\Pro\left(\ohat<0\,|\,\Ec\right)\leq 0$, so the error rate is $\delta$ as stated in the theorem.}

To do this,
condition on $\Ec$ for which  $|\ohat_2|\leq \eps\|\w_2^\gamma\|$, thus 
\[
\ohat \geq \ohat_1 - \eps\|\w_2^\gamma\| \,.
\]
Further note that $\ohat_1=\inp{\w_2^\gamma}{\wstar}-\abs{\inp{\w_2^\gamma}{\qstar}}$.  Thus, it suffices to show that
\begin{align}
 \inp{\w_2^\gamma}{\wstar} \geq \abs{\inp{\w_2^\gamma}{\qstar}}+\eps\|\w_2^\gamma\| \label{eq:2show popop}\,.
\end{align}
For this, go back to $\w_2^\gamma$ and write continuing from \eqref{eq:w2 pop expand popop}
\begin{align}\nn
\w_2^\gamma&= \E_{(\X,y)\sim\Dc}\left[y\qstar+\wstar\right] + \E_{(\X,y)\sim\Dc}\left[y\left(\vb(\X)-(\qstar+y\wstar)\right)\right]\nn\\
&= \wstar + \E_{(\X,y)\sim\Dc}\left[y\left(\vb(\X)-(\qstar+y\wstar)\right)\right]\nn\,.
\end{align}
Denote for convenience $\eb:=\eb(y,\X):=\vb(\X)-(\qstar+y\wstar)$. 
Thus,
\begin{align}\nn
\inp{\w_2^\gamma}{\wstar} &= \wnorm^2 + \E\left[y\inp{\wstar}{\eb}\right]   \\ 
\inp{\w_2^\gamma}{\qstar} &= \rho\qnorm\wnorm + \E\left[y\inp{\qstar}{\eb}\right]\nn
\\ 
\|\w_2^\gamma\| &\leq \wnorm + \E\left[\|{\eb}\|\right]\,.\nn
\end{align}
where in the last line we used triangle and Jensen's inequalities.
We now compute (recall  $\Ec$ is the event for which $\|\eb\|\leq \eps$):
\begin{align}
    \E\left[\abs{\inp{\wstar}{\eb}}\right] &\leq \E\left[\abs{\inp{\wstar}{\eb}}\,|\,\Ec\right] + \E\left[\abs{\inp{\wstar}{\eb}}\,|\,\Ec^c\right]\Pro(\Ec)\nn
    \\
    &\leq \eps\,\wnorm + \wnorm\,\delta\, \E\left[\|\eb\|\,|\,\Ec^c\right]\,.
\end{align}
Similarly, we can upper bound  $\E\left[\abs{\inp{\qstar}{\eb}}\right]$ and $\E\left[\|{\eb}\|\right]$. Combining these with the above displays, the desired Eq. \eqref{eq:2show popop} holds provided:
\begin{align}
\wnorm^2-\abs{\rho}\qnorm\wnorm \geq \eps\left(\wnorm+\qnorm\right) + \delta\,\left(\wnorm+\qnorm\right)\, B  
 + \eps \delta B + \eps^2 \label{eq:2show popop 2}\,.
\end{align}
Above, we have denoted $B:=\E\left[\|\eb\|\,|\,\Ec^c\right]$. 
$
\|\w_2^\gamma\|\leq \wnorm + B\,.
$
Note that the LHS of \ref{eq:2show popop 2} is $>0$ because of Assumption \ref{ass:rho general but small} that $\abs{\rho}\leq \wnorm/\qnorm$\,. Thus, we can guarantee \eqref{eq:2show popop 2} holds once $\eps$ is small enough (by making $\gamma$ large enough) and  $\delta$ is also small enough (by making $\gamma$ large enough). It only 
remains to bound $B$. To do this, note that
\begin{align*}
    \|\eb\|_2\leq \|\vb(\X)\| + \qnorm + \wnorm \leq  
 \max_{t\in[T]}\|\x_t\| + \qnorm + \wnorm \leq 2(\qnorm+\wnorm) + \max_{t\in\Rcc}\|\z_t\|\,.
\end{align*}

By Lemma \ref{lem:todo max}  we further have that
\[
\E\left[\max_{t\in\Rcc}\|\z_t\|_2\bgl\Ec^c\right]\Pro(\Ec^c)\leq \delta\cdot C\sigma\sqrt{d}\sqrt{\log\left({2T}/{\delta}\right)}\,.
\]
Hence,
\[
B\leq 2(\qnorm+\wnorm)+C\sigma\sqrt{d}\sqrt{\log\left({2T}/{\delta}\right)}\,. 
\]

\subsubsection{Auxiliary lemma}

\begin{lemma}[Subgaussian euclidean-norm tail control]\label{lem:todo max}
Let  Let $\z_i\in\R^d, i\in[N]$ be $K$-subgaussian random vectors. Then, for any event $\Ec$ with $\Pro(\Ec)=\delta$, it holds that
\[\E\left[\max_{i\in[N]}\tn{\z_i}\bgl\Ec^c\right]\leq 12 K\sqrt{d}\sqrt{\log\left({2N}/{\delta}\right)}.
\]
\end{lemma}
\begin{proof}
Set $Z=\max_{i\in[n]}\tn{\z_i}$ and define event $\Bc=\{Z\geq M\}$ for $M:=4K\sqrt{d}\sqrt{\log\left({2N}/{\delta}\right)}.$ By Fact \ref{fact:norm subg} for all $t>0$,  $\Pro\left(Z>t\right)\leq 2Ne^{-t^2/(16dK^2)}$. Thus, by choice of $M$, 
$\Pro(\Bc)\leq \Pro(\Ec) = \delta$.

Denote the pdf and  cdf complement of $Z$ by $f_Z,Q_Z$ respectively. Observe that, we set $Q_Z(M)\leq \delta$. Using integration by parts we have, 
\begin{align*}
    \E[Z|\Bc]\Pro(\Bc)&=\int_{M}^\infty zf_Z(z)\mathrm{d}z=-\int_{M}^\infty zdQ_Z(z)=\int_{M}^\infty Q_Z(z)\mathrm{d}z-[Q_Z(z)z]_M^\infty\nn
\\
&= \int_{M}^\infty Q_Z(z)\mathrm{d}z+\delta M
    \\
    &=\delta M + \int_{M}^{\infty}\Pro(Z\geq t)\mathrm{d}t\leq \delta M + \int_{M}^{\infty}2Ne^{-t^2/(16dK^2)}\mathrm{d}t
    \\
    &\leq \delta M + 2\sqrt{2}K\sqrt{d}\,(2N)\,\int_{\sqrt{2\log(2N/\delta)}}^{\infty}e^{-u^2/2}\mathrm{d}u
    \\
    &= \delta\,4 K\sqrt{d}\sqrt{\log(2N/\delta)} + 2\sqrt{\pi}K\sqrt{d}\,\delta \leq 2 \delta M.
\end{align*}
We can conclude the proof by noting:
\begin{align*}
\E\left[Z\bgl \Ec\right]\Pro(\Ec)&=\E[Z\bgl \Ec\cap \Bc^c]\Pro(\Ec\cap \Bc^c)+\E[Z\bgl \Ec\cap \Bc]\Pro(\Ec\cap \Bc)\\
&\leq M\delta+\E[Z\bgl \Bc]\Pro(\Bc)\,.
\end{align*}
\end{proof}


%
%

\section{Proofs of results on discrete datasets}

\subsection{Proof of Theorem \ref{separate thm} and Observation \ref{obv 1}}

\noindent $\bullet$ \textbf{Proof for \Prml:} Let $\wstab=\wstar/\tn{\wstar}$ and $\qstab=\qstar/\tn{\qstar}$. $\qstar'$ be the projection of $\qstar$ to the orthogonal complement of $\wstar$ i.e.~$\qstar'=\qstar-\wstab\wstab^\top\qstar$. Similarly, let $\wstar'$ be the projection of $\wstar$ to the orthogonal complement of $\qstar$ i.e.~$\wstar'=\wstar-\qstab\qstab^\top\wstar$. Denote correlation coefficient between two vectors by $\rho(\ab,\bb)=\frac{\ab^\top\bb}{\tn{\ab}\tn{\bb}}$.

To proceed, observe that, $\qstar'^\top \qstar=\tn{\qstar}^2-(\wstab^\top\qstar)^2=\tn{\qstar}^2(1-\rho(\qstar,\wstar)^2)>0$. The positivity follows from the fact that $\qstar,\wstar$ are not parallel, thus, the absolute value of their correlation coefficient is strictly bounded away from $1$. Similarly $\wstar'^\top \wstar=\tn{\wstar}^2(1-\rho(\qstar,\wstar)^2)$>0. To proceed, set $\bro:=1-\rho(\qstar,\wstar)^2$ and observe that the classifier $\thetab=(\wstar',\Gamma\qstar')$ achieves the attention scores
\[
\ab_i=\phi(\X\qstar')_i=\begin{cases}S^{-1}e^{\tn{\qstar}^2\Gamma \bro}\quad\text{if}\quad i~\text{relevant}\,,\\S^{-1}e^{-\tn{\qstar}^2\Gamma \delq \bro}\quad\text{if}\quad i~\text{irrelevant}\,,\end{cases}
\]
where $S=T\zeta e^{\tn{\qstar}^2\Gamma \bro}+T(1-\zeta)e^{-\tn{\qstar}^2\Gamma \delq \bro}$. Using orthogonality of $\wstar'$ and $\qstar$, the final prediction obeys
\[
yf_{\thetab}(\X)=\tn{\wstar}^2\bro S^{-1} \left[\zeta e^{\tn{\qstar}^2\Gamma \bro}-\delw(1-\zeta)e^{-\tn{\qstar}^2\Gamma \delq \bro}\right]\,.
\]
The classifier achieves perfect accuracy when $\zeta e^{\tn{\qstar}^2\Gamma \bro}> |\delw| (1-\zeta)e^{-\tn{\qstar}^2\Gamma \delq \bro}$. Since we have $\delq\geq 0$ and we have assumed $\delw$ is a $C$-bounded variable (i.e.~$|\delw|\leq C$), thus, the desired inequality can be guaranteed by choosing 
\[
\Gamma>\frac{1}{\tn{\qstar}^2\bro}\log(\frac{C(1-\zeta)}{\zeta}).
\]
\smallskip
\noindent $\bullet$ \textbf{Proof for Observation \ref{obv 1}:} To prove this, observe that for any $\delq=\Delta^q$, $\delw=\Delta^w$ choices, using orthogonality of $\qstar,\wstar'$, for any $(y,\X)\sim\Dc$, we have
\[
y\flin(\wstar')= \tn{\wstar}^2\bro(\zeta-(1-\zeta)\delw).
\]
Thus, as long as $\delw\neq \zeta/(1-\zeta)$, $\sign{y\flin(\wstar')}$ is always $1$ or always $-1$, resulting in perfect accuracy for $\wstar'$ or $-\wstar'$.

\smallskip
\noindent $\bullet$ \textbf{Proof for Self-attention:} The proof is provided under Theorem \ref{self-att thm}.

\smallskip
\noindent $\bullet$ \textbf{Proof for Linear \Prml:} Let $W_1=\w^\top\wstar,W_2=\w^\top \qstar$, $Q_1=\qb^\top\qstar,Q_2=\qb^\top \wstar$. Since \ctx-irrelevant tokens are of the form $-\delq \qstar-y\delw \wstar$, the model decision is given by $\frac{1}{T}f(\X)=\frac{1}{T}\w^\top\X^\top\X\qb=\zeta\w^\top (y\wstar+\qstar)(y\wstar+\qstar)^\top \qb+(1-\zeta)\w^\top (y\delw\wstar+\delq \qstar)(y\delw\wstar+\delq \qstar)^\top \qb$ and
\begin{align*}
\frac{1}{T}f(\X)&=\zeta(yW_1+W_2)(Q_1+yQ_2)+(1-\zeta)(y\delw W_1+\delq W_2)(y\delw Q_2+\delq Q_1)\\
&=\zeta y(W_1Q_1+W_2Q_2)+\zeta(W_2Q_1+W_1Q_2)+\\
&~~~(1-\zeta)y\delq\delw( W_1Q_1+W_2Q_2)+(1-\zeta)({\delq}^2W_2Q_1+{\delw}^2W_1Q_2).\\
\frac{yf(\X)}{T}&=(\zeta+(1-\zeta)\delq\delw)(W_1Q_1+W_2Q_2)+y((\zeta+(1-\zeta){\delq}^2)W_2Q_1+(\zeta+(1-\zeta){\delw}^2)W_1Q_2).
\end{align*}
To proceed, set $(\delq,\delw)$ to be $(0,0)$ or $(\Delta,-\Delta)$ equally-likely for $\Delta> \sqrt{\zeta/(1-\zeta)}$. For fixed $\Delta$, for any choice of $W_1,W_2,Q_1,Q_2$ observe that, with $1/2$ probability the event $E=\{y((\zeta+(1-\zeta){\delq}^2)W_2Q_1+(\zeta+(1-\zeta){\delw}^2)W_1Q_2)\leq 0\}$ happens. On this event (which is over the label $y$), probability that $(\zeta+(1-\zeta)\delq\delw)(W_1Q_1+W_2Q_2)>0$ is at most $1/2$ because $\text{sign}(\zeta+(1-\zeta)\delq\delw)$ is Rademacher variable. Combining, we find that $\Pro(\frac{yf(\X)}{T}\leq 0)\geq 25\%$ as advertised whenever $\Delta> \sqrt{\zeta/(1-\zeta)}$.

\subsection{Failure proof for Self-attention}\label{app satt fail}
We have the following theorem regarding self-attention.
\begin{theorem} \label{self-att thm}Fix $\Delta>0$ to be sufficiently large. In \eqref{eq:CGMM}, choose $\del=(\delq,\delw)$ to be $(0,0)$ or $({\Delta,\Delta})$ equally-likely, where $\Delta>1/(1-\zeta)^2$. 
\begin{itemize}
\item For any choice of $(\Ub=\ones\ub^\top,\W)$, $\fsat(\ones\ub^\top,\W)$ achieves 50\% accuracy (i.e.~random guess).
\item For any choice of $(\Ub,\W)$, there exists a \eqref{eq:CGMM} distribution with adversarial relevance set choices such that $\fsat(\Ub,\W)$ achieves 50\% accuracy.
\end{itemize}
Here, adversarial relevance set choice means that, the relevance set can be chosen adaptively to the label $y$, out-of-context term $\del$, and the self-attention model weights $(\Ub,\W)$ to cause misclassification.
\end{theorem}
\begin{proof} Let $\wtt=\W\wstar$ and $\qtt=\W\qstar$. Also let $b_w=\ub^\top \wstar$ and $b_q=\ub^\top \qstar$. Since $\W$ is allowed to be full-rank and arbitrary, $\wtt,\qtt$ are allowed to be arbitrary as well (but fixed given $\W$). In our analysis, the critical terms are the attention weights given by the correlation between the relevant/irrelevant keys/queries. 

Setting attention queries as the raw tokens (without losing any generality), relevant queries $\x_R$ and keys $\kb_R$ become
\[
\x_R= y \wstar+\qstar,~\kb_R= y \wtt+\qtt.
\] 
Thanks to our choice of $\delta:=\delw=\delq$ to be equally-likely in $\{0,\Delta\}$, observe that irrelevant queries and keys are simply
\[
\x_I= -\delta \x_R,~\kb_I= -\delta \kb_R.
\] 
This will greatly help the proof because it will mean that attention weights are highly structured. Specifically, set $\rho=\x_R^\top\kb_R$. All weights of the attention similarities belong to the set $(\rho,-\delta\rho,\delta^2\rho)$. Consequently, softmax-attention output $\A=\phi(\X \W\X^\top)\X=\begin{bmatrix}\ab_1^\top\\\vdots\\\ab_T^\top\end{bmatrix}$ is given by
\begin{align}
\ab_i=\begin{cases}\frac{\zeta e^{\rho}-\delta(1-\zeta)e^{-\delta\rho}}{\zeta e^{\rho}+(1-\zeta)e^{-\delta\rho}}\cdot \frac{1}{T}\cdot\x_R\quad\text{if}\quad i\in \Rc~\text{(relevant)}\\
\frac{\zeta e^{-\delta\rho}-\delta(1-\zeta)e^{\delta^2\rho }}{\zeta e^{-\delta\rho}+(1-\zeta)e^{\delta^2\rho}}\cdot \frac{1}{T}\cdot\x_R\quad\text{if}\quad i\in \Rc^c~\text{(irrelevant)}\end{cases}.\label{att outputs}
\end{align}
Set $a_+=\frac{ e^{\rho}}{\zeta e^{\rho}+(1-\zeta)e^{-\delta\rho}}$, $a_-=\frac{ e^{-\delta\rho}}{\zeta e^{\rho}+(1-\zeta)e^{-\delta\rho}}$, $b_-=\frac{e^{-\delta\rho}}{\zeta e^{-\delta\rho}+(1-\zeta)e^{\delta^2\rho}}$, $b_+=\frac{e^{\delta^2\rho}}{\zeta e^{-\delta\rho}+(1-\zeta)e^{\delta^2\rho}}$. With this, we also set 
\begin{align*}
&\Delta_R=\frac{\zeta e^{\rho}-\delta(1-\zeta)e^{-\delta\rho}}{\zeta e^{\rho}+(1-\zeta)e^{-\delta\rho}}=\zeta a_+-\delta(1-\zeta)a_-\\
&\Delta_I=\frac{\zeta e^{-\delta\rho}-\delta(1-\zeta)e^{\delta^2\rho }}{\zeta e^{-\delta\rho}+(1-\zeta)e^{\delta^2\rho}}=\zeta b_--\delta(1-\zeta)b_+.
\end{align*}
 Also define $\Delta_i=\Delta_R$ if $i$ is relevant and $\Delta_I$ otherwise. With this, we have $\ab_i=\Delta_i \x_R$ based on \eqref{att outputs}.

 The following lemma will be helpful for the downstream analysis. The goal of this lemma is showing that, by choosing $\del\in \{0,\Delta\}$, we can confuse the model output.
\begin{lemma} \label{lem5050} Fix a scalar $\kappa$. Set $f_\del=\kappa\Delta_R+(1-\kappa)\Delta_I$. Recalling $\rho=\x_R^\top\kb_R$, the following statements hold:
\begin{itemize}
\item Set $\delta=0$. Suppose $``1\geq \kappa\geq 0''\quad\text{OR}\quad ``\kappa\geq 1,\rho\geq 0''\quad\text{OR}\quad``\kappa\leq 0,\rho\leq 0''$. Then $f_\del>0$.
\item Fix $0\leq \alpha\leq 1$. Suppose
\[
\delta>\Delta_0:=\frac{1}{1-\zeta}\max(\frac{\zeta}{\alpha(1-\zeta)},\frac{1}{1-\alpha}).
\]
and $``\kappa\leq \alpha,\rho\geq 0''\quad\text{OR}\quad``\kappa\geq \alpha,\rho\leq 0''$. Then $f_\del<0$. 
\end{itemize}
\end{lemma}
\begin{proof} Plugging in $\delta$, we write 
\begin{align}
f_\del=\kappa\Delta_R+(1-\kappa)\Delta_I&=\kappa\zeta a_+-\delta \kappa(1-\zeta) a_-+ \zeta(1-\kappa) b_--\delta (1-\zeta)(1-\kappa)b_+\\
&=\zeta (\kappa a_+ +(1-\kappa)b_-)-\delta (1-\zeta) (\kappa a_- +(1-\kappa)b_+).\label{main delta eq}
\end{align}
$\bullet$ \textbf{Suppose $\delta=0$.} In this case, we obtain the first statement of the lemma as follows
\begin{align}
f_\del/\zeta=\frac{\kappa e^{\rho}}{\zeta e^{\rho}+1-\zeta}+1-\kappa>0\quad \text{whenever}\quad \begin{cases}1\geq \kappa\geq 0\quad\text{OR}\\\kappa\geq 1,\rho\geq 0\quad\text{OR}\\\kappa\leq 0,\rho\leq 0\end{cases}\label{del=0 cases}
\end{align}
$\bullet$ \textbf{Now suppose $\delta>\Delta_0$.} \textbf{First, assume $\rho \geq 0$ and $\kappa\leq \alpha$.} We use the facts 
\[
1/\zeta \geq a_+\geq 1,\quad 1\geq a_-\geq 0,\quad b_+\geq 1,\quad 1\geq b_-\geq 0.
\]
Observe that, since $b_+\geq a_-$ and $\kappa\leq \alpha$
\[
\kappa a_- +(1-\kappa)b_+\geq\begin{cases}b_+\quad\text{if}\quad \kappa \leq 0\\(1-\alpha) b_+\quad\text{if}\quad \kappa \geq 0\end{cases}\geq  1-\alpha.
\]
Additionally, if $\kappa\leq 0$, we have that
\[
\kappa a_- +(1-\kappa)b_+\geq b_+\geq b_-\geq  \kappa a_+ +(1-\kappa)b_-.
\]
If $\kappa\leq 0$, we obtain $f_\del\leq \zeta b_--\delta(1-\zeta)b_+$. Thus, $f_\del<0$ whenever $\del>\Delta_0\geq \zeta/(1-\zeta)$.

\noindent If $\kappa\geq 0$, we use $\kappa a_+ +(1-\kappa)b_-\leq 1/\zeta$ to obtain that whenever $\del>\Delta_0\geq \frac{1}{(1-\zeta)(1-\alpha)}$
\[
f_\del\leq 1-\del (1-\zeta)(1-\alpha)<0.
\]

\noindent\textbf{Now assume $\rho \leq 0$ and $\kappa\geq \alpha$.} We use the facts 
\[
1\geq a_+\geq 0,\quad a_-\geq 1,\quad 1\geq b_+\geq 0,\quad \frac{1}{1-\zeta}\geq b_-\geq 1.
\]
Observe that, since $b_+\leq a_-$ and $\kappa\geq \alpha$
\[
\kappa a_- +(1-\kappa)b_+\geq\begin{cases}a_-\quad\text{if}\quad \kappa \geq 1\\\alpha a_-\quad\text{if}\quad \kappa \leq 1\end{cases}\geq  \alpha.
\]
Additionally, if $\kappa\geq 1$, we have that
\[
\kappa a_- +(1-\kappa)b_+\geq a_-\geq a_+\geq  \kappa a_+ +(1-\kappa)b_-.
\]
If $\kappa\geq 1$, we obtain $f_\del\leq \zeta a_+-\delta(1-\zeta)a_-$. Thus, $f_\del<0$ whenever $\del>\Delta_0\geq \zeta/(1-\zeta)$.

\noindent If $\kappa\leq 1$, we use $\kappa a_+ +(1-\kappa)b_-\leq \frac{1}{1-\zeta}$ to obtain that whenever $\del>\Delta_0\geq \frac{\zeta}{(1-\zeta)^2\alpha}$
\[
f_\del\leq \frac{\zeta}{1-\zeta}-\del (1-\zeta)\alpha<0.
\]
\end{proof}

To proceed, we will conclude with the proof as follows. Set $\nu_i=y\ub_i^\top \x_R$ for $i\in [T]$ where $\ub_i$ is the $i$th row of the output layer weights $\Ub$. Here $\nu_i$ is obviously $y$-dependent.  However, we will show that for any choice of $y$, the model accuracy is at most $50\%$. Towards this we fix $y$ and (mostly) omit it from the notation during the following discussion. Let $\ab_i$ be the $i$th token of the attention output. The linear output layer $\Ub$ aggregates $\ub_i^\top\ab_i$ to obtain
\[
yf(\Ub,\W)=\sum_{i=1}^T  \ub_i^\top \ab_i=\sum_{i=1}^T \nu_i\Delta_i.
\]
Aggregating $v_+=\frac{1}{T}\sum_{i\in \Rc=\text{relevant}} v_i$ and $v_-=\frac{1}{T}\sum_{i\in \Rc^c=\text{irrelevant}} v_i$ and recalling from \eqref{att outputs} that over relevant/irrelevant sets attention tokens are given by $\Delta_R\x_R$ and $\Delta_I\x_I$, we find
\[
\frac{1}{T}yf(\Ub,\W)=\nu_R\Delta_R+\nu_I\Delta_I.
\]

\noindent\textbf{Scenario 1: Rows of $\Ub$ are identical and we have $\Ub=\ones\ub^\top$.} In this scenario, we simply have $\nu_i=\nu$ and $\nu_R=\zeta \nu$ and $\nu_I=(1-\zeta)\nu$. Thus, we find
\[
\frac{1}{T}yf(\Ub,\W)=T\nu [\zeta \Delta_R+(1-\zeta)\Delta_I].
\]
Set $f_\del=\zeta \Delta_R+(1-\zeta)\Delta_I$. We claim that $\text{sign}(f_\del)$ is Rademacher (given arbitrary $y$ choice) which will prove that accuracy is at most 50\%. Specifically, let us apply Lemma \ref{lem5050} with $\kappa=\zeta$ and $\alpha=\zeta$. When $\delta=0$, we have $f_\delta>0$. When $\delta=\Delta$, since the conditions $\kappa\leq \alpha$ and $\kappa\geq \alpha$ hold, for any choice of $\rho$, for $\Delta>  \Delta_0:=\frac{1}{(1-\zeta)^2}$ we have that $f_\del <0$. Thus, we have that $\Pro_{\delta}(f_\del> 0)=\Pro_{\delta}(f_\del< 0)= 0.5$ as advertised. This follows from the fact that $f_\del>0$ for $\delta=0$ and $f_\del<0$ for $\delta=\Delta$ and $\del$ is equally likely over two options.

\noindent\textbf{Scenario 2:} Suppose rows of $\Ub$ are not identical. In this case, we will leverage the fact that relevant set $\Rc$ is allowed to be chosen adversarially with respect to the self-attention weights. We will show that by selecting $\Rc$ adversarially, on any label $y$ event, accuracy is a coin flip.

\noindent\textbf{First consider the scenario $\nu_{\text{tot}}:=\nu_R+\nu_I\leq 0$:} We will show that model achieves at least 50\% error on label $y$: Let us denote $\nu_R$ with $\nu_R^\Rc$ which makes the relevance set dependence explicit. Given $\Rc$, fixing $\del=0$, the model outputs (following \eqref{del=0 cases})
\[
\frac{1}{T}yf(\Ub,\W)=\frac{\nu^\Rc_R e^{\rho}}{\zeta e^{\rho}+1-\zeta}+\nu^\Rc_I.
\]
Suppose there is a relevance set $\Rc_0$ (that depends on $y$) such that the right hand side is non-positive. Let us select this $\Rc_0$ as our relevance set. Then, the model makes 50\% error on label $y$ thanks to the event $\del=0$ (which is exactly what we want). If there is no such $\Rc_0$, then, for all $\Rc$, we have
\[
\frac{\nu^\Rc_R e^{\rho}}{\zeta e^{\rho}+1-\zeta}+\nu^\Rc_I>0
\]
 By taking average of all relevance sets (``$T$ choose $\zeta T$'' many), all $v_i$'s will be equally-weighted and we obtain $\nu_{\text{tot}}=\nu_R+\nu_I>0$. This contradicts with our initial $\nu_{\text{tot}}\leq 0$ assumption, thus, $\Rc_0$ has to exist.
 
\noindent\textbf{Now consider the scenario $\nu_{\text{tot}}=\nu_R+\nu_I> 0$:}  Let $\Dc$ be the uniform distribution over ``$T$ choose $\zeta T$'' relevant sets $\Rc$. Clearly
$
\E_{\Dc}[\nu^\Rc_R]=\zeta \nu_{\text{tot}}>0.
$
Thus, there is a relevance set $\Rc_+$ such that $\nu^{\Rc_+}_R\geq \zeta\nu_{\text{tot}}$ and there is a relevance set $\Rc_-$ such that $\nu^{\Rc_-}_R\leq \zeta\nu_{\text{tot}}$. We will make use of these two sets to finalize the proof.

To proceed, set $\kappa_\pm=\nu^{\Rc_\pm}/\nu_{\text{tot}}$ and again set $\alpha=\zeta$ and $\Delta_0=\frac{1}{(1-\zeta)^2}$ in Lemma \ref{lem5050}. Here, we are investigating the sign of the prediction
\[
\frac{1}{T\nu_{\text{tot}}}yf(\Ub,\W)= \frac{\nu_R\Delta_R+\nu_I\Delta_I}{\nu_{\text{tot}}} = \kappa_\pm\Delta_R+(1-\kappa_\pm)\Delta_I.
\]
First, assume that the attention weights are so that $\rho=\rho_y\geq 0$. In this case (and for this particular label $y$),
\begin{itemize}
\item When $\del=0$, we choose the relevance set $\Rc_+$ which ensures $\kappa_+\geq\zeta\geq 0$ and $f_0>0$. 
\item When $\del=\Delta>\Delta_0$, we choose the relevance set $\Rc_-$ which ensures $\kappa_-\leq \zeta$ and $f_\Delta<0$. 
\end{itemize}
Secondly, assume that the attention weights are so that $\rho=\rho_y\leq 0$. In this case, 
\begin{itemize}
\item When $\del=0$, we choose the relevance set $\Rc_-$ which ensures $\kappa_-\leq \zeta\leq 1$ and $f_0>0$. 
\item When $\del=\Delta>\Delta_0$, we choose the relevance set $\Rc_+$ which ensures $\kappa_+\geq \zeta$ and $f_\Delta<0$. 
\end{itemize}
In either case, by adaptively choosing $\Rc\in\{\Rc_+,\Rc_-\}$ as a function of $(\del,y)$ pair, we ensure accuracy is at most 50\% because $f_\Delta$ and $f_0$ have conflicting signs.
\end{proof}

\subsection{Success proof for $\Rc$-Adaptive Self-Attention}

Consider the setting of Theorem \ref{separate thm} and Appendix \ref{app satt fail}. We have the following lemma which shows that self-attention can succeed in Theorem \ref{separate thm} if $\Ub$ can adapt to the relevance set (rather than $\Rc$ being adversarial to $\Ub$).
\begin{lemma} \label{lem satt success}In \eqref{eq:CGMM}, choose $(\delq,\delw)$ to be $(0,0)$ or $({\Delta,\Delta})$ equally-likely. Consider the self-attention model $\fsat(\Ub,\W)$ where we set
\[
\Ub=\ones_{\Rc}\wstar'^\top\quad\text{and}\quad \W=\Gamma\Iden.
\]
This model achieves perfect accuracy whenever $\wstar'=(\Iden-\qstab\qstab^\top)\wstar\neq 0$ by choosing 
\[
\Gamma> \frac{1}{(1+\Delta)(\tn{\qstar}-\tn{\wstar})^2+\tn{\wstar}\tn{\qstar}(1-|\rho(\qstar,\wstar)|)}\log(\Delta\frac{1-\zeta}{\zeta}).
\]
where $\rho(\cdot)$ is the correlation coefficient.\footnote{Note that the only instance $\Gamma$ does not exist is when $\qstar=c\wstar$ for $|c|\geq 1$. In this scenario, classification is impossible using the linear head $\wstar'$ without a bias term because all tokens are in the $\sign{c}$ direction regardless of the label $y$.}
\end{lemma}
\begin{proof} Thanks to the masking $\ones_{\Rc}$, we only need to consider the attention scores along relevant tokens. Let $c=\tn{y\wstar+\qstar}^2$. For each relevant token, the attention rows are given by
\[
\ab_i=\begin{cases}e^{\Gamma c}\quad\text{if}\quad i\in\Rc\\e^{-\Delta\Gamma c}\quad\text{if}\quad i\not\in\Rc.\end{cases}
\]
To proceed, attention tokens corresponding to relevant tokens are given by
\begin{align}
\fb&=\sum_{i\in\Rc} \ab_i(\wstar+y\qstar)-\sum_{i\not\in\Rc} \Delta\ab_i(y\wstar+\qstar)\\
&=(\zeta e^{\Gamma c}-\Delta(1-\zeta)e^{-\Delta\Gamma c})(y\wstar+\qstar).
\end{align}
Thus, using $\wstar'\wstar>0$,
\[
\sign{y\fsat(\Ub,\W)}=\sign{y\wstar'^\top\fb}=\sign{\zeta e^{\Gamma c}-\Delta(1-\zeta)e^{-\Delta\Gamma c}}.
\]
Thus, we need $e^{(1+\Delta)\Gamma c}> \Delta\frac{1-\zeta}{\zeta}$ which is implied by $\Gamma> \frac{1}{(1+\Delta)c}\log(\Delta\frac{1-\zeta}{\zeta})$. To conclude, note that for both $y=\pm 1$
\[
c\geq \tn{y\wstar+\qstar}^2\geq \tn{\qstar}^2+\tn{\wstar}^2-2|\qstar^\top \wstar|\geq (\tn{\qstar}-\tn{\wstar})^2+\tn{\wstar}\tn{\qstar}(1-|\rho(\qstar,\wstar)|)>0.
\]
where we used $|\qstar^\top\wstar|=\tn{\qstar}\tn{\wstar}|\rho(\qstar,\wstar)|$.
\end{proof}
\section{Proofs of sharp population risk formulas (Theorem \ref{thm sharp and precise})}

Throughout this section, we use slightly different notation from the one stated in the main body for compactness purposes. Specifically, we set $Q=\tn{\qstar}^2,W=T\tn{\wstar}^2$ rather than $Q=\tn{\qstar},W=\tn{\wstar}$.

\begin{theorem}\label{thm sharp and precise appendix} Consider the \prml model $\fat_{\thetab}$. Set $Q=\tn{\qstar}^2,W=T\tn{\wstar}^2$, suppose $\wstar\perp\qstar$, and let $\taut,\taub>0$ be hyperparameters. Consider the following algorithm which uses the hindsight knowledge of $\qstar$ to estimate $\wstar$ and make prediction:
\begin{enumerate}
\item $\hat{\w}=(\Iden-\qstab\qstab) \nabla \Lc_{\w}(0,\taut\qstab)$.
\item Set $\thetab=(\hat{\w},\taub\qstab)$.
\end{enumerate}
Suppose $\zeta^2W,1-\zeta,\alpha:=n/d,e^Q,e^\taut$ each lie between two positive absolute constants. Suppose $T$ is polynomially large in $n$ and these constants and $\ordet{\cdot}$ hides polynomial terms in $n$. Define \emph{inverse-signal-to-noise-ratio}: $\isnr{\alpha,\taut}=\frac{(1-\zeta)e^{2\taut(\taut-\sqrt{Q})}}{\zeta^2W \alpha}$. With probability $1-2e^{-t^2/2}-\ordet{T^{-1/3}}$ over the training data, the test error obeys
\[
\err(\fat_{\thetab})= \Qc\left(\frac{e^{\sqrt{Q}\taub-\taub^2}}{\sqrt{1+(1\mp\frac{1+t}{\sqrt{d}})\isnr{\alpha,\taut}}}\cdot \sqrt{\frac{\zeta^2W}{1-\zeta}}\right)\pm \ordet{T^{-1/3}}.
\]
Above, $\mp,\pm$ highlights the upper/lower range of the test error (see \eqref{exact Q bound} for exact statement). In the limit $T,d\rightarrow\infty$, the test error converges in probability to 
\[
\err(\alpha,\zeta,Q,W,\taut,\taub)=\Qc\left(\frac{e^{\sqrt{Q}\taub-\taub^2}}{\sqrt{1+\isnr{\alpha,\taut}}}\cdot \sqrt{\frac{\zeta^2W}{1-\zeta}}\right)
\]
In this limit, optimal hyperparameters are $\taut=\taub=\sqrt{Q}/2$ and leads to optimal $\isnr{\alpha}:=\frac{(1-\zeta)e^{-Q/2}}{\zeta^2W \alpha}$ and the error
\[
\err(\alpha,\zeta,Q,W)=\Qc\left(\frac{e^{Q/4}}{\sqrt{1+\isnr{\alpha}}}\cdot \sqrt{\frac{\zeta^2W}{1-\zeta}}\right)
\]
\end{theorem}
\begin{proof} Without losing generality, assume first $\zeta T$ tokens are relevant and remaining tokens are irrelevant. Consider $\X_I$ of size $(1-\zeta)T\times d$ induced by the irrelevant tokens with normal distribution. Observe that $\g=\X_I\wstab$ and $\hb=\X_I\qstab$ are two independent i.i.d.~$\Nc(0,\Iden_{(1-\zeta)T})$ vectors. Also for standard normal $g\sim\Nn(0,1)$, recall that moment-generating function is given by $\E[e^{\taut g}]=e^{\taut^2/2}$.

\noindent\textbf{Step 1: Characterizing the distribution of $\hat\w$.} Note that, the attention weights have the form $\ab=\sft{\taut\begin{bmatrix}\sqrt{Q}\ones_{\zeta T}\\\hb\end{bmatrix}}$. Here, the softmax denominator is $T\cdot D_T$ where $D_T:=(\zeta e^{\sqrt{Q}\taut}+\frac{1}{T}\sum_{i=1}^{(1-\zeta)T} e^{\taut h_i})$. Define $e^{\taut \hb}$ to be the numerator corresponding to irrelevant tokens i.e.
\[
e^{\taut \hb}=[e^{\taut h_1}~\dots~e^{\taut h_{(1-\zeta)T}}].
\]
Define the matrix $\Qb_\perp=\Iden-\qstab\qstab^\top$,  $\W_\perp=\Iden-\wstab\wstab^\top$. Set the vector $\vb=\frac{1}{T}\Qb_\perp\X_I^\top e^{\taut \hb}$ and $\vb_{\perp}=\frac{1}{T}\hb^\top e^{\taut \hb}\qstab $. To proceed, observe that, for a single sample $(y,\X)$, the gradient has the form
\begin{align}
\nabla \Lc^{y,\X}_{\w}(0,\taut\qstab)=y\X^\top\ab = \frac{\zeta (\wstar+y\qstar)e^{\sqrt{Q}\taut}+\vb+\vb_{\perp}}{D_T}.
\end{align}
After projection this onto the $\qstar$-complement $\Qb_\perp$, we get rid of the $\qstar$ direction to obtain
\[
\hat\w_{y,\X}=\Qb_\perp\nabla \Lc^{y,\X}_{\w}(0,\taut\qstab)=\frac{1}{D_T}[\zeta \wstar e^{\sqrt{Q}\taut}+\Qb_\perp\X_I^\top e^{\taut \hb}/T].
\]
The projected gradient over the full dataset is given by the empirical average
\[
\hat\w=\Qb_\perp\nabla \Lc_{\w}(0,\taut\qstab)=\frac{1}{n}\sum_{i=1}^n\frac{1}{D_{i,T}}[\zeta \wstar e^{\sqrt{Q}\taut}+\Qb_\perp\X_{i,I}^\top e^{\taut \hb_i}/T].
\]
Here $\h_i,\X_{i,I},D_{i,T}$ denote the random variables induced by the $i$th sample. Here, a critical observation is the fact that $\Qb_\perp\X_{i,I}$ is independent of $\hb_i$ (thanks to Gaussian orthogonality), thus, $\Qb_\perp\X_{i,I}e^{\taut \hb_i}$ is normal conditioned on $\hb_i$. To proceed, we apply Chebyshev's inequality over number of tokens $T$. Recall that we assumed $e^\taut\leq C$ for an absolute constant $C\geq 1$. This means that $e^{c\taut^2}\leq C^{c\taut}\leq C^{c\log C}$ is polynomial in $C$ and is also upper bounded by a constant. In what follows $\ordet{\cdot}$ only reflects the $T$ dependence and subsumes polynomial dependence on the terms $n,C$. For all $1\leq i\leq n$, applying Chebyshev's inequality, for $T\gtrsim \text{poly}(n,e^{\taut^2})$, with probability $1-T^{-1/3}$, we have that
\begin{itemize}
\item Since $\tn{e^{\taut \hb_i}}^2/T=\frac{1}{T}\sum_{j=1}^T e^{2\taut \hb_{ij}}$ thus $\tn{e^{\taut \hb_i}}^2/T- (1-\zeta)e^{2\taut^2}\leq \ordet{T^{-1/3}}$,
\item Set $\E[D_T]=D_{\infty}:=\zeta e^{\sqrt{Q}\taut}+(1-\zeta) e^{\taut^2/2}$. $|D_{i,T}-D_\infty|\leq \ordet{T^{-1/3}}$.
\end{itemize}
With these, set $\bb_i=\frac{\sqrt{1-\zeta}e^{\taut^2}}{\tn{e^{\taut \hb_i}}}e^{\taut \hb_i}$ which is a vector with fixed $\ell_2$ norm that is perfectly parallel to $e^{\taut \hb_i}$. Since $\tn{\bb_i}^2=\E[\tn{e^{\taut \hb_i}}^2/T]=(1-\zeta)e^{2\taut^2}$, from above, observe that,  
\[
\tn{\bb_i-\frac{1}{\sqrt{T}}e^{\taut \hb_i}}\leq \ordet{T^{-1/3}}.
\]
Now, let 
\[
\bar{\vb}=\frac{1}{\sqrt{n}}\sum_{i=1}^n\Qb_\perp\X_{i,I}^\top\bb_i.
\]
Since $\Qb_\perp\X_{i,I}^\top,\bb_i$ are independent and $\bb_i$ has fixed $\ell_2$ norm, we have that
\[
\bar{\vb}\sim\Nn(0,(1-\zeta)e^{2\taut^2}\Qb_\perp).
\]
Finally, let $\cb=\zeta e^{\sqrt{Q}\taut} \wstar$. Recalling $\sqrt{T}\tn{\wstar}=W$,  combining the perturbations bounds above, we have that
\[
\sqrt{T}\tn{\cb/D_\infty-\frac{1}{n}\sum_{i=1}^n\frac{1}{D_{i,T}}(\zeta \wstar e^{\sqrt{Q}\taut})}\leq \ordet{T^{-1/3}}.
\]
Combining these observe that
\begin{align}
\tn{\sqrt{T}D_\infty\hat\w-\sqrt{T}\zeta e^{\sqrt{Q}\taut} \wstar-\bar{\vb}/\sqrt{n}}\leq \ordet{T^{-1/3}}. \label{T control}
\end{align}
Since $\bar{\vb}$ is normally distributed, above also implies that $\sqrt{T}D_\infty\hat\w$ converges to the normal distribution $\Nn(\sqrt{T}\zeta e^{\sqrt{Q}\taut} \wstar,\frac{(1-\zeta)e^{2\taut^2}}{n}\Qb_\perp)$ in the limit $T\rightarrow\infty$.

\begin{lemma}[Inverse Signal-to-Noise Ratio (ISNR)] Set $\W_\perp=\Iden-\wstab\wstab^\top$. Define SNR of $\hat\w$ to be 
\[
\isnr{\hat\w}=\frac{\tn{\W_\perp\hat\w}^2}{\tn{\wstab^\top\hat\w}^2}.
\]
Recall $\isnr{\alpha,\taut}=\frac{(1-\zeta)e^{2\taut(\taut-\sqrt{Q})}}{\zeta^2W \alpha}$. With probability $1-2e^{-t^2/2}-T^{-1/3}$ over the dataset, we have that
\[
 \left(1-\frac{t+1}{\sqrt{d}}-\ordet{T^{-\frac{1}{3}}}\right)_+^2\leq \frac{\isnr{\hat\w}}{\isnr{\alpha,\taut}}\leq \left(1+\frac{\taut}{\sqrt{d}}+\ordet{T^{-\frac{1}{3}}}\right)^2.
\]
\end{lemma}
\begin{proof} Let us recall the standard normal concentration: For $\g\sim\Nn(0,\Iden_{d-1})$, $\sqrt{d-1}\geq \E[\tn{\g}]\geq \frac{d-1}{\sqrt{d}}$. Thus, with probability $1-2e^{-t^2/2}$, through Lipschitz concentration,
\[
\sqrt{d}+t\geq \tn{\g}\geq \sqrt{d}-1-t.
\]
This means that, with the same probability
\[
\sqrt{d}+t\geq \frac{\tn{\bar{\vb}}}{\sqrt{1-\zeta}e^{\taut^2}}\geq (\sqrt{d}-1-t)_+.
\]
We first upper bound $\tn{\W_\perp\hat\w}^2$. Recalling \eqref{T control}, 
\[
\tn{\W_\perp\hat\w-\bar{\vb}/\sqrt{n}}\leq \ordet{T^{-1/3}}.
\]
Thus, 
\[
(\sqrt{d}+t)^2+\ordet{T^{-1/3}}\geq \frac{n\tn{\W_\perp\hat\w}^2}{(1-\zeta)e^{2\taut^2}}\geq (\sqrt{d}-1-t)_+^2-\ordet{T^{-1/3}}.
\]
Using $\tn{\wstar}^2T=W$, We similarly have that 
\[
|\tn{\wstab^\top\hat\w}^2-W\zeta^2 e^{2\sqrt{Q}\taut}|\leq  \ordet{T^{-1/3}}.
\]
To conclude, with probability $1-2e^{-t}-T^{-1/3}$, $\isnr{\hat\w}$ obeys
\[
\frac{(\sqrt{d}+t)^2+\ordet{T^{-1/3}}}{W\zeta^2 e^{2\sqrt{Q}\taut}-\ordet{T^{-1/3}}}\geq \frac{n}{(1-\zeta)e^{2\taut^2}}\isnr{\hat\w}\geq \frac{(\sqrt{d}-1-t)_+^2-\ordet{T^{-1/3}}}{W\zeta^2 e^{2\sqrt{Q}\taut}+\ordet{T^{-1/3}}}
\]
Rewriting this bound, we find
\[
\left(1+\frac{t}{\sqrt{d}}+\ordet{T^{-\frac{1}{3}}}\right)^2\frac{(1-\zeta)e^{2\taut^2}}{\zeta^2W \alpha e^{2\sqrt{Q}\taut}}\geq \isnr{\hat\w}\geq \left(1-\frac{1+t}{\sqrt{d}}-\ordet{T^{-\frac{1}{3}}}\right)_+^2\frac{(1-\zeta)e^{2\taut^2}}{\zeta^2W \alpha e^{2\sqrt{Q}\taut}}.
\]
Recalling the definition of $\isnr{\alpha,\taut}=\frac{(1-\zeta)e^{2\taut(\taut-\sqrt{Q})}}{\zeta^2W \alpha}$, we conclude with the bound.
\end{proof}

\smallskip
\noindent\textbf{Step 2: Characterizing the error rate of $\thetab=(\hat{\w},\taub\qstar)$.} To achieve this goal, we will leverage Theorem \ref{thm sharp error}. Since conditions of this theorem is satisfied (noticing that their $\gamma$ is our $\isnr{\hat\w}$ which is upper bounded by a positive constant), for a new test point $(y,\X)$, we have that 
\[
\left|\err(\fat_{\thetab})-\Qc\left(\frac{e^{\sqrt{Q}\taub-\taub^2}}{\sqrt{1+\isnr{\hat\w}}}\cdot \sqrt{\frac{\zeta^2W}{1-\zeta}}\right)\right|\leq \ordet{T^{-1/3}}.
\]
Using the Lipschitzness of the Q-function (i.e.~$\Qc(x+\eps)-\Qc(x)=\int_x^{x+\eps} e^{-t^2/2}dt\leq \eps$), as we have done in Theorem \ref{thm sharp error}, we pull out the perturbation term $\ordet{T^{-1/3}}$ within $\isnr{\hat\w}$ to obtain the advertised bound
\begin{align}
&\Qc\left(\frac{e^{\sqrt{Q}\taub-\taub^2}}{\sqrt{1+(1+\frac{t}{\sqrt{d}})\isnr{\alpha,\taut}}}\cdot \sqrt{\frac{\zeta^2W}{1-\zeta}}\right)- \ordet{T^{-1/3}}\leq \err(\fat_{\thetab})\leq\\
 &\quad\quad\quad\quad\quad\quad \Qc\left(\frac{e^{\sqrt{Q}\taub-\taub^2}}{\sqrt{1+(1-\frac{1+t}{\sqrt{d}})_+\isnr{\alpha,\taut}}}\cdot \sqrt{\frac{\zeta^2W}{1-\zeta}}\right)+ \ordet{T^{-1/3}}.\label{exact Q bound}
\end{align}
To emphasize, this bound holds with probability $1-2e^{-t^2/2}-\ordet{T^{-1/3}}$ over a new test datapoint $(y,\X)$. To see the optimal choices for $\taub,\taut$, we need to optimize the error bound. This results in
\begin{align}
&\taub_\st=\arg\min_{\taub} \sqrt{Q}\taub-\taub^2=\sqrt{Q}/2\\
&\taut_\st=\arg\min_{\taut} \isnr{\alpha,\taut}=2\tau(\tau-\sqrt{Q})=\sqrt{Q}/2.
\end{align}

\end{proof}

\begin{theorem}\label{thm sharp error} Consider the \prml model $\fat_{\thetab}$ where we set $\thetab=(\wstar+\pb,\taut\qstab)$. Set $Q=\tn{\qstar}^2,W=T\tn{\wstar}^2$. Here $\taut$ is a tuning parameter and $\pb$ is a perturbation vector and assume all vectors are perpendicular i.e.~$\pb\perp \wstar\perp\qstar$. Set $\gamma:=\tn{\pb}^2/\tn{\wstar}^2$ and suppose $1+\gamma,\zeta^2W,1-\zeta,e^Q,e^\taut$ each lie between two positive absolute constants. $\ordet{\cdot}$ subsumes polynomial dependencies in these constants. We have that
\[
\left|\err(\fat_{\thetab})-\Qc\left({\frac{e^{\sqrt{Q}\taut-\taut^2}}{\sqrt{1+\gamma}}}\cdot \sqrt{\frac{\zeta^2W}{1-\zeta}} \right)\right|\leq \order{T^{-1/3}}.
\]
Thus, as $T\rightarrow\infty$, the optimal tuning obeys {$\taut_\star=\sqrt{Q}/2$} and yields an error of $\Qc\left({e^{Q/4}}\cdot \sqrt{\frac{\zeta^2W}{1-\zeta}} \right)$.
\end{theorem}
\begin{proof} Let us recap the notation of Theorem \ref{thm sharp and precise}. Without losing generality, assume first $\zeta T$ tokens are relevant and remaining tokens are irrelevant. Consider $\X_I$ of size $(1-\zeta)T\times d$ induced by the irrelevant tokens with normal distribution. Using orthogonality of $\qstar,\wstar,\pb$, observe that $\g=\X_I(\wstab+\frac{\pb}{\tn{\wstar}})\sim \Nc(0,(1+\gamma)\Iden_{(1-\zeta)T})$ and $\hb=\X_I\qstab\sim\Nc(0,\Iden_{(1-\zeta)T})$ are independent vectors. Also for standard normal $g\sim\Nn(0,1)$, recall that moment-generating function is given by $\E[e^{\taut g}]=e^{\taut^2/2}$.

Note that, the attention weights have the form $\ab=\sft{\taut\begin{bmatrix}\sqrt{Q}\ones_{\zeta T}\\\hb\end{bmatrix}}$. Here, the softmax denominator is $T\cdot D_T$ where $D_T:=(\zeta e^{\sqrt{Q}\taut}+\frac{1}{T}\sum_{i=1}^{(1-\zeta)T} e^{\taut h_i})$. Define $e^{\taut \hb}$ to be the numerator corresponding to irrelevant tokens i.e.
\[
e^{\taut \hb}=[e^{\taut h_1}~\dots~e^{\taut h_{(1-\zeta)T}}].
\]
Define the matrix $\Qb_\perp=\Iden-\qstab\qstab^\top$,  $\W_\perp=\Iden-\wstab\wstab^\top$. To proceed, observe that, the prediction with $\thetab=(\wstar+ \pb,\taut\qstab)$ is given by
\begin{align}
\frac{D_T}{\sqrt{T}\tn{\wstar}}y\fat_{\thetab}(\X)&=\sqrt{T}(\wstab+\frac{\pb}{\tn{\wstar}})^\top[\zeta e^{\tn{\qstar}\taut}(\wstar+y\qstar) +\frac{\X_I e^{\taut\hb}}{T}]\\
&=\zeta e^{\tn{\qstar}\taut}\sqrt{W}+\frac{1}{\sqrt{T}}\g^\top e^{\taut \hb}.
\end{align}
With this, conditioned on $e^{\taut \hb}$ observe that $\g^\top e^{\taut \hb}\sim \Nn(0,\frac{1}{T}\tn{e^{\taut \hb}}^2)$, thus,
\[
\Pro_\g(y\fat_{\thetab}(\X)> 0)=1-\Qc\left(\frac{\zeta e^{\sqrt{Q}\taut}\sqrt{W}}{\sqrt{1+\gamma}\tn{e^{\taut \hb}}/\sqrt{T}}\right).
\]
To proceed, similar to Theorem \ref{thm sharp and precise}, we apply Chebyshev's inequality over number of tokens $T$ to find that with probability $1-\ordet{T^{-1/3}}$ over $\hb$, 
\[
|\tn{e^{\taut \hb}}^2/T-e^{2\taut^2}|\leq \ordet{T^{-1/3}}.
\]
In aggregate, this implies that, with probability $1-\ordet{T^{-1/3}}$ over $\hb$, we have that
\[
1-\Qc\left((1+\ordet{T^{-1/3}})\frac{\zeta e^{\sqrt{Q}\taut}\sqrt{W}}{\sqrt{1+\gamma}e^{\taut^2}}\right)\geq \Pro_\g(y\fat_{\thetab}(\X)> 0)\geq 1-\Qc\left((1-\ordet{T^{-1/3}})\frac{\zeta e^{\sqrt{Q}\taut}\sqrt{W}}{\sqrt{1+\gamma}e^{\taut^2}}\right),
\]
Finally, note that since $\frac{\zeta e^{\sqrt{Q}\taut}\sqrt{W}}{\sqrt{1+\gamma}e^{\taut^2}}$ is upper/lower bounded by a positive constant, and since
$\Qc(x+\eps)-\Qc(x)=\int_x^{x+\eps} e^{-t^2/2}dt\leq \eps$, we can rewrite
\[
\left|\Pro_\g(y\fat_{\thetab}(\X)> 0)-\Qc\left(\frac{\zeta e^{\sqrt{Q}\taut}\sqrt{W}}{\sqrt{1+\gamma}e^{\taut^2}}\right)\right|\leq \ordet{T^{-1/3}}.
\]
Union bounding with failure probability over $\hb$, we conclude with the result.
\end{proof}
\section{Useful facts}

For a random variable $Z$ and $\alpha>0$, $\|Z\|_{\psi_\alpha}$ denotes its $\psi_\alpha$-norm for Orlicz function $\psi_\alpha(z)=e^{z^\alpha}-1$ \cite{ledoux1991probability}.

\begin{fact}\label{fact:psi_additive}
Let $X_1,\ldots,X_n$ be independent zero-mean sub-gaussian or sub-exponential random variables with $\psinorm{X_i}{m}\leq K$ for all $i\in[n]$ for either $m=2$ or $m=1$. Then,
$$
\psinorm{\frac{1}{n}\sum_{i\in[n]}X_i}{m}\leq \frac{CK}{\sqrt{n}}.
$$
\end{fact}

\begin{fact}\cite{pollard1990empirical}
\label{fact:orlitz_inequality}
The following identity holds for Orlitz norms 
\begin{align}\label{eq:Orlitz_inequality}
\|XY\|_{\psi_{\frac{\alpha\beta}{\alpha+\beta}}}\le c\,\|X\|_{\psi_\alpha}\cdot\|Y\|_{\psi_\beta} 
\end{align}
for a fixed numerical constant $c$.
\end{fact}
Next we state a Lemma from Talagrand quoted directly from Lemma 22 of \cite{mohammadi2019convergence}.
\begin{fact}\label{fact:psi_heavy_Talagrand}\cite{ledoux1991probability}For any scalar $\alpha\in(0,1]$, there exists a constant $C_\alpha$ such that for any sequence of independent random variables $\xi_1,\xi_2,\ldots,\xi_N$ we have 
\begin{align*}
\Big\|\sum_i \xi_i -\mathbb{E}\big[\sum_i \xi_i\big]\Big\|_{\psi_\alpha}\le C_\alpha\left(\max_i \|\xi_i\|_{\psi_\alpha}\right)\sqrt{N}\log N.
\end{align*}
\end{fact}

\begin{fact}\label{fact:norm subg}
    Let $\z\in\R^d$ be a $K$-subgausssian vector i.e. $\z^\top \vb$ is $K$-subgaussian for fixed $\tn{\vb}=1$. Then, the following are true for a constant $c>0$
    \[
    \Pro\left(\tn{\z}\geq cK({\sqrt{d}}+t)\right)\leq e^{-t^2}\,.
    \]
\end{fact}
\begin{proof} For completeness, we provide a proof. Repeating Lemma 31 of \cite{oymak2019stochastic}, we can pick a $1/2$ cover $\Cc$ of the unit Euclidean ball in $\R^d$ with size $\log|\Cc|\leq 2d$. For any $\vb\in \Cc$ subgaussianity implies $\Pro(\vb^\top \z\geq t)\leq \exp(-ct^2/K^2)$. Setting $K'=K/\sqrt{c}$, $t=K'(\sqrt{2d}+\tau)$ and union bounding over all $\vb\in\Cc$, we find
\[
\Pro(\sup_{\vb\in\Cc} \vb^\top\z\geq K'(\sqrt{d}+\tau))\leq \exp(-\tau^2).
\]
To proceed, set $\vb(\z)\in\Cc$ to be the nearest point to $\bar{\z}=\z/\tn{\z}$ in $\Cc$. Since $\tn{\vb(\z)-\z}\leq 0.5$, note that 
\[
\tn{\z}=\bar{\z}^\top\z = \vb(\z)^\top\z+(\z-\vb(\z))^\top\z\leq \vb(\z)^\top\z+0.5\tn{\z}.
\]
Thus, $\tn{\z}\leq 2K'(\sqrt{d}+\tau)$ with probability at least $1-\exp(-\tau^2)$.

\end{proof}


%


\end{document}